\documentclass{article}
\usepackage{iclr2027_conference,times}
\usepackage{amsmath,amssymb,amsfonts,mathtools,amsthm}
\usepackage{graphicx,booktabs,array,makecell,tabularx}
\usepackage[table]{xcolor}
\usepackage{xspace,pifont,microtype,placeins,needspace}
\usepackage{xr-hyper}
\usepackage[hidelinks]{hyperref}
\usepackage{url}
\usepackage{xurl}
\definecolor{smInk}{HTML}{26232A}
\definecolor{smMuted}{HTML}{68636C}
\definecolor{smBlue}{HTML}{5402A4}
\definecolor{smBlueFill}{HTML}{F0E8FA}
\definecolor{smGreen}{HTML}{3F3946}
\definecolor{smGreenFill}{HTML}{F7F5F9}
\definecolor{smOchre}{HTML}{FEBC2A}
\definecolor{smViolet}{HTML}{8438D1}
\definecolor{smRed}{HTML}{F2694C}
\definecolor{smRedFill}{HTML}{FFF0EC}
\definecolor{smNullInk}{HTML}{68636C}
\definecolor{smNullFill}{HTML}{F7F5F9}
\definecolor{smBand}{HTML}{F7F5F9}
\definecolor{smRule}{HTML}{E7E3EA}
\definecolor{smAstraMagenta}{HTML}{B71995}
\definecolor{smExactMB}{HTML}{5402A4}
\definecolor{smTableNavy}{HTML}{3F3946}
\definecolor{smTableHeader}{HTML}{F1EFF4}
\definecolor{smTableSection}{HTML}{F7F5F9}
\definecolor{smTableFocus}{HTML}{F0E8FA}
\definecolor{smTableRule}{HTML}{D4CFDA}
\definecolor{smTableRuleDark}{HTML}{8E8A92}

\newcommand{\ExactMB}{ExactMB\xspace}

\newcommand{\LDSyn}{LD-Syn\xspace}
\newcommand{\LDReal}{LD-Real\xspace}
\newcommand{\MDReal}{MD-3K\xspace}

\newcommand{\TauDiagDefault}{0.5}

\newcommand{\nullsym}{\varnothing}
\newcommand{\missing}{\bot}

\newcommand{\indicator}{\mathbf{1}}
\newcommand{\R}{\mathbb{R}}
\newcommand{\DepthSpace}{\mathcal{Y}}

\newcommand{\dd}{\,\mathrm d}

\newcommand{\Bern}{\operatorname{Bernoulli}}
\newcommand{\sg}{\operatorname{sg}}
\newcommand{\im}{\operatorname{im}}
\newcommand{\KL}{\mathrm{KL}}
\newcommand{\E}{\mathbb{E}}
\newcommand{\Prob}{\mathbb{P}}

\newcommand{\Ltarget}{\mathcal L_{\mathrm{target}}}
\newcommand{\Lcover}{\mathcal L_{\mathrm{cover}}}
\newcommand{\Lbi}{\mathcal L_{\mathrm{bi}}}

\newcommand{\TauDiag}{\tau_{\mathrm{diag}}}

\newcommand{\TableHead}[1]{\textbf{#1}}

\newcommand{\ResultLead}[1]{\par\smallskip\noindent\textcolor{smGreen}{\textbf{#1}}\enspace}

\DeclareMathOperator*{\argmax}{arg\,max}
\DeclareMathOperator{\softplus}{softplus}

\newtheorem{proposition}{Proposition}
\newtheorem{theorem}{Theorem}

\hypersetup{pdftitle={Depth Any Seen: Which Surfaces and How Far?},pdfauthor={Xiaohao Xu; Xiaonan Huang},pdfsubject={Preprint}}

\title{Depth Any Seen: Which Surfaces and How Far?}
\author{Xiaohao Xu \qquad Xiaonan Huang\\
\normalfont Robotics Department, University of Michigan, Ann Arbor\\
\normalfont\texttt{xiaohaox@umich.edu}}
\iclrfinalcopy
\begin{document}
\maketitle
\lhead{Preprint}
\vspace{-12pt}
\begingroup
\renewcommand{\thefootnote}{}
\footnotetext{Video demo: \url{https://youtu.be/D8TIzhjq-2Y}}
\endgroup
\begin{abstract}
When several surfaces are visible along a ray, recovering visible 3D structure from one image requires jointly estimating their presence and metric depth. Depth Any Seen represents these surfaces as image-conditioned multi-Bernoulli depth sets, whose components each contribute one depth or remain absent. Its auxiliary-free Exact Multi-Bernoulli objective (ExactMB) learns depth and presence by marginalizing one-to-one assignments to complete, distinct targets. Our analysis shows that matching expected count can leave component--surface assignment unresolved. We extend real and synthetic layered-depth benchmarks to evaluate depth accuracy, recovered support, and overprediction. Compared to depth stacking, ExactMB reduces overprediction by a relative $88.2\%$ on LD-Real and $80.5\%$ on MD-3K while retaining most ordinal accuracy, with comparable conditional metric-depth error on LD-Syn. Further ablation studies show that ordered assignment improves depth-accurate recall and precision over marginalization, whereas the count-regularized configuration achieves higher deeper-rank precision than ordered assignment at lower recall. Our code will be publicly released.
\end{abstract}

\begin{figure}[!ht]
\centering
\includegraphics[page=20,width=\linewidth]{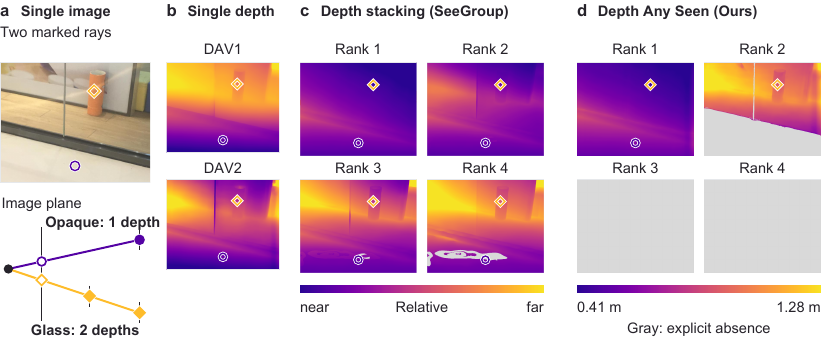}
\caption{\textbf{Recovering visible 3D geometry requires two answers: which surfaces are present, and how far away are they?} (a) A single RGB image: purple circles and amber diamonds mark opaque and through-glass rays. (b) Depth Anything V1/V2 \citep{depthanything,depthanythingv2} compress layered visibility into one depth per ray. (c) SeeGroup's released evaluation decoder \citep{seegroup2026} retains up to four ranked relative depths per ray. (d) Depth Any Seen predicts ray-adaptive metric depth sets with explicit absence for unused components. Gray denotes no retained depth.}
\label{fig:iclr-teaser}
\end{figure}

\Needspace{4\baselineskip}
\section{Introduction}
\label{sec:iclr-introduction}

A single image can reveal more surfaces than a single depth map can represent. Through glass, the surface and the scene behind it can both be visible along the same ray. \textbf{Recovering everything seen requires learning which surfaces are present and where they lie.} As shown in Figure~\ref{fig:iclr-teaser}(d), their depths must share a metric scale to capture both variation within layers and separation between them.

Layered depth images, support masks, and adaptive layers represent multiple surfaces \citep{shade1998,dhamo2019object}. LayeredDepth supplies large-scale synthetic layered-depth data \citep{layereddepth2025}, while SeeGroup learns unordered components with variable retained counts \citep{seegroup2026}. Recovery also requires deciding which predictions represent visible surfaces. Count alone is insufficient: extra predictions can offset missing surfaces. Depth error on retained predictions cannot reveal omissions. This motivates joint depth--presence learning and support-aware evaluation.

Depth Any Seen learns these decisions jointly through an \textbf{image-conditioned random finite set}. Its elements are coexisting surfaces with uncertain presence and depth, not alternative estimates of one depth. We leverage multi-Bernoulli set prediction \citep{hess2022object}: components independently emit one depth or remain absent. Our Exact Multi-Bernoulli likelihood (ExactMB) marginalizes one-to-one assignments to complete, distinct targets, coupling presence and depth.

Components receive presence credit for explaining observed depths. These credits sum to the target count, linking geometric explanation to presence supervision. Our analysis distinguishes expected-count correction from the allocation of presence credit among components. \textbf{Matching expected count can therefore leave assignment unresolved.} We consequently test how count-regularized configurations and assignment supervision shape visible-surface recovery beyond ExactMB.

Our joint depth--presence benchmarks combine synthetic metric depth, real-data ordinal relations and membership, and fixed-ray set recovery to assess where surfaces lie and which are recovered.

This evaluation shows that auxiliary-free ExactMB reduces annotation-defined LD-Real overprediction from $98.4\%$ to $11.6\%$ compared with depth stacking, retaining most ordinal accuracy on LD-Real and MD-3K with comparable conditional metric-depth error on LD-Syn. Additional supervision changes this balance: the count-regularized configuration improves aggregate count accuracy and deeper-rank precision, while ordered assignment improves depth-accurate recall and precision over marginalization, with a higher mean MD-3K overprediction rate.

\textbf{Contributions.}
\textbf{1)~Joint depth and membership.} ExactMB specializes the multi-Bernoulli model to dense visible-depth sets with exact assignment marginalization and no auxiliary losses.
\textbf{2)~Count and assignment supervision.} We show that matching expected count need not resolve surface assignment, and that count-regularized configurations and ordered assignment favor different precision--recall trade-offs.
\textbf{3)~Joint depth and presence benchmarking.} We extend existing multilayer-depth benchmarks to jointly evaluate depth accuracy and explicit surface presence. Results show lower overprediction than depth stacking while retaining most ordinal accuracy, and expose localization--support trade-offs across emission and assignment design choices.

\section{Related Work}
\label{sec:iclr-related}

\textbf{Visible-depth and amodal prediction.}
LayeredDepth presents a large-scale synthetic layered-depth dataset \citep{layereddepth2025}, while SeeGroup learns unordered components with intensity and coverage losses \citep{seegroup2026}. MDA permits transparent multilayers through independent weights \citep{bian2026ambiguity}. With some frozen monocular models, \citet{twodepths2026} elicit unordered ordinal depth pairs using RGB and Laplacian Visual Prompting (LVP) in two-layer transparent scenes. Amodal methods extend beyond visible surfaces: LaRI learns stopping for amodal intersections \citep{lari2026}, World Tracing reconstructs visible and occluded intersections \citep{worldtracing2026}, and TRELLIS and SAM 3D generate complete object geometry \citep{trellis2025,trellis22026,sam3d2026}. Our objective instead couples visible-set membership with metric depth to recover coexisting surfaces seen along each ray, including through transparent foreground objects.

\textbf{Feed-forward geometry.}
MVSNet \citep{mvsnet2018} learns depth from calibrated views. DUSt3R \citep{dust3r2024} regresses pairwise pointmaps, and MASt3R \citep{mast3r2024} adds local matching features. VGGT \citep{vggt2025} and $\pi^3$ \citep{pi32026} predict cameras and dense geometry, while MapAnything \citep{mapanything2026} unifies metric reconstruction. Our complementary goal is coexisting visible depths and their membership on a shared metric scale from one image.

\textbf{Layered geometry and rendering.}
Learned layers support view synthesis \citep{tulsiani2018layer}, masked reconstruction \citep{shin2019scene}, and object-wise decomposition \citep{dhamo2019object}, while layered inpainting adds occluded samples \citep{shih2020photography}. For insertion and shading, \citet{engel2024decomposition} infer depth, color, and opacity intervals from semitransparent volume renderings.

\textbf{Transparent-scene depth and reconstruction.}
ClearGrasp \citep{sajjan2020} and TransCG \citep{fang2022transcg} advance RGB-D completion, and local implicit functions jointly predict termination probability and position \citep{zhu2021implicit}. Depth4ToM \citep{costanzino2023depth4tom}, MODEST \citep{liu2025modest}, and SeeClear \citep{seeclear2026} recover selected surface depth. For multilayer recovery, ASGrasp \citep{shi2024asgrasp} reconstructs two layers from RGB and active stereo for grasping, while DepthFocus \citep{depthfocus2026} selects layers through distance queries on stereo observations. Using only one monocular image, we jointly model presence and metric depth across visible layers.

\begin{figure}[t]
\centering
\includegraphics[page=21,width=\linewidth]{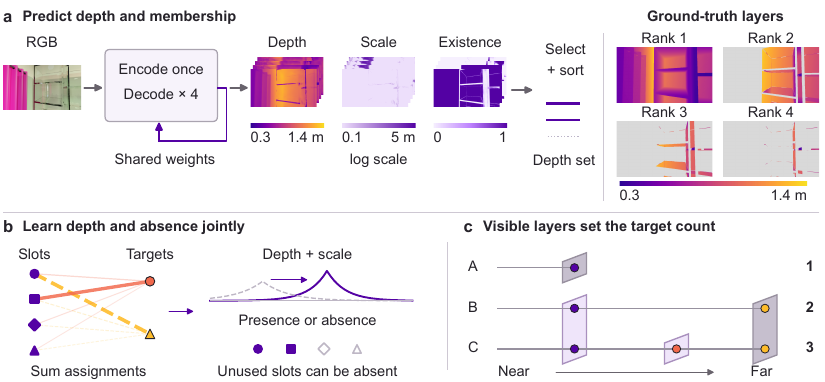}
\caption{\textbf{One likelihood learns depth and absence.} (a) Our default feed-forward model uses recurrent decoding to explicitly predict depth, scale, and existence probability. Selection and sorting form the output set. (b) ExactMB sums over one-to-one assignments to explain observed returns and score unused components as absent. Ordered assignment retains optional emission. (c) Annotated visibility determines the target return count. Scale describes localization conditional on presence.}
\label{fig:iclr-overview}
\end{figure}

\section{Depth Any Seen via Ray-Adaptive Depth Sets}
\label{sec:iclr-method}

\noindent\textbf{Problem statement and assumptions.}
Given one image, we seek a set of visible depths per viewing ray. Each layer records an intersected object surface's optical-axis depth in meters. The target $\mathcal T=\{t_1,\ldots,t_M\}$ contains exactly $M$ depths, with $K$ the maximum supported cardinality ($0\le M\le K$). We assume that objects through which farther layers are visible are transparent and locally modeled as finite-thickness glass. Geometry hidden behind opaque surfaces is excluded. Ray-wise targets leave cross-pixel surface identities unspecified. Qualitative out-of-distribution examples explore irregular transparent objects beyond this approximation.

\subsection{Continuous depth, discrete membership}
\label{sec:iclr-components}
Figure~\ref{fig:iclr-overview}(a) shows the prediction pipeline. The image is encoded once, and shared-weight recurrent decoder passes predict $K$ triples $(d_j,\beta_j,e_j)$: depth center, positive localization scale, and existence logit. A Bernoulli variable $R_j\sim\Bern(q_j)$ determines whether component $j$ contributes a return. It is absent when $R_j=0$ and emits one depth when $R_j=1$, with an untruncated Laplace density
\begin{equation}
q_j=\sigma(e_j),\qquad
\ell_j(t)=\frac{1}{2\beta_j}\exp\!\left(-\frac{|t-d_j|}{\beta_j}\right).
\label{eq:iclr-component}
\end{equation}
Here $\beta_j$ measures conditional localization spread, with a positive floor preventing unbounded likelihood. Densities are normalized on $\mathbb R$, although targets and decoded depths are positive. Components are independent given the image and network outputs, defining a multi-Bernoulli random set $\mathsf X$. Continuous draws are distinct almost surely, so $|\mathsf X|=\sum_jR_j$.

Optional emission lets count vary within capacity $K$. Component identities are not depth ranks: assignments link components to targets, while selection, sorting, and filtering define decoded ranks.

\subsection{ExactMB: one likelihood for depth and absence}
\label{sec:iclr-set-likelihood}
To score targets without prescribing component ranks, each one-to-one assignment, or injection, $\phi:\{1,\ldots,M\}\hookrightarrow\{1,\ldots,K\}$ matches each observed depth to a distinct predicted depth component, with whole-set density
\begin{equation}
f(\mathcal T,\phi)=
\underbrace{\prod_{i=1}^{M}q_{\phi(i)}\ell_{\phi(i)}(t_i)}_{\substack{\text{explain observed}\\\text{returns}}}
\underbrace{\prod_{j\notin\operatorname{im}\phi}(1-q_j)}_{\substack{\text{leave unused}\\\text{components absent}}}.
\label{eq:iclr-joint-assignment}
\end{equation}
As shown in Figure~\ref{fig:iclr-overview}(b), each assignment explains the targets and leaves unused components absent. ExactMB sums their densities to evaluate the multi-Bernoulli likelihood:
\begin{equation}
\boxed{f(\mathcal T)=\sum_{\phi\in\operatorname{Inj}(M,K)}f(\mathcal T,\phi),
\qquad \mathcal L(\mathcal T)=-\log f(\mathcal T).}
\label{eq:iclr-likelihood}
\end{equation}
The sum is invariant to target enumeration. Each explanation gives distinct targets distinct owners, while allowing coincident component centers. Absence factors score unused components, yielding $f(\varnothing)=\prod_j(1-q_j)$ for the empty set. Absence denotes visible-set nonmembership, not physical nonexistence. ExactMB thus jointly scores localization and membership without auxiliary penalties. Appendix~\ref{sec:normalization-details} defines the finite-set measure and proves normalization and log-score propriety.

\paragraph{Training and inference.}
We train the reference model by averaging the ExactMB loss in Equation~\eqref{eq:iclr-likelihood} over image rays. At inference, we threshold the predicted existence maps, then sort and de-duplicate the retained depths following SeeGroup's released evaluation code \citep{seegroup2026}.

\subsection{How assignment supervises depth and presence}
\label{sec:iclr-geometric-ownership}
Marginalized assignments determine how each component learns. Normalized explanation weights give the probability $P_{ji}$ that component $j$ explains target $i$. Posterior occupancy $\rho_j$ sums these responsibilities over targets:
\begin{equation}
\omega_\phi=\frac{f(\mathcal T,\phi)}{f(\mathcal T)},\qquad
P_{ji}=\sum_{\phi:\phi(i)=j}\omega_\phi,\qquad
\rho_j=\sum_iP_{ji}.
\label{eq:iclr-posterior}
\end{equation}
Here $\rho_j$ is the probability that component $j$ explains any target. Whereas $q_j$ is predicted from the image, posterior occupancy also uses observed depths and competing components. Differentiating the finite marginal likelihood gives the presence gradient
\begin{equation}
\frac{\partial\mathcal L}{\partial e_j}=q_j-\rho_j,
\label{eq:iclr-gradients}
\end{equation}
and the geometric gradients
\begin{equation}
\frac{\partial\mathcal L}{\partial d_j}=\sum_{i=1}^{M}P_{ji}
\frac{\operatorname{sign}(d_j-t_i)}{\beta_j},\qquad
\frac{\partial\mathcal L}{\partial\beta_j}=\sum_{i=1}^{M}P_{ji}
\left(\frac1{\beta_j}-\frac{|d_j-t_i|}{\beta_j^2}\right).
\label{eq:iclr-geometric-gradients}
\end{equation}
The same assignment responsibilities supervise depth and presence: each component earns presence credit by plausibly explaining an observed surface. Output derivatives use an absolute-value subgradient at zero residual. Raw-head gradients also differentiate through Softplus and scale clipping.

As shown in Figure~\ref{fig:iclr-overview}(c), presence credits sum to the annotated target count, $\sum_j\rho_j=M$. Summing the presence gradients therefore gives
\begin{equation}
\sum_j\frac{\partial\mathcal L}{\partial e_j}
=\sum_jq_j-M=\mathbb E|\mathsf X|-|\mathcal T|.
\label{eq:iclr-mass}
\end{equation}
The aggregate signal leaves component--surface assignment unresolved. Each assignment selects an occupied subset and pairs it with targets. At $M=1$, within-subset correspondence is unique, leaving only selection ambiguous. At $M=K$, the subset is fixed and every $\rho_j=1$ despite uncertain correspondence. Intermediate counts involve both decisions.

ExactMB scores observed cardinality jointly with geometry conditional on count:
\begin{equation}
\mathcal L(\mathcal T)
=\underbrace{-\log\Pr(|\mathsf X|=M)}_{\text{observed cardinality}}
+\underbrace{\bigl[-\log f(\mathcal T\mid |\mathsf X|=M)\bigr]}_{\text{geometry at that cardinality}}.
\label{eq:iclr-count-factorization}
\end{equation}
The factors share parameters: relative presence probabilities can affect conditional geometry. In contrast, for complete distinct targets with $0<M<K$, finite logits, fixed centers, and positive fixed scales, a common logit shift preserves posterior assignment probabilities. Appendix~\ref{sec:selection-curvature} proves that its unique minimum of the set loss matches expected count to $M$.

\section{Evaluating Visible Depth Layers}
\label{sec:iclr-evaluation}

Count and assignment must also be distinguished in evaluation: correct counts can hide missed surfaces and extra predictions, while correct ordering can coexist with incorrect presence. We distinguish whether a surface is recovered from how accurately its depth is estimated. We extend existing benchmarks to measure depth accuracy, recovered support, and overprediction together.

\subsection{Benchmarks and annotation conventions}
\label{sec:iclr-benchmarks}

\textbf{LayeredDepth (LD)} \citep{layereddepth2025} provides 300 real validation images (LD-Real) with sparse ordinal queries and 500 synthetic images (LD-Syn) with metric-depth channels L1/L3/L5/L7. Headline LD-Real evaluation uses valid and all-absent quadruplets that can span rays. \textbf{MultiDepth-3K (MD-3K)} complements these with 3,161 material-labeled point pairs \citep{twodepths2026}. We align its up-to-two-layer annotations with LD: L1 denotes foreground and L3 optional background. For presence evaluation, L1 is always valid, L3 only at transparent points, and unused slots are absent. Dense synthetic labels test metric depth and count, while sparse real labels test presence and ordering.

\subsection{Joint depth--presence evaluation}
\label{sec:iclr-recovery-metrics}

\textbf{Ordinal accuracy and presence.} LD-Real ACC requires all requested depths in a valid quadruplet to be present in annotated strict near-to-far order. Conversely, OverPred is the fraction of all-absent quadruplets with any requested prediction. On MD-3K, pattern ACC requires both points' validity patterns to match their targets. Recall and OverPred measure L3 emission at transparent and opaque points, respectively. Headline ordinal ACC requires correct foreground and background ordering on transparent--transparent (TT) and transparent--opaque (TO) pairs. Background order uses L3 at transparent points and L1 at opaque points, without filling missing L3. Joint ACC requires both pattern and ordinal correctness across all material-pair types.

\textbf{Metric depth and count.} On LD-Syn, decoded ranks 1--4 are paired with corresponding target channels. AbsRel averages relative depth error where target and prediction are both valid. Because omitting difficult surfaces can lower this conditional error, we also report target-support recall: the fraction of target-valid pixels with a valid prediction. Depth-accurate recall additionally requires $\max(\hat d/d,d/\hat d)<1.25$, while depth-accurate precision divides the same accurate matches by all prediction-valid pixels. Count ACC/MAE pool occupied target rays and retain supplied ties.

\textbf{Shared protocol and reporting.} Primary ablations use native selection followed by common geometric filters, without alignment or clipping. Summary AbsRel equally weights four conditional rank errors per seed. We report means and sample SD across seeds. Rank-wise scores complement fixed-ray set-recovery diagnostics. Rates and AbsRel are percentages, while count MAE is in layers.

\Needspace{6\baselineskip}
\suppressfloats[t]
\section{Experiments}
\label{sec:iclr-experiments}
\begin{table}[!t]
\centering
\setlength{\abovecaptionskip}{0pt}
\setlength{\belowcaptionskip}{4pt}
\caption{\textbf{Primary cardinality and assignment ablations.}
Separately trained variants (means$\pm$sample SD). All denotes depth stacking and Marginalized denotes ExactMB. Count reg.: explicit count regularization. Ordered: compacted annotation-channel assignment. Rows (iii--v) differ only in recorded assignment per seed.
Section~\ref{sec:iclr-evaluation} defines ACC/OverPred. AbsRel (\%) averages conditional errors over four ranks. MAE is in layers.
Markers match Figure~\ref{fig:iclr-behavior}(a).}
\label{tab:iclr-primary-ablation}
\begingroup
\fontsize{8}{9.5}\selectfont
\setlength{\tabcolsep}{1pt}
\renewcommand{\arraystretch}{1.12}
\newcommand{\PrimaryRowKey}[1]{\raisebox{0.25ex}{\hbox to 5pt{\hfil\pdfliteral{q 0.35 w 0.149020 0.137255 0.164706 RG #1 B Q}\hfil}}}
\newcommand{\PrimaryFixedKey}{\PrimaryRowKey{0.556863 0.541176 0.572549 rg -1.6 -1.6 m 1.6 -1.6 l 1.6 1.6 l -1.6 1.6 l h}}
\newcommand{\PrimaryCountKey}{\PrimaryRowKey{0.996078 0.737255 0.164706 rg 1.6 0 m 1.6 0.883656 0.883656 1.6 0 1.6 c -0.883656 1.6 -1.6 0.883656 -1.6 0 c -1.6 -0.883656 -0.883656 -1.6 0 -1.6 c 0.883656 -1.6 1.6 -0.883656 1.6 0 c h}}
\newcommand{\PrimaryMapKey}{\PrimaryRowKey{0.949020 0.411765 0.298039 rg 0 1.6 m -1.6 -1.6 l 1.6 -1.6 l h}}
\newcommand{\PrimaryMarginalKey}{\PrimaryRowKey{0.329412 0.007843 0.643137 rg 0 2.262742 m -2.262742 0 l 0 -2.262742 l 2.262742 0 l h}}
\newcommand{\PrimarySuppliedKey}{\PrimaryRowKey{0.717647 0.098039 0.584314 rg 0 -1.6 m -1.6 1.6 l 1.6 1.6 l h}}
\newcommand{\PrimaryScore}[2]{\ensuremath{#1_{\pm#2}}}
\begin{tabular*}{\linewidth}{@{\extracolsep{\fill}}lllrrrrrrrr@{}}
\toprule
& & & \multicolumn{2}{c}{\textbf{LD-Real}} & \multicolumn{3}{c}{\textbf{MD-3K}} & \multicolumn{3}{c}{\textbf{LD-Syn}} \\
\cmidrule(lr){4-5}\cmidrule(lr){6-8}\cmidrule(l){9-11}
& \textbf{Emission} & \textbf{Assignment} & \shortstack{\textbf{Ordinal}\\\textbf{ACC}$\uparrow$} & \textbf{OverPred}$\downarrow$ & \textbf{Recall}$\uparrow$ & \textbf{OverPred}$\downarrow$ & \shortstack{\textbf{Pattern}\\\textbf{ACC}$\uparrow$} & \textbf{AbsRel}$\downarrow$ & \shortstack{\textbf{Count}\\\textbf{ACC}$\uparrow$} & \textbf{MAE}$\downarrow$ \\
\midrule
\makebox[11pt][l]{(i)}\,\PrimaryFixedKey & All & Ordered & \PrimaryScore{67.4}{0.7} & \PrimaryScore{98.4}{0.3} & \PrimaryScore{100.0}{0.0} & \PrimaryScore{98.5}{0.2} & \PrimaryScore{0.03}{0.04} & \PrimaryScore{18.3}{0.1} & \PrimaryScore{3.7}{0.2} & \PrimaryScore{2.244}{0.031} \\
\makebox[11pt][l]{(ii)}\,\PrimaryCountKey & Count reg. & Ordered & \PrimaryScore{58.9}{0.9} & \PrimaryScore{7.4}{0.6} & \PrimaryScore{79.1}{1.0} & \PrimaryScore{8.6}{0.5} & \PrimaryScore{65.9}{1.2} & \PrimaryScore{17.2}{0.1} & \PrimaryScore{95.0}{0.0} & \PrimaryScore{0.067}{0.001} \\
\addlinespace[1.5pt]
\makebox[11pt][l]{(iii)}\,\PrimaryMapKey & Presence & MAP & \PrimaryScore{60.7}{1.2} & \PrimaryScore{12.1}{1.1} & \PrimaryScore{81.7}{0.7} & \PrimaryScore{20.9}{1.4} & \PrimaryScore{47.7}{2.7} & \PrimaryScore{18.7}{0.3} & \PrimaryScore{88.4}{0.2} & \PrimaryScore{0.183}{0.002} \\
\makebox[11pt][l]{(iv)}\,\PrimaryMarginalKey & Presence & Marginalized & \PrimaryScore{61.7}{1.8} & \PrimaryScore{11.6}{1.3} & \PrimaryScore{81.5}{2.9} & \PrimaryScore{19.2}{3.8} & \PrimaryScore{46.3}{4.0} & \PrimaryScore{19.2}{0.2} & \PrimaryScore{88.2}{0.5} & \PrimaryScore{0.194}{0.007} \\
\makebox[11pt][l]{(v)}\,\PrimarySuppliedKey & Presence & Ordered & \PrimaryScore{63.2}{1.9} & \PrimaryScore{10.1}{1.2} & \PrimaryScore{82.6}{3.9} & \PrimaryScore{21.7}{1.2} & \PrimaryScore{52.4}{3.1} & \PrimaryScore{17.6}{0.2} & \PrimaryScore{89.7}{0.1} & \PrimaryScore{0.154}{0.002} \\
\bottomrule
\end{tabular*}
\endgroup
\end{table}

\begin{table}[!t]
\centering
\setlength{\abovecaptionskip}{0pt}
\setlength{\belowcaptionskip}{3pt}
\caption{\textbf{Released-model transfer references.}
Single estimates from unretrained models for different tasks. Scale denotes native output.
$\dagger$~AbsRel (\%) averages four conditional rank errors after per-image/layer GT least-squares calibration and clipping; primary errors are unaligned.}
\label{tab:iclr-external}
\begingroup
\fontsize{8}{9.2}\selectfont
\setlength{\tabcolsep}{1pt}
\renewcommand{\arraystretch}{1.0}
\begin{tabular*}{\linewidth}{@{\extracolsep{\fill}}llrrrrrrrr@{}}
\toprule
 & & \multicolumn{2}{c}{\textbf{LD-Real}} & \multicolumn{3}{c}{\textbf{MD-3K}} & \multicolumn{3}{c}{\textbf{LD-Syn}} \\
\cmidrule(lr){3-4}\cmidrule(lr){5-7}\cmidrule(l){8-10}
\textbf{Released model} & \textbf{Scale} & \shortstack{\textbf{Ordinal}\\\textbf{ACC}$\uparrow$} & \textbf{OverPred}$\downarrow$ & \textbf{Recall}$\uparrow$ & \textbf{OverPred}$\downarrow$ & \shortstack{\textbf{Pattern}\\\textbf{ACC}$\uparrow$} & \shortstack{\textbf{GT-cal.}\\\textbf{AbsRel}$^{\dagger}\!\downarrow$} & \shortstack{\textbf{Count}\\\textbf{ACC}$\uparrow$} & \textbf{MAE}$\downarrow$ \\
\midrule
SeeGroup \citep{seegroup2026} & Relative & 74.2 & 100.0 & 100.0 & 100.0 & 0.0 & 15.8 & 2.4 & 2.756 \\
WT r69l \citep{worldtracing2026} & Relative & 58.2 & 100.0 & 100.0 & 100.0 & 0.0 & 24.5 & 2.4 & 2.789 \\
LaRI scenes \citep{lari2026} & Relative & 52.0 & 100.0 & 100.0 & 100.0 & 0.0 & 29.5 & 2.4 & 2.789 \\
LaRI objects \citep{lari2026} & Relative & 12.9 & 50.2 & 79.1 & 80.5 & 38.1 & 37.7 & 33.1 & 0.750 \\
WT r75b \citep{worldtracing2026} & Metric & 17.0 & 100.0 & 100.0 & 100.0 & 0.0 & 46.1 & 2.4 & 2.789 \\
\bottomrule
\end{tabular*}
\endgroup
\end{table}

We first compare ExactMB with depth stacking and released models, then examine count and assignment choices. Primary ablations and a separate matched-recipe study use four seeds per configuration, reporting means$\pm$sample SD and seed-paired differences.

\Needspace{12\baselineskip}
\paragraph{Main ablation setup.}
\label{sec:iclr-implementation}
\label{sec:iclr-predictor}
Table~\ref{tab:iclr-primary-ablation} compares:
\textbf{(i) Depth stacking} regresses supplied channels with SeeGroup's decoder, selecting all $K$ slots without absence supervision. It differs from native SeeGroup's many-to-one maximum-density coverage.
\textbf{(ii) Explicit count regularization} adds categorical count supervision, averages valid-entry depth loss, and selects the argmax-count prefix.
\textbf{(iii) MAP assignment} uses the highest-scoring injective Bernoulli--Laplace explanation.
\textbf{(iv) ExactMB} marginalizes all injective explanations without fixed assignment constraints.
\textbf{(v) Ordered assignment} fixes matches in compacted annotation-channel order, marking matched components present and others absent.
Rows (i--ii) compare complete constant-gate configurations differing in count supervision, depth-loss reduction, and selection. Rows (iii--v) share presence-gated recurrence and the remaining recorded training settings, differing only in their assignment criterion within each seed.

\paragraph{Implementation.}
All variants use 14,800 synthetic training images, $K=4$, terminal-800k checkpoints, and zero geometric auxiliary weights. Real benchmarks assess synthetic-to-real transfer. They share metric Depth Anything V2-L initialization \citep{depthanythingv2} and SeeGroup's decoder \citep{seegroup2026}. Presence-based variants select $q_j>.01$. All variants then sort finite depths above $.02$\,m and retain centers more than $.02$\,m beyond the last retained depth. These filters can reduce even the depth-stacking count. Further details appear in Appendix~\ref{sec:implementation-details}.

\subsection{Joint depth--presence modeling reduces overprediction}
\label{sec:iclr-behavior}
\begin{figure}[!t]
\centering
\setlength{\abovecaptionskip}{4pt}
\begin{minipage}[t]{.44\linewidth}
\vspace{0pt}
\raggedright\fontfamily{phv}\fontsize{8.5}{10}\selectfont
\textbf{a\quad Count accuracy}\hfill{\fontsize{8}{9}\selectfont\color{black!65}LD-Syn}\par\vspace{3pt}
\includegraphics[page=1,width=\linewidth]{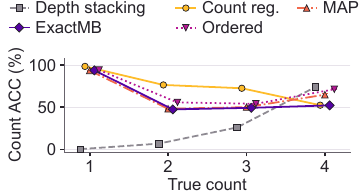}
\end{minipage}\hfill
\begin{minipage}[t]{.54\linewidth}
\vspace{0pt}
\raggedright\fontfamily{phv}\fontsize{8.5}{10}\selectfont
\textbf{b\quad Paired effects}\hfill{\fontsize{8}{9}\selectfont\color{black!65}Separate recipe}\par\vspace{3pt}
\begingroup
\fontsize{8}{10}\selectfont
\setlength{\tabcolsep}{0pt}
\renewcommand{\arraystretch}{1.05}
\newcommand{\EffectCell}[2]{#1\thinspace\textpm\thinspace#2}
\arrayrulecolor{black!25}
\begin{tabular*}{\linewidth}{@{\extracolsep{\fill}}lcc@{}}
Metric & PPP $\rightarrow$ MAP & MAP $\rightarrow$ ExactMB\\[3pt]
\midrule[.3pt]
\addlinespace[3pt]
\makecell[l]{MD pattern ACC\\\textcolor{black!65}{gain (pp)}} & \EffectCell{+14.39}{2.84} & \EffectCell{+1.96}{1.37}\\
\addlinespace[5pt]
\makecell[l]{MD OverPred\\\textcolor{black!65}{reduction (pp)}} & \EffectCell{-0.32}{1.59} & \EffectCell{+5.31}{1.49}\\
\addlinespace[5pt]
\makecell[l]{LD-Syn count MAE\\\textcolor{black!65}{reduction (layers)}} & \EffectCell{+0.1211}{0.0030} & \EffectCell{-0.0069}{0.0030}\\
\end{tabular*}
\endgroup
\end{minipage}
\caption{\textbf{Count and presence metrics favor different modeling choices.}
(a) Primary-study LD-Syn Count ACC by true count.
(b) Separate-recipe paired effects. Positive values indicate improvement.
ExactMB denotes marginalization. pp denotes percentage points. Means$\pm$sample SD.}
\label{fig:iclr-behavior}
\end{figure}

\begin{figure}[t]
\centering
\setlength{\abovecaptionskip}{4pt}
\includegraphics[page=17,width=\linewidth]{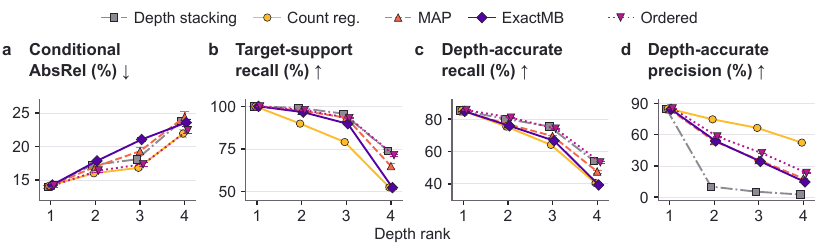}
\caption{\textbf{Evaluating depth quality requires both localization and recovered support.}
The same five LD-Syn ablations, means$\pm$sample SD.
(a) AbsRel on jointly valid target--prediction pixels. (b) Target-support recall. (c) Depth-accurate recall. (d) Depth-accurate precision. Section~\ref{sec:iclr-evaluation} defines these metrics.
Decoded ranks 1--4 are paired with supplied target channels L1/L3/L5/L7. Principal outputs are unaligned and unclipped.
Full rank-wise scores and denominators are in Appendix~\ref{sec:primary-geometry}.}
\label{fig:iclr-primary-geometry}
\end{figure}

\textbf{Auxiliary-free ExactMB substantially reduces overprediction.}
As shown in Table~\ref{tab:iclr-primary-ablation}, ExactMB lowers LD-Real OverPred from $98.4\%$ to $11.6\%$ and MD-3K OverPred from $98.5\%$ to $19.2\%$ compared with depth stacking. It retains most LD-Real ordinal accuracy ($61.7\%$ versus $67.4\%$), with LD-Syn conditional AbsRel of $19.2\%$ versus $18.3\%$. Joint depth--presence learning thus suppresses false emissions with similar conditional depth accuracy, without auxiliary count or geometric losses.

\textbf{Visible-set membership matters even with learned stopping.}
The released-model results in Table~\ref{tab:iclr-external} show why optional emission need not recover visible membership. LaRI objects allows absence via its ray-stop mask, yet LD-Real and MD-3K OverPred are $50.2\%$ and $80.5\%$, versus ExactMB's $11.6\%$ and $19.2\%$. MD recall is comparable ($79.1\%$ versus $81.5\%$). These transfers complement the shared-training comparison: released models target different tasks and retain native adapters.

\textbf{Count regularization and ordered assignment further shape recovery.}
The count-regularized configuration exceeds ExactMB on MD pattern ACC ($65.9\%$ versus $46.3\%$) and synthetic count ACC ($95.0\%$ versus $88.2\%$). One-return rays dominate the latter: an always-one predictor achieves $87.49\%$ micro count ACC and $.2112$-layer MAE on occupied LD-Syn rays. Figure~\ref{fig:iclr-behavior}(a) stratifies the results by cardinality. Count regularization leads on one- through three-return rays, while ordered assignment reaches $71.4\%$ on four-return rays versus ExactMB's $52.3\%$. Stratified scores therefore test multilayer counting beyond performance on the dominant one-return case.

\textbf{Assignment marginalization shapes recovery trade-offs.}
Figure~\ref{fig:iclr-behavior}(b) compares a Poisson point-process (PPP) reference, MAP, and ExactMB under a separate recipe. MAP$\rightarrow$ExactMB lowers MD OverPred by $5.31\pm1.49$ points and raises pattern ACC by $1.96\pm1.37$, while count MAE worsens by $.0069\pm.0030$ layers. The primary-study pattern-ACC change is $-1.35\pm6.69$ points: marginalization's benefits depend on the recipe and evaluation metric.

\subsection{Depth accuracy must be read with recovered support}
\label{sec:iclr-set-quality}
\suppressfloats[t]
\begin{figure}[t]
\centering
\setlength{\abovecaptionskip}{4pt}
\includegraphics[page=16,width=\linewidth]{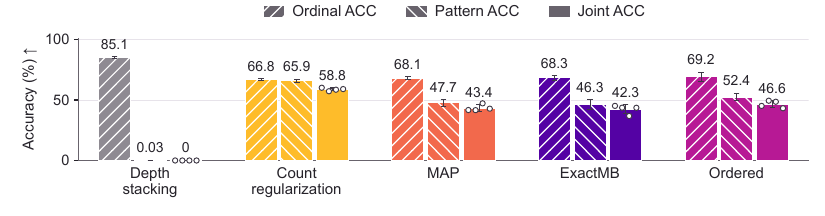}
\caption{\textbf{Correct ordering does not guarantee correct visible-surface membership.}
All 3,161 MD-3K pairs; principal operating point. Joint ACC requires correct presence and ordering against the annotations. Bars/whiskers: four-seed means/sample SD; circles: per-seed Joint ACC.}
\label{fig:iclr-set-quality}
\end{figure}

\begin{figure}[!t]
\centering
\setlength{\abovecaptionskip}{4pt}
\includegraphics[page=30,width=\linewidth]{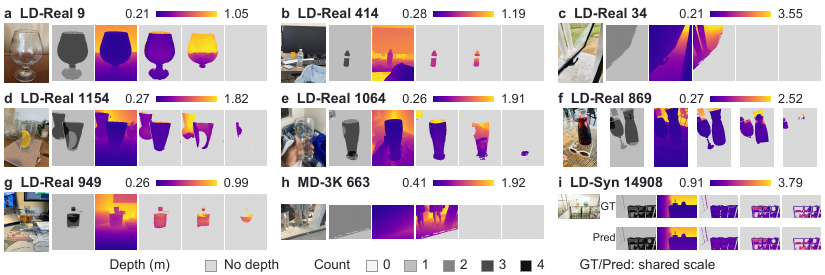}
\caption{\textbf{Visible depth layers with spatially varying support.}
Selected auxiliary-regularized ExactMB predictions. Columns: RGB, count, front-to-back depth ranks 1--4.
(a--h) Real predictions. (i) Synthetic GT/prediction on a shared metric scale, without alignment.
Gray: absent or invalid depth. The supplementary \textbf{video demo} provides additional qualitative cases and clearer views.}
\label{fig:iclr-layered-depth}
\end{figure}

\textbf{Geometric evaluation distinguishes selective accuracy from recovery.}
Which surfaces remain accurately recovered at these lower emission rates? Figure~\ref{fig:iclr-primary-geometry} compares four ranks across the same 20 checkpoints and 500 LD-Syn images. Relative to depth stacking, ExactMB raises second-rank depth-accurate precision from $10.07\%$ to $54.04\%$, while recall falls from $80.20\%$ to $75.78\%$. The count-regularized configuration lowers conditional errors relative to marginalization, but second- and third-rank target-support recall falls by $6.71\pm.26$ and $10.68\pm.28$ points. Lower conditional error can therefore accompany selective omission rather than more complete recovery.

\textbf{Ordered assignment improves deeper geometry and support.}
Relative to ExactMB, ordered assignment lowers conditional AbsRel and increases target-support recall at ranks 2--4 in every seed. Fourth-rank depth-accurate recall and precision improve by $13.95\pm.58$ and $8.76\pm.59$ points. These joint gains indicate more accurate recovery, not just more emitted depths. In contrast, count regularization attains higher deeper-rank precision than ordered assignment at lower recall, though it also exceeds marginalization in fourth-rank recall and precision. Ordered assignment's geometric gains accompany an increase in MD OverPred from $19.23\%$ to $21.65\%$. These configurations yield different precision--recall operating points for deeper-layer recovery.

\textbf{Correct ordering can coexist with incorrect membership.}
Figure~\ref{fig:iclr-set-quality} reports ordinal results on all 3,161 MD pairs. Depth stacking obtains $85.13\%$ ordinal ACC but $0\%$ Joint ACC. ExactMB, ordered assignment, and the count-regularized configuration attain $42.34\%$, $46.62\%$, and $58.79\%$ Joint ACC, respectively. As detailed in Appendix~\ref{sec:exp-joint-errors}, ordered assignment's gain over ExactMB reflects better presence-pattern recovery despite lower ordering accuracy conditional on correct presence. Joint ACC measures annotated presence and ordering together, not metric-set recovery.

\Needspace{6\baselineskip}
\subsection{Multilayer 3D reconstruction and applications}
\label{sec:iclr-qualitative-evidence}
\begin{figure}[!t]
\centering
\setlength{\abovecaptionskip}{4pt}
\includegraphics[page=22,width=\linewidth]{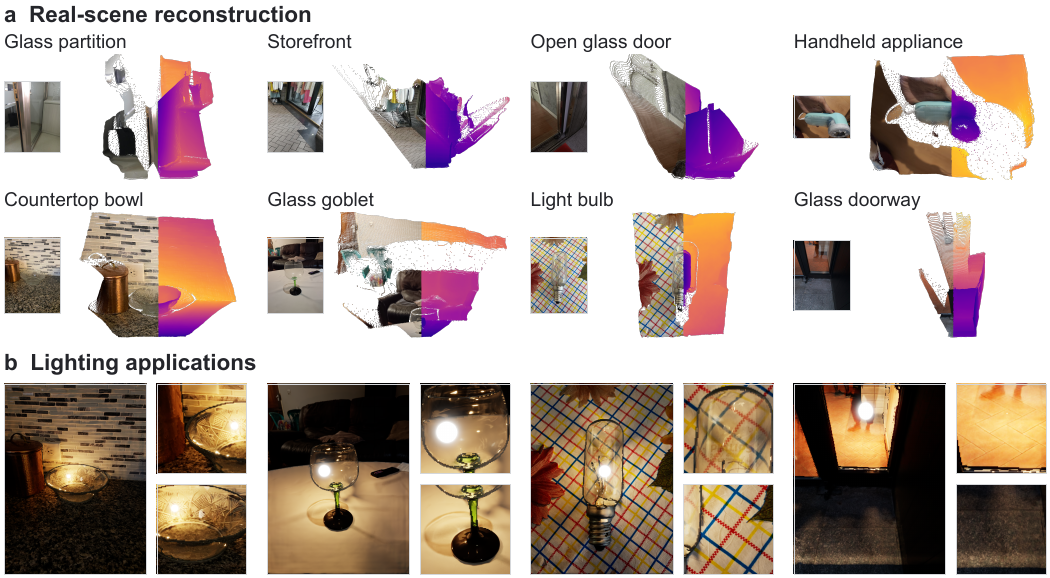}
\caption{\textbf{Layered 3D reconstruction for relighting.}
(a) Textured and depth-colored unprojections (purple near, yellow far, per-scene scales).
(b) Inserted-light re-rendering and detail crops.}
\label{fig:iclr-showcase}
\end{figure}

\begin{figure}[!t]
\centering
\setlength{\abovecaptionskip}{4pt}
\includegraphics[page=31,width=\linewidth]{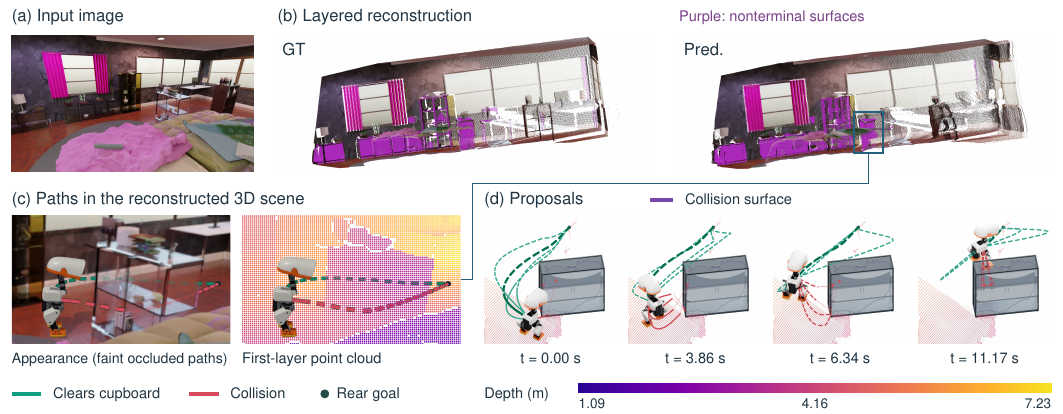}
\caption{\textbf{Local route screening with reconstructed surfaces.}
(a) Monocular input. (b) Layered reconstructions.
(c) MicroDuck routes \citep{microduck2026} clear or intersect selected first-layer facets (green/red).
(d) Rejected (red) and feasible (green) rollouts.}
\label{fig:iclr-navigation}
\end{figure}

ExactMB's retained depth sets support layered 3D reconstruction. Figure~\ref{fig:iclr-layered-depth} shows ExactMB depth and spatial support, including out-of-distribution MD-3K and LD-Real examples after synthetic-only task training. Unprojecting only retained depths carries ray-adaptive membership into 3D: foreground and background surfaces can coexist without turning unused components into additional mesh layers. Figure~\ref{fig:iclr-showcase} shows the resulting RGB-textured meshes rendered with inserted lights and assigned thin-glass materials. These pinhole reconstructions offer a starting point for calibrated refractive ray models~\citep{agrawal2012refractive}. Beyond rendering, Figure~\ref{fig:iclr-navigation} shows candidate robot trajectories screened against selected first-layer surfaces of the reconstructed partial scene. Red and green indicate rejected and locally feasible simulated rollouts. Extending this local geometric test to closed-loop control and whole-scene safety verification is a promising direction for future work.

\Needspace{8\baselineskip}
\section{Conclusion}
\label{sec:iclr-conclusion}

Depth Any Seen learns which visible surfaces are present and where they lie on a shared metric scale through ExactMB's joint depth--presence likelihood. Our analysis distinguishes expected-count correction from component--surface assignment. Experiments show that conditional depth accuracy and correct ordering can conceal incomplete or excessive predictions, motivating joint evaluation of localization, recovered support, and overprediction. The central goal is to recover the right surfaces at the right depths, rather than merely produce additional depth maps.

\paragraph{Limitations and future work.}
Our synthetic-only task training leaves room for stronger real-world generalization through mixed synthetic--real supervision. We believe that extending our image-based layered 3D prediction to video and streaming reconstruction would offer promising next steps toward temporally consistent recovery \citep{depthcrafter2025,streamvggt2026,lingbotmap2026}. We hope this work sheds light on joint depth--presence learning for recovering visible 3D structure.

\FloatBarrier
\label{iclr:main-text-end}
\clearpage
\subsection*{AI use statement}
Generative AI assistants (Codex) supported manuscript preparation, grammar checking, language refinement, consistency checks, and code for analysis and figure layout. No reported experimental results or plotted values were synthesized by generative AI assistants.

\subsection*{Ethics statement}
For our quantitative benchmarking, we adapt existing datasets to evaluate visible multilayer depth and presence, without collecting new data. Our method aims to improve the reliability of multilayer 3D representations and support safer downstream tasks, including robot navigation. Future work should validate real-world deployment safety and downstream control.

\subsection*{Reproducibility statement}
The appendix details the formulation, implementation, evaluation protocols, and study populations. Primary comparisons use multi-seed experiments, with single-seed analyses identified separately. Training histories and experiment records are archived in Weights \& Biases (W\&B). We will publicly release all our code and model checkpoints with the camera-ready paper to support reproducibility.

\bibliographystyle{iclr2027_conference}
\bibliography{bib/main}

@inproceedings{mvsnet2018,
  title={Mvsnet: Depth inference for unstructured multi-view stereo},
  author={Yao, Yao and Luo, Zixin and Li, Shiwei and Fang, Tian and Quan, Long},
  booktitle={European conference on computer vision},
  pages={785--801},
  year={2018},
  organization={Springer}
}

@inproceedings{seegroup2026,
  title         = {{SeeGroup}: Multi-Layer Depth Estimation of Transparent Surfaces via Self-Determined Grouping},
  author        = {Wen, Hongyu and Deng, Jia},
  booktitle     = {Proceedings of the IEEE/CVF Conference on Computer Vision and Pattern Recognition},
  year          = {2026},
}

@inproceedings{layereddepth2025,
  title={Seeing and seeing through the glass: Real and synthetic data for multi-layer depth estimation},
  author={Wen, Hongyu and Zuo, Yiming and Subramanian, Venkat and Chen, Patrick and Deng, Jia},
  booktitle={2025 IEEE/CVF International Conference on Computer Vision (ICCV)},
  pages={6715--6725},
  year={2025},
  organization={IEEE}
}

@inproceedings{depthfocus2026,
  title={DepthFocus: Controllable depth estimation for see-through scenes},
  author={Min, Junhong and Kim, Jimin and Kim, Minwook and Min, Cheol-Hui and Jeon, Youngpil and Choi, Minyong},
  booktitle={Proceedings of the IEEE/CVF Conference on Computer Vision and Pattern Recognition},
  pages={12595--12605},
  year={2026}
}

@inproceedings{seeclear2026,
  title={SeeClear: Reliable transparent object depth estimation via generative opacification},
  author={Wang, Xiaoying and He, Yumeng and Shi, Jingkai and Lu, Jiayin and Yang, Yin and Jiang, Ying and Jiang, Chenfanfu},
  booktitle={European Conference on Computer Vision},
  pages={75--93},
  year={2026},
  organization={Springer}
}

@inproceedings{twodepths2026,
  title         = {One Scene, Two Depths: Probing Geometric Ambiguity in Monocular Foundation Models},
  author        = {Xu, Xiaohao and Xue, Feng and Li, Xiang and Li, Haowei and Yang, Shusheng and Zhang, Tianyi and Johnson-Roberson, Matthew and Huang, Xiaonan},
  booktitle     = {European Conference on Computer Vision},
  year          = {2026},
  organization={Springer}
}

@inproceedings{lari2026,
  title         = {{LaRI}: Layered Ray Intersections for Single-view {3D} Geometric Reasoning},
  author        = {Li, Rui and Zhang, Biao and Li, Zhenyu and Tombari, Federico and Wonka, Peter},
  booktitle     = {Proceedings of the 43rd International Conference on Machine Learning},
  series        = {Proceedings of Machine Learning Research},
  publisher     = {PMLR},
  year          = {2026},
}

@article{worldtracing2026,
  title={World Tracing: Generative Pixel-Aligned Geometry Beyond the Visible},
  author={Zhang, Hao and Banani, Mohamed El and Cheng, Jen-Hao and Zhang, Paul and Hua, Yi and Mildenhall, Ben and Lassner, Christoph and Ahuja, Narendra and Yang, Gengshan},
  journal={arXiv preprint arXiv:2606.13652},
  year={2026}
}

@article{bian2026ambiguity,
  title={Modeling Depth Ambiguity: A Mixture-Density Representation for Flying-Point-Free Depth Estimation},
  author={Bian, Siyuan and Xu, Congrong and Gao, Jun},
  journal={arXiv preprint arXiv:2606.02552},
  year={2026}
}

@inproceedings{hess2022object,
  title={Object detection as probabilistic set prediction},
  author={Hess, Georg and Petersson, Christoffer and Svensson, Lennart},
  booktitle={European Conference on Computer Vision},
  pages={550--566},
  year={2022},
  organization={Springer}
}

@article{depthanythingv2,
  title={Depth anything v2},
  author={Yang, Lihe and Kang, Bingyi and Huang, Zilong and Zhao, Zhen and Xu, Xiaogang and Feng, Jiashi and Zhao, Hengshuang},
  journal={Advances in neural information processing systems},
  volume={37},
  pages={21875--21911},
  year={2024}
}

@inproceedings{depthcrafter2025,
  title={Depthcrafter: Generating consistent long depth sequences for open-world videos},
  author={Hu, Wenbo and Gao, Xiangjun and Li, Xiaoyu and Zhao, Sijie and Cun, Xiaodong and Zhang, Yong and Quan, Long and Shan, Ying},
  booktitle={2025 IEEE/CVF Conference on Computer Vision and Pattern Recognition (CVPR)},
  pages={2005--2015},
  year={2025},
  organization={IEEE}
}

@inproceedings{streamvggt2026,
  title={Streaming visual geometry transformer},
  author={Zhuo, Dong and Zheng, Wenzhao and Guo, Jiahe and Wu, Yuqi and Zhou, Jie and Lu, Jiwen},
  booktitle={International Conference on Learning Representations},
  volume={2026},
  pages={88055--88072},
  year={2026}
}

@inproceedings{lingbotmap2026,
  title={Geometric context transformer for streaming 3d reconstruction},
  author={Chen, Lin-Zhuo and Gao, Jian and Chen, Yihang and Cheng, Ka Leong and Sun, Yipengjing and Hu, Liangxiao and Xue, Nan and Zhu, Xing and Shen, Yujun and Yao, Yao and others},
  booktitle     = {European Conference on Computer Vision},
  year          = {2026}
}

@inproceedings{vggt2025,
  title={Vggt: Visual geometry grounded transformer},
  author={Wang, Jianyuan and Chen, Minghao and Karaev, Nikita and Vedaldi, Andrea and Rupprecht, Christian and Novotny, David},
  booktitle={2025 IEEE/CVF Conference on Computer Vision and Pattern Recognition (CVPR)},
  pages={5294--5306},
  year={2025},
  organization={IEEE}
}

@inproceedings{pi32026,
  title={{$\pi^3$}: Permutation-Equivariant Visual Geometry Learning},
  author={Wang, Yifan and Zhou, Jianjun and Zhu, Haoyi and Chang, Wenzheng and Zhou, Yang and Li, Zizun and Chen, Junyi and Pang, Jiangmiao and Shen, Chunhua and He, Tong},
  booktitle={International Conference on Learning Representations},
  volume={2026},
  pages={10481--10497},
  year={2026}
}

@inproceedings{dust3r2024,
  title={Dust3r: Geometric 3d vision made easy},
  author={Wang, Shuzhe and Leroy, Vincent and Cabon, Yohann and Chidlovskii, Boris and Revaud, Jerome},
  booktitle={2024 IEEE/CVF Conference on Computer Vision and Pattern Recognition (CVPR)},
  pages={20697--20709},
  year={2024},
  organization={IEEE}
}

@inproceedings{mast3r2024,
  title={Grounding image matching in 3d with mast3r},
  author={Leroy, Vincent and Cabon, Yohann and Revaud, J{\'e}r{\^o}me},
  booktitle={European conference on computer vision},
  pages={71--91},
  year={2024},
  organization={Springer}
}

@inproceedings{mapanything2026,
  title={Mapanything: Universal feed-forward metric 3d reconstruction},
  author={Keetha, Nikhil and M{\"u}ller, Norman and Sch{\"o}nberger, Johannes and Porzi, Lorenzo and Zhang, Yuchen and Fischer, Tobias and Knapitsch, Arno and Zauss, Duncan and Weber, Ethan and Antunes, Nelson and others},
  booktitle={2026 International Conference on 3D Vision (3DV)},
  pages={499--509},
  year={2026},
  organization={IEEE}
}

@inproceedings{trellis2025,
  title={Structured 3d latents for scalable and versatile 3d generation},
  author={Xiang, Jianfeng and Lv, Zelong and Xu, Sicheng and Deng, Yu and Wang, Ruicheng and Zhang, Bowen and Chen, Dong and Tong, Xin and Yang, Jiaolong},
  booktitle={2025 IEEE/CVF Conference on Computer Vision and Pattern Recognition (CVPR)},
  pages={21469--21480},
  year={2025},
  organization={IEEE}
}

@inproceedings{trellis22026,
  title={Native and compact structured latents for 3d generation},
  author={Xiang, Jianfeng and Chen, Xiaoxue and Xu, Sicheng and Wang, Ruicheng and Lv, Zelong and Deng, Yu and Zhu, Hongyuan and Dong, Yue and Zhao, Hao and Yuan, Nicholas Jing and others},
  booktitle={Proceedings of the IEEE/CVF Conference on Computer Vision and Pattern Recognition},
  pages={14419--14429},
  year={2026}
}

@inproceedings{sam3d2026,
  title={Sam 3d: 3dfy anything in images},
  author={Chen, Xingyu and Chu, Fu-Jen and Gleize, Pierre and Liang, Kevin J and Sax, Alexander and Tang, Hao and Wang, Weiyao and Guo, Michelle and Hardin, Thibaut and Li, Xiang and others},
  booktitle={Proceedings of the IEEE/CVF Conference on Computer Vision and Pattern Recognition},
  pages={7220--7232},
  year={2026}
}

@inproceedings{depthanything,
  title={Depth anything: Unleashing the power of large-scale unlabeled data},
  author={Yang, Lihe and Kang, Bingyi and Huang, Zilong and Xu, Xiaogang and Feng, Jiashi and Zhao, Hengshuang},
  booktitle={2024 IEEE/CVF Conference on Computer Vision and Pattern Recognition (CVPR)},
  pages={10371--10381},
  year={2024},
  organization={IEEE}
}

@inproceedings{ranftl2021,
  title     = {Vision Transformers for Dense Prediction},
  author    = {Ranftl, Ren{\'e} and Bochkovskiy, Alexey and Koltun, Vladlen},
  booktitle = {IEEE/CVF International Conference on Computer Vision (ICCV)},
  year      = {2021}
}

@article{dino,
title={{DINO}v2: Learning Robust Visual Features without Supervision},
author={Maxime Oquab and Timoth{\'e}e Darcet and Th{\'e}o Moutakanni and Huy V. Vo and Marc Szafraniec and Vasil Khalidov and Pierre Fernandez and Daniel HAZIZA and Francisco Massa and Alaaeldin El-Nouby and Mido Assran and Nicolas Ballas and Wojciech Galuba and Russell Howes and Po-Yao Huang and Shang-Wen Li and Ishan Misra and Michael Rabbat and Vasu Sharma and Gabriel Synnaeve and Hu Xu and Herve Jegou and Julien Mairal and Patrick Labatut and Armand Joulin and Piotr Bojanowski},
journal={Transactions on Machine Learning Research},
issn={2835-8856},
year={2024},
url={https://openreview.net/forum?id=a68SUt6zFt},
note={Featured Certification}
}

@INPROCEEDINGS{sajjan2020,
  author={Sajjan, Shreeyak and Moore, Matthew and Pan, Mike and Nagaraja, Ganesh and Lee, Johnny and Zeng, Andy and Song, Shuran},
  booktitle={2020 IEEE International Conference on Robotics and Automation (ICRA)}, 
  title={Clear Grasp: 3D Shape Estimation of Transparent Objects for Manipulation}, 
  year={2020},
  volume={},
  number={},
  pages={3634-3642},
  doi={10.1109/ICRA40945.2020.9197518}}

@inproceedings{shade1998,
author = {Shade, Jonathan and Gortler, Steven and He, Li-wei and Szeliski, Richard},
title = {Layered depth images},
year = {1998},
isbn = {0897919998},
publisher = {Association for Computing Machinery},
address = {New York, NY, USA},
url = {https://doi.org/10.1145/280814.280882},
doi = {10.1145/280814.280882},
booktitle = {Proceedings of the 25th Annual Conference on Computer Graphics and Interactive Techniques},
pages = {231–242},
numpages = {12},
series = {SIGGRAPH '98}
}

@article{gneiting2007proper,
author = {Tilmann Gneiting and Adrian E Raftery},
title = {Strictly Proper Scoring Rules, Prediction, and Estimation},
journal = {Journal of the American Statistical Association},
volume = {102},
number = {477},
pages = {359--378},
year = {2007},
publisher = {Taylor \& Francis},
doi = {10.1198/016214506000001437},


URL = { 
    
        https://doi.org/10.1198/016214506000001437
    
    

},
eprint = { 
    
        https://doi.org/10.1198/016214506000001437
    
    

}

}

@article{wainwright2008graphical,
  author = {Wainwright, Martin J. and Jordan, Michael I.},
  title = {Graphical Models, Exponential Families, and Variational Inference},
  journal = {Foundations and Trends in Machine Learning},
  volume = {1},
  number = {1--2},
  pages = {1--305},
  year = {2008},
  doi = {10.1561/2200000001}
}

@INPROCEEDINGS{shin2019scene,
  author={Shin, Daeyun and Ren, Zhile and Sudderth, Erik and Fowlkes, Charless},
  booktitle={2019 IEEE/CVF International Conference on Computer Vision (ICCV)}, 
  title={3D Scene Reconstruction With Multi-Layer Depth and Epipolar Transformers}, 
  year={2019},
  volume={},
  number={},
  pages={2172-2182},
  doi={10.1109/ICCV.2019.00226}}

@INPROCEEDINGS{dhamo2019object,
  author={Dhamo, Helisa and Navab, Nassir and Tombari, Federico},
  booktitle={2019 IEEE/CVF International Conference on Computer Vision (ICCV)}, 
  title={Object-Driven Multi-Layer Scene Decomposition From a Single Image}, 
  year={2019},
  volume={},
  number={},
  pages={5368-5377},
  doi={10.1109/ICCV.2019.00547}}

@INPROCEEDINGS{costanzino2023depth4tom,
  author={Costanzino, Alex and Ramirez, Pierluigi Zama and Poggi, Matteo and Tosi, Fabio and Mattoccia, Stefano and Di Stefano, Luigi},
  booktitle={2023 IEEE/CVF International Conference on Computer Vision (ICCV)}, 
  title={Learning Depth Estimation for Transparent and Mirror Surfaces}, 
  year={2023},
  volume={},
  number={},
  pages={9210-9221},
  doi={10.1109/ICCV51070.2023.00848}}

@INPROCEEDINGS{liu2025modest,
  author={Liu, Jiangyuan and Ma, Hongxuan and Guo, Yuxin and Zhao, Yuhao and Zhang, Chi and Sui, Wei and Zou, Wei},
  booktitle={2025 IEEE International Conference on Robotics and Automation (ICRA)}, 
  title={Monocular Depth Estimation and Segmentation for Transparent Object with Iterative Semantic and Geometric Fusion}, 
  year={2025},
  volume={},
  number={},
  pages={11162-11168},
  doi={10.1109/ICRA55743.2025.11128401}}

@INPROCEEDINGS{agrawal2012refractive,
  author={Agrawal, Amit and Ramalingam, Srikumar and Taguchi, Yuichi and Chari, Visesh},
  booktitle={2012 IEEE Conference on Computer Vision and Pattern Recognition}, 
  title={A theory of multi-layer flat refractive geometry}, 
  year={2012},
  volume={},
  number={},
  pages={3346-3353},
  doi={10.1109/CVPR.2012.6248073}}

@misc{microduck2026,
  author       = {{Pollen Robotics}},
  title        = {{Microduck Sandbox}},
  year         = {2026},
  howpublished = {Hugging Face Space},
  url          = {https://huggingface.co/spaces/pollen-robotics/microduck-simulator},
  note         = {Robot model and simulator assets, revision e81974b}
}

@inproceedings{tulsiani2018layer,
author = {Tulsiani, Shubham and Tucker, Richard and Snavely, Noah},
title = {Layer-Structured 3D Scene Inference via View Synthesis},
year = {2018},
isbn = {978-3-030-01233-5},
publisher = {Springer-Verlag},
address = {Berlin, Heidelberg},
url = {https://doi.org/10.1007/978-3-030-01234-2_19},
doi = {10.1007/978-3-030-01234-2_19},
booktitle = {Computer Vision – ECCV 2018: 15th European Conference, Munich, Germany, September 8–14, 2018, Proceedings, Part VII},
pages = {311–327},
numpages = {17},
location = {Munich, Germany}
}

@INPROCEEDINGS{shih2020photography,
  author={Shih, Meng-Li and Su, Shih-Yang and Kopf, Johannes and Huang, Jia-Bin},
  booktitle={2020 IEEE/CVF Conference on Computer Vision and Pattern Recognition (CVPR)}, 
  title={3D Photography Using Context-Aware Layered Depth Inpainting}, 
  year={2020},
  volume={},
  number={},
  pages={8025-8035},
  doi={10.1109/CVPR42600.2020.00805}}

@article{engel2024decomposition,
author = {Engel, Dominik and Hartwig, Sebastian and Ropinski, Timo},
title = {Monocular Depth Decomposition of Semi-Transparent Volume Renderings},
year = {2024},
issue_date = {July 2024},
publisher = {IEEE Educational Activities Department},
address = {USA},
volume = {30},
number = {7},
issn = {1077-2626},
url = {https://doi.org/10.1109/TVCG.2023.3245305},
doi = {10.1109/TVCG.2023.3245305},
journal = {IEEE Transactions on Visualization and Computer Graphics},
month = jul,
pages = {3981–3994},
numpages = {14}
}

@INPROCEEDINGS{zhu2021implicit,
  author={Zhu, Luyang and Mousavian, Arsalan and Xiang, Yu and Mazhar, Hammad and Eenbergen, Jozef van and Debnath, Shoubhik and Fox, Dieter},
  booktitle={2021 IEEE/CVF Conference on Computer Vision and Pattern Recognition (CVPR)}, 
  title={RGB-D Local Implicit Function for Depth Completion of Transparent Objects}, 
  year={2021},
  volume={},
  number={},
  pages={4647-4656},
  doi={10.1109/CVPR46437.2021.00462}}

@ARTICLE{fang2022transcg,
  author={Fang, Hongjie and Fang, Hao-Shu and Xu, Sheng and Lu, Cewu},
  journal={IEEE Robotics and Automation Letters}, 
  title={TransCG: A Large-Scale Real-World Dataset for Transparent Object Depth Completion and a Grasping Baseline}, 
  year={2022},
  volume={7},
  number={3},
  pages={7383-7390},
  doi={10.1109/LRA.2022.3183256}}

@INPROCEEDINGS{shi2024asgrasp,
  author={Shi, Jun and A, Yong and Jin, Yixiang and Li, Dingzhe and Niu, Haoyu and Jin, Zhezhu and Wang, He},
  booktitle={2024 IEEE International Conference on Robotics and Automation (ICRA)}, 
  title={ASGrasp: Generalizable Transparent Object Reconstruction and 6-DoF Grasp Detection from RGB-D Active Stereo Camera}, 
  year={2024},
  volume={},
  number={},
  pages={5441-5447},
  doi={10.1109/ICRA57147.2024.10611152}}
\clearpage
\appendix
\setcounter{figure}{0}
\setcounter{table}{0}
\setcounter{equation}{0}
\setcounter{proposition}{0}
\setcounter{theorem}{0}
\setcounter{remark}{0}
\renewcommand{\thefigure}{S\arabic{figure}}
\renewcommand{\thetable}{S\arabic{table}}
\renewcommand{\theequation}{S\arabic{equation}}
\renewcommand{\theproposition}{S\arabic{proposition}}
\renewcommand{\thetheorem}{S\arabic{theorem}}
\renewcommand{\theremark}{S\arabic{remark}}
\renewcommand{\theHfigure}{supp.\arabic{figure}}
\renewcommand{\theHtable}{supp.\arabic{table}}
\renewcommand{\theHequation}{supp.\arabic{equation}}
\renewcommand{\theHproposition}{supp.\arabic{proposition}}
\renewcommand{\theHtheorem}{supp.\arabic{theorem}}
\renewcommand{\theHremark}{supp.\arabic{remark}}
\suppressfloats[t]
\section*{Appendix contents}
\label{iclr:appendix-start}

The appendix connects probabilistic analysis, evaluation protocols, and recovery behavior.
Appendices~\ref{sec:selection-correspondence}--\ref{sec:selection-curvature} explain how count and assignment interact
through selection, correspondence, and specialization.
The evaluation definitions and study provenance in
Appendices~\ref{sec:evaluation-details}--\ref{sec:ablation-factors} support the primary ablations
and geometric recovery results in Appendices~\ref{sec:exp-behavior}--\ref{sec:primary-geometry}.
Appendix~\ref{sec:primary-geometry-operating-points} relates geometric operating points to retained support,
while Appendices~\ref{sec:exp-optimization}--\ref{sec:saved-ray-dynamics} trace count, activation,
uncertainty, and geometry during training. Appendix figures, tables, and equations use an S prefix.

\begingroup
\small
\newcommand{\AppIndexSection}[2]{%
  \par\addvspace{.3\baselineskip}\noindent
  \hyperref[#1]{\textbf{\ref*{#1}\quad #2}}\nobreak\hfill
  \hyperref[#1]{\textbf{\pageref*{#1}}}\par}
\newcommand{\AppIndexEntry}[2]{%
  \noindent\hspace*{1.4em}\hyperref[#1]{\makebox[2.8em][l]{\ref*{#1}}#2}\nobreak\dotfill
  \hyperref[#1]{\pageref*{#1}}\par}

\AppIndexSection{sec:formulation}{Model and probabilistic foundations}
\AppIndexEntry{sec:theory-checks}{Distributional and objective details}
\AppIndexEntry{sec:selection-correspondence}{Selecting components versus assigning surfaces}
\AppIndexEntry{sec:selection-curvature}{Mass correction and specialization are different directions}

\AppIndexSection{sec:reproducibility}{Implementation, evaluation, and reproducibility}
\AppIndexEntry{sec:implementation-details}{Reference implementation and training recipe}
\AppIndexEntry{sec:evaluation-details}{Evaluation definitions and adapters}
\AppIndexEntry{sec:ablation-factors}{Ablation factors and configuration provenance}
\AppIndexEntry{sec:synthetic-interface-audit}{Annotation semantics and the training interface}

\AppIndexSection{sec:experiments}{Additional experimental results}
\AppIndexEntry{sec:exp-behavior}{Primary ablations and released-model transfer}
\AppIndexEntry{sec:primary-geometry}{Primary four-seed geometric recovery}
\AppIndexEntry{sec:exp-error-structure}{Error structure across cardinalities and relations}
\AppIndexEntry{sec:exp-modeling}{Point-process and assignment controls}
\AppIndexEntry{sec:exp-exactmb-geometry}{Matched geometric-loss interventions}
\AppIndexEntry{sec:exp-architecture}{Encoder capacity and decoder design}
\AppIndexEntry{sec:exp-practical-boundaries}{Archived decoder operating points}
\AppIndexEntry{sec:exp-real-qualitative}{Additional real-world qualitative results}

\AppIndexSection{sec:iclr-count-geometry-diagnostics}{Learning dynamics and geometric recovery}
\AppIndexEntry{sec:exp-optimization}{Count, activation, and held-out learning}
\AppIndexEntry{sec:saved-ray-dynamics}{Fixed-ray uncertainty and geometric recovery}
\AppIndexEntry{sec:extended-analyses}{Cardinality weighting and retained support}

\AppIndexSection{sec:representation-details}{Output and membership conventions}
\endgroup
\clearpage

\begingroup
\newcommand{\AppendixTableSetup}{%
  \setlength{\abovecaptionskip}{0pt}%
  \setlength{\belowcaptionskip}{4pt}}
\newcommand{\AppendixTableBase}{%
  \color{black}\arrayrulecolor{black}%
  \setlength{\tabcolsep}{3pt}%
  \renewcommand{\arraystretch}{1.10}%
  \setlength{\heavyrulewidth}{.72pt}%
  \setlength{\lightrulewidth}{.42pt}%
  \setlength{\cmidrulewidth}{.42pt}}
\newcommand{\AppendixTextTableStyle}{%
  \AppendixTableBase\fontsize{9}{11}\selectfont}
\newcommand{\AppendixNumericTableStyle}{%
  \AppendixTableBase\fontsize{8.5}{10.2}\selectfont}
\newcommand{\AppendixTableNote}[1]{%
  \par\smallskip\noindent
  {\fontsize{8.5}{10.2}\selectfont\color{black}\raggedright #1\par}}
\definecolor{appendixRankBand}{gray}{.96}

\makeatletter
\setlength{\@fptop}{0pt}
\makeatother
\section{Model and Probabilistic Foundations}
\suppressfloats[t]
\label{sec:formulation}
\label{sec:ls-updated-method}

\subsection{Distributional and Objective Details}
\label{sec:method-theory-details}
\label{sec:theory-checks}

\subsubsection{What the Coverage Objective Encourages}
\label{sec:coverage-diagnosis-main}

For a nonempty target, the maximum-density coverage terms associated with SeeGroup~\citep{seegroup2026} take the form
\begin{align}
    \Ltarget
    &= -\sum_{i=1}^{M}\log\max_{j\in[K]}\ell_j(t_i),
    \label{eq:target-max}\\
    \Lcover
    &= -\sum_{j=1}^{K}\log\max_{i\in[M]}\ell_j(t_i),
    \label{eq:component-max}\\
    \Lbi
    &=\lambda_t\Ltarget+\lambda_c\Lcover,
    \quad \lambda_t,\lambda_c>0.
    \label{eq:bidirectional-coverage}
\end{align}
The target-side term permits targets to share components; the component-side term encourages every component to explain a target. With distinct targets, $1\leq M\leq K$, freely optimized ray-wise centers and scales, and $\beta_j\geq\beta_{\min}>0$, each density is bounded by $(2\beta_{\min})^{-1}$. Both losses reach their lower bounds by covering every target and centering every component on a target at the scale floor. Positive weights require every global minimum to attain both bounds, so surplus components duplicate target depths. Shared parameters, regularization, finite training, and filtering can alter this idealized behavior. Our depth-stacking baseline instead uses ordered depth regression.

\subsubsection{A Poisson Intensity Defines a Different Set Model}

Assignment constraints also distinguish ExactMB from Poisson set models. Intensity specifies only a first-order moment; a Poisson point-process assumption with finite total intensity gives the normalized density
\begin{equation}
 f_{\mathrm{PPP}}(\mathcal T)
 =\exp(-\Lambda)\prod_{t\in\mathcal T}v(t),
 \qquad \Lambda=\int v(t)\,\dd t.
 \label{eq:ppp-density}
\end{equation}
For $v(t)=\sum_j\lambda_j\ell_j(t)$, $\lambda_j\geq0$, expansion sums over all target-to-component assignments, allowing each component to explain multiple targets. This law has unbounded count and empty-event probability $\exp(-\Lambda)$. ExactMB bounds count through optional singletons and injective assignments.

\subsubsection{Reference Measure, Normalization, and Propriety}
\label{sec:normalization-details}

For ExactMB, let $\mathfrak F_{\leq K}(\DepthSpace)$ denote simple finite subsets of $\DepthSpace=\R$ of cardinality at most $K$. This domain accommodates untruncated densities despite positive targets and decoded depths. For symmetric nonnegative measurable or integrable $g$, define
\begin{equation}
 \int g(\mathcal T)\,\delta\mathcal T
 =g(\nullsym)+\sum_{M=1}^{K}\frac{1}{M!}
 \int_{\DepthSpace^M}g(\{t_1,\ldots,t_M\})\,\dd t_{1:M}.
 \label{eq:fisst-integral}
\end{equation}
The factorial removes repeated set enumerations. Under the fixed Lebesgue measure in meters, $M$-element densities have units $\mathrm m^{-M}$. Holding existing components fixed, appending a component with $q=0$ leaves the law unchanged; sigmoid logits approach this boundary as $e\to-\infty$.

\begin{proposition}[Multi-Bernoulli normalization]
\label{prop:normalization}
Under~\eqref{eq:fisst-integral}, density~\eqref{eq:iclr-likelihood} satisfies $\int f(\mathcal T)\,\delta\mathcal T=1$.
\end{proposition}
\begin{proof}
Fix $M$ and substitute~\eqref{eq:iclr-joint-assignment} into the integral. Each $\ell_j$ integrates to one. Every size-$M$ component subset admits $M!$ bijections from the integration variables, whose multiplicity is canceled by $1/M!$. The total mass at this cardinality is
\begin{equation}
 \sum_{\substack{S\subseteq[K]\\|S|=M}}
 \prod_{j\in S}q_j\prod_{j\notin S}(1-q_j).
 \label{eq:poisson-binomial-mass}
\end{equation}
Summing over $M=0{:}K$ enumerates each subset once and gives $\prod_j[(1-q_j)+q_j]=1$.
\end{proof}

Normalization gives the complete-set law the standard logarithmic-score identity~\citep{gneiting2007proper}:
\begin{theorem}[Finite-set logarithmic-score propriety]
\label{thm:proper-score}
Let $f^\star$ be a normalized finite-set density with finite entropy and let $f$ be normalized on the same space and reference measure. Whenever cross-entropy is finite,
\begin{equation}
 \E_{\mathcal T\sim f^\star}[-\log f(\mathcal T)]
 =H(f^\star)+\KL(f^\star\|f).
 \label{eq:proper-score}
\end{equation}
The expected score is uniquely minimized, up to equality almost everywhere, by $f=f^\star$. Within a restricted model family, any attained minimizer is a KL projection.
\end{theorem}
\begin{proof}
Add and subtract $-\int f^\star\log f^\star\,\delta\mathcal T$. The remainder is $\KL(f^\star\|f)\geq0$. If $f$ vanishes on positive $f^\star$ mass, cross-entropy is infinite. Under corresponding input-integrability conditions, the conditional identity can also be averaged over images and camera rays.
\end{proof}
Propriety characterizes the induced density, not component identifiability, calibration, or optimization. Auxiliary penalties, unequal existence weights, temperature changes, and target-dependent loss rescaling can modify this normalized score and require separate analysis.

\subsubsection{Count and Conditional Location Density}
\label{sec:count-factor-details}

Normalization also yields the count--geometry factorization. Continuous draws are distinct almost surely, so $N=|\mathsf X|=\sum_jR_j$. For $S\subseteq[K]$ of size $M$, define
\begin{equation}
\begin{aligned}
w_q(S)&=\prod_{j\in S}q_j\prod_{j\notin S}(1-q_j),\\
G_S(\mathcal T)&=
\sum_{\substack{\phi:[M]\hookrightarrow[K]\\\im(\phi)=S}}
\prod_{i=1}^{M}\ell_{\phi(i)}(t_i).
\end{aligned}
\label{eq:count-conditional-subset-terms}
\end{equation}
Grouping by occupied subset gives $f(\mathcal T)=\sum_{|S|=M}w_q(S)G_S(\mathcal T)$. The $M!$ terms in each $G_S$ cancel the cardinality-slice factor $1/M!$ on integration:
\begin{equation}
\begin{aligned}
\int_{|\mathcal T|=M}G_S(\mathcal T)\,\delta\mathcal T&=1,\\
p_q(M)=\Prob(N=M)&=\sum_{|S|=M}w_q(S).
\end{aligned}
\label{eq:count-slice-normalization}
\end{equation}
For $p_q(M)>0$, the conditional geometric density is
\begin{equation}
f(\mathcal T\mid N=M)
=\sum_{|S|=M}\frac{w_q(S)}{p_q(M)}G_S(\mathcal T),
\qquad |\mathcal T|=M.
\label{eq:conditional-cardinality-density}
\end{equation}
The weights sum to one, giving~\eqref{eq:iclr-count-factorization}. For $0<M<K$, relative presence weights can affect conditional geometry through the subset mixture. At $M=K$, only $S=[K]$ remains, eliminating dependence on $q$; at $M=0$, the conditional density is one. Conditioning is undefined if $p_q(M)=0$.

\subsection{Selecting Components Versus Assigning Surfaces}
\label{sec:selection-correspondence}

The subset mixture separates selection from correspondence. For a complete distinct target, $\Phi\sim\omega$ selects $S=\im(\Phi)$ with $|S|=M$. At $M=1$, within-subset correspondence is unique, leaving only selection ambiguity. At $M=K$, the subset is fixed despite $K!$ potentially plausible correspondences. Hence, for finite outputs under positive Laplace densities,
\begin{equation}
 M=K:\quad \rho_j=1,\qquad
 \frac{\partial\mathcal L}{\partial e_j}=q_j-1.
 \label{eq:full-capacity-presence}
\end{equation}
For $M=0$, occupancy is zero and the gradient is $q_j$. At identical outputs, MAP and ExactMB thus share full-capacity existence gradients; training can differ through geometry, recurrence, shared parameters, and other cardinalities. Complete four-target rays with four slots therefore have posterior presence targets of one despite uncertain depth-to-component correspondence.

\subsection{Mass Correction and Specialization Are Different Directions}
\label{sec:selection-curvature}

To distinguish count correction from component selection, fix locations and scales, and write $R_j(\Phi)=\mathbf1[j\in\im(\Phi)]$. Each assignment has log density
\[
 \begin{aligned}
 \log f(\mathcal T,\phi)&=\sum_j e_jR_j(\phi)-\sum_j\softplus(e_j)\\
 &\quad+\sum_i\log\ell_{\phi(i)}(t_i).
 \end{aligned}
\]
Every injection satisfies $\sum_jR_j=M$. A common finite-logit shift $e_j\mapsto e_j+a$ adds $Ma-\sum_j[\softplus(e_j+a)-\softplus(e_j)]$ to every log weight and preserves the entire ownership posterior. For the shifted loss,
\[
 \begin{aligned}
 \frac{\mathrm d\mathcal L(a)}{\mathrm da}&=\sum_j\sigma(e_j+a)-M,\\
 \frac{\mathrm d^2\mathcal L(a)}{\mathrm da^2}
 &=\sum_j\sigma(e_j+a)[1-\sigma(e_j+a)]>0.
 \end{aligned}
\]
For $0<M<K$, the derivative increases from $-M$ to $K-M$, yielding a unique finite optimum matching expected count without changing ownership. Fixed centers and positive scales suffice; symmetry is unnecessary. Shared-network updates need not follow this output-space direction.

Relative logit changes, however, can alter selection. The log-partition covariance identity~\citep{wainwright2008graphical} gives $\partial\rho_j/\partial e_k=\operatorname{Cov}_{\omega}(R_j,R_k)$, hence
\begin{equation}
 \nabla^2_{\mathbf e}\mathcal L
 =\operatorname{diag}\{q_j(1-q_j)\}
  -\operatorname{Cov}_{\omega}(\mathbf R).
 \label{eq:presence-hessian}
\end{equation}
\begin{proposition}[Fractional symmetry is a saddle with surplus capacity]
\label{prop:presence-saddle}
Suppose all slots have the same density value at each target, $0<M<K$, and $q_j=p=M/K$. The existence gradient vanishes. With $v=p(1-p)$, the Hessian eigenvalues are $v$ along $\mathbf1/\sqrt K$ and $-v/(K-1)$ along each of the $K-1$ orthogonal contrast directions between component logits.
\end{proposition}
\begin{proof}
The posterior over size-$M$ subsets is uniform, with $\rho_j=p$, $\operatorname{Var}(R_j)=v$, and off-diagonal covariance $-v/(K-1)$. The Hessian has zero diagonal and off-diagonal entries $v/(K-1)$. Its row sum is $v$, and its action on a zero-sum vector is multiplication by $-v/(K-1)$.
\end{proof}
At empty or full capacity, deterministic occupancy instead gives strict convexity in finite existence logits at fixed localization. The network parameterization and optimizer preconditioning determine how optimization in parameter space follows this output-space geometry.

\paragraph{Expected count, count probability, and gates.}
\label{sec:appendix-count-gate-pattern}
For decoding, distinguish expected count from the pre-thresholding, pre-merging distribution:
\begin{equation}
 p_q(n)=[z^n]\prod_{j=1}^{K}(1-q_j+q_jz).
 \label{eq:poisson-binomial-count}
\end{equation}
Four $.25$ gates have mean count one, count-one probability $27/64$, and variance $3/4$. A strict cutoff below $.25$ retains all four; otherwise none survive. Merging can remove coincident depths. For one target and $K>1$ identical component densities, the tied-gate likelihood maximum at $q=1/K$ is the independent-logit saddle above. Matching expected count therefore guarantees neither concentration of the pre-decoding count distribution nor correctness of the count after decoding.

\FloatBarrier
\clearpage
\section{Implementation, Evaluation, and Reproducibility}
\suppressfloats[t]
\label{sec:reproducibility}

We describe auxiliary-free ExactMB, evaluation, and annotation semantics, distinguishing primary ablations from point-process, geometric-loss, architecture, and fixed-ray studies.

\subsection{Reference Implementation and Training Recipe}
\label{sec:implementation-details}
\begin{table*}[t]
\centering
\AppendixTableSetup
\caption{\textbf{Auxiliary-free ExactMB configuration.} The configuration connects each implementation choice to its operational role. Variant-specific changes are stated with the corresponding comparisons. Section~\ref{sec:ablation-factors} describes the available historical configuration records and their limits.}
\label{tab:layerset-reference-recipe}
\AppendixTextTableStyle
\begin{tabularx}{\linewidth}{@{}>{\raggedright\arraybackslash}p{.19\linewidth}>{\raggedright\arraybackslash}X>{\raggedright\arraybackslash}X@{}}
\toprule
\textbf{Component} & \textbf{Reference choice} & \textbf{Operational meaning} \\
\midrule
Representation & $K=4$ unordered Bernoulli--Laplace components & Capacity bounds the emitted set; all four decoder passes execute. \\
Encoder / decoder & DINOv2 ViT-L \citep{dino}; metric Depth Anything V2 initialization; shared recurrent DPT \citep{ranftl2021} & One encoder, shared depth/scale/existence heads, absolute rather than cumulative depths. \\
Parameterization & $d=20\operatorname{softplus}(u)$; $q=\sigma(e)$;\newline $\beta=\operatorname{clip}_{[.1,10]}(.1+\operatorname{softplus}(v))$ & Camera-$Z$ meters. The factor $20$ is not a depth clamp. \\
Recurrence & Detached existence gate; learned removal strength initialized $.05$; update cap $.25$ & Bound~\eqref{eq:recurrence-bound} is relative to each feature token's norm. \\
Targets / mask & Compacted provided L1/L3/L5/L7 channels; element mask from positive valid depth & No numeric target sorting or deduplication; all-false masks contribute $M=0$ in the historical all-ray mean. \\
Complete criterion & ExactMB only; outer weight $1$, temperature $1$, existence weight $1$; no alignment & All geometric auxiliary weights are zero. Padded enumeration includes the $(K-M)!$ correction. \\
Data ordering & 14,800 training images; equal-weight four-bucket round robin without replacement & Interleaves remaining buckets without equalizing epoch mass or per-ray cardinality exposure. \\
Preprocessing & Aspect-preserving resize, $518\times518$ crop, $p=.5$ horizontal flip, ImageNet RGB normalization & Shared image/depth/mask geometry; metric units are retained. \\
Optimization & AdamW; $10^{-5}$ peak LR; encoder multiplier $.1$; moments $(.9,.999)$; decay $.01$; gradient-norm clipping $.5$ & Batch $2$, one recorded process; four training seeds $7/42/61/123$. \\
Schedule / endpoint & Warmup $5{,}000$; hold $127{,}200$; cosine decay to $.02$ of peak at $800{,}000$ & Terminal/latest checkpoint; no best-validation alias substitution in principal comparisons. \\
Principal decoder & Strict $q>.01$, finite $d>.02$\,m, ascending sort, $.02$\,m last-retained gap & No relative gap, scale/shift alignment, or ground-truth-dependent selection. \\
\bottomrule
\end{tabularx}
\end{table*}

\subsubsection{From RGB and annotations to slot parameters}
As summarized in Table~\ref{tab:layerset-reference-recipe}, the image is encoded once, and four recurrent passes predict absolute depth centers, Laplace scales, and existence logits:
\begin{equation}
 d_j=20\operatorname{softplus}(u_j),\qquad
 \beta_j=\operatorname{clip}_{[.1,10]}(.1+\operatorname{softplus}(v_j)),\qquad
 q_j=\sigma(e_j).
 \label{eq:reference-output-parameterization}
\end{equation}
Depth is camera optical-axis $Z$ in meters. The factor 20 scales rather than bounds the output; centers need not increase across passes. The loader converts millimeter PNGs to meters, zeros nonfinite, nonpositive, or above-80-m entries, and compacts valid L1/L3/L5/L7 channels in supplied order without sorting depths or deduplicating ties. RGB and depth/mask tensors share an aspect-preserving resize to cover a $518\times518$ crop with 14-pixel-compatible dimensions, followed by a shared random crop and horizontal flip with probability $.5$. RGB uses ImageNet normalization. Training and principal evaluation fit no target-dependent affine scale or shift.

\subsubsection{Presence-aligned recurrent update}
At encoder scale $\ell$ and token $n$, let $\mathbf F_{j-1,\ell n}$ be the previous feature state, $\mathbf C_j$ decoded features, and $\mathcal R_\ell$ the reverse-DPT projection. With $g=\sigma(a)$, $a$ initialized to $\operatorname{logit}(.05)$, $\eta_{\rm cap}=.25$, and $\epsilon=10^{-6}$, the update is
\begin{align}
 \mathbf U_{j,\ell n}&=g\,\sg(q_{j,n})\mathcal R_\ell(\mathbf C_j)_n,
 \label{eq:gated-proposal}\\
 \gamma_{j,\ell n}&=\min\!\left\{1,
 \frac{\eta_{\rm cap}\max\{\|\mathbf F_{j-1,\ell n}\|_2,\epsilon\}}
 {\max\{\|\mathbf U_{j,\ell n}\|_2,\epsilon\}}\right\},
 \label{eq:gate-cap}\\
 \mathbf F_{j,\ell n}&=\mathbf F_{j-1,\ell n}-\gamma_{j,\ell n}\mathbf U_{j,\ell n},\nonumber\\
 \|\mathbf F_{j,\ell n}-\mathbf F_{j-1,\ell n}\|_2&
 \leq\eta_{\rm cap}\max\{\|\mathbf F_{j-1,\ell n}\|_2,\epsilon\}.
 \label{eq:recurrence-bound}
\end{align}
Detaching the presence gate blocks direct gradients from later-slot losses to earlier existence probabilities through feature removal, while indirect shared-feature interactions remain. The shared existence head uses a $128\!\to\!32$ $3\times3$ convolution, rectifier, and $32\!\to\!1$ projection. All $K$ passes execute.

\subsubsection{Training loss and the annotation interface}
Auxiliary-free ExactMB uses linear-depth Laplace densities with unit outer weight, temperature, and existence weight, and no geometric auxiliary terms. Training averages over all $BHW$ rays, treating all-false masks as $M=0$ and supplied ties as separate valid entries. Element masks do not distinguish empty targets from unannotated rays. Section~\ref{sec:synthetic-interface-audit} relates these annotations to the likelihood's assumptions of complete, distinct ray-wise depth targets.

\subsubsection{Training schedule and sampling}
Table~\ref{tab:layerset-reference-recipe} specifies the training data, optimizer, and four-seed schedule. Equal-weight cardinality buckets are visited round-robin without replacement, skipping exhausted buckets. This interleaving does not equalize full-epoch bucket mass or per-ray exposure. Primary results use terminal/latest checkpoints at 800k updates. Section~\ref{sec:ablation-factors} details the ablation factors.

\subsubsection{Native selection and geometric decoding}
\label{sec:control-selectors}
Selection is variant-specific. Depth stacking selects all $K$ slots. Count regularization predicts categorical count probabilities $r_m(x)$ for $m=0,\ldots,K$ and selects a prefix of length $\argmax_m r_m(x)$. Bernoulli-slot variants select components with finite $q_j>\tau$, whereas the first two variants ignore this cutoff. A common geometric decoder then keeps finite depths $d_j>d_{\min}$ and sorts them. An empty candidate list yields an empty prediction. Otherwise, it retains the shallowest candidate and each later candidate more than $\Delta_{\rm merge}$ beyond the last retained depth. Rejected candidates leave this reference unchanged: with a $.02$-m gap, $1.000,1.015,1.030$ retains $1.000,1.030$ m.

The principal thresholds are $\tau=.01$, $d_{\min}=.02$\,m, and $\Delta_{\rm merge}=.02$\,m, with strict comparisons and no relative-gap term, target cardinality, alignment, or test-set threshold optimization. Accepted depths form a valid prefix for four-rank MD/LD-Syn scoring. Sorting defines front-to-back evaluator ranks, while target--component assignments are determined by the training objective.

\subsubsection{Batched ExactMB implementation and checks}
\label{sec:batched-exactmb-implementation}
Predictions, targets, and element masks have shape $B\times K\times H\times W$. For target-column validity $m_i$, the broadcast pair costs are
\[
 C_{ji}=\begin{cases}
 |d_j-t_i|/\beta_j+\log(2\beta_j)-\operatorname{logsigmoid}(e_j),&m_i=1,\\
 -\operatorname{logsigmoid}(-e_j),&m_i=0.
 \end{cases}
\]
Finite placeholders replace invalid targets before residual computation. Stable log-sigmoid terms avoid logarithms of rounded probabilities. For each cached permutation $p$, sum $S_p=\sum_jC_{j,p(j)}$. The exact per-ray loss is
\[
 -\operatorname{logsumexp}_{p}(-S_p)+\log\Gamma(K-M+1),\qquad M=\sum_i m_i.
\]
The correction removes $(K-M)!$ permutations of indistinguishable null columns. The $K=4$ implementation differentiates through 24 padded permutations at arithmetic cost $O(BHW\,K!K)$. Empty targets reduce to $-\sum_j\log(1-q_j)$, while full targets need no correction. Checks compare direct injections, padded sums, target permutations, finite-difference gradients, and strict decoding boundaries. For diagnostics, occupancy excludes null columns, and subtracting $\log((K-M)!)$ from padded-permutation entropy gives the entropy over distinct injections.

\subsection{Evaluation Definitions and Adapters}
\label{sec:evaluation-details}
\label{sec:evaluation}
The benchmark annotations in Table~\ref{tab:benchmark-roles} give ACC different meanings: ordinal-tuple, presence-pattern, or count accuracy. LD OverPred complements invalid-tuple rejection, while MD OverPred complements opaque-point rejection; this conversion preserves sample SD and reverses paired-effect signs. Each metric refers to its annotation population, which provides neither dense real multilayer metric ground truth nor a complete enumeration of physical interfaces at every pixel.
\begin{table*}[t]
\centering
\AppendixTableSetup
\caption{\textbf{Complementary annotations for visible-layer evaluation.} Each benchmark defines absence through its annotations. Together, they assess ordering, presence, and metric depth.}
\label{tab:benchmark-roles}
\AppendixTextTableStyle
\begin{tabularx}{\linewidth}{@{}>{\raggedright\arraybackslash}p{.15\linewidth}>{\raggedright\arraybackslash}p{.23\linewidth}>{\raggedright\arraybackslash}X>{\raggedright\arraybackslash}X@{}}
\toprule
\textbf{Benchmark} & \textbf{Available annotation} & \textbf{Reported measures} & \textbf{Not established by this view alone} \\
\midrule
\LDReal\newline300 real images & Sparse valid/invalid tuples, requested ranks, and near-to-far relations & ACC on valid tuples; OverPred on invalid tuples; same/mixed-rank breakdowns & Dense metric accuracy, physical surface count on every ray, or a literal transparent/opaque material partition. \\
\MDReal\newline3,161 real pairs & Transparent mask at query points; foreground/background order; TT/TO/OO and Same/Reverse strata & Recall, OverPred, pattern ACC, ordinal ACC, and Joint ACC on their declared point/pair populations & A census of all physical interfaces or dense real metric depth. \\
\LDSyn\newline500 synthetic images & Four dense metric channels and validity masks, in camera-$Z$ meters & Count ACC, Count MAE, under/over-count, and conditional same-rank AbsRel with support & Unconditional real-image metric generalization or calibration of existence probabilities. \\
\bottomrule
\end{tabularx}
\end{table*}

\subsubsection{Selection and diagnostic contexts}
\begin{table}[htbp]
\centering
\AppendixTableSetup
\caption{\textbf{Decoding rules and populations across evaluation settings.} Each setting specifies its own thresholds and aggregation rules. Diagnostic settings leave the principal decoder unchanged.}
\label{tab:decoder-contexts}
\AppendixTextTableStyle
\begin{tabularx}{\linewidth}{@{}>{\raggedright\arraybackslash}p{.22\linewidth}>{\raggedright\arraybackslash}p{.33\linewidth}>{\raggedright\arraybackslash}X@{}}
\toprule
\textbf{Context} & \textbf{Selection and geometry} & \textbf{Population}\\
\midrule
Principal terminal suite & Native selector; Bernoulli cutoff $.01$; $.02$ m floor and last-retained gap & Full splits; first four accepted ranks on MD/Syn; no alignment\\
Primary cutoff sweep & Same 12 Bernoulli-slot checkpoints; six cutoffs; fixed geometry & Evaluates operating points without treating them as replicates or test-selected optima\\
Training / Syn histories & Raw expected mass or $q>.5$ activation; no principal geometric filter & Minibatch exposure / fixed 64-image cohort; image-conditional validation means\\
Fixed-ray geometry & $q>.01$, depth $>.02$ m, gap $>.02$ m; injective tolerance & 4,513 fixed supplied-distinct-target rays from 20 images; one auxiliary-regularized run\\
Qualitative displays & $q>.5$ with depth sorting; pinhole geometry & Selected examples and checkpoints; illustrative geometry and graphics\\
Released models & Dataset-specific masks, order and scale below & Single released estimates; native-task comparisons remain separate\\
\bottomrule
\end{tabularx}
\end{table}
Table~\ref{tab:decoder-contexts} separates principal and diagnostic protocols. ``Before fallback'' denotes post-decoder validity without shallower-depth substitution, as used for principal presence scores. Appendix~\ref{sec:primary-geometry-operating-points} presents the 20-checkpoint geometric comparison, with cutoff curves for 12 Bernoulli-slot checkpoints alongside the native fixed-count and count-regularized exports.

\subsubsection{LD-Real: recovery and rejection of queried ranks}
The 300-image validation split of the LayeredDepth real benchmark supplies sparse valid/invalid pairs, triplets, and quadruplets \citep{layereddepth2025}. For a tuple of pixel/rank requests with decoded depths $\hat d_1,\ldots,\hat d_k$, correct recovery requires all requested ranks and their annotated strict near-to-far order. Correct rejection of an annotated-invalid tuple requires every requested rank to be absent:
\begin{align}
 \operatorname{correct}_{\rm real}&=\indicator[\bigwedge_r(\hat d_r\ne\missing)\land\hat d_1<\cdots<\hat d_k],\label{eq:ld-real-correct}\\
 \operatorname{correct}_{\rm fake}&=\indicator[\bigwedge_r(\hat d_r=\missing)].\label{eq:ld-fake-correct}
\end{align}
The main-paper ACC is micro-averaged over valid quadruplets. OverPred is the fraction of invalid quadruplets with any requested prediction, so the two metrics have separate denominators. Because tuples can span rays, tuple arity, requested rank, and per-ray cardinality are distinct quantities. Same-rank tuples request one rank index at every point, whereas mixed-rank tuples request different indices. A separate last-visible diagnostic checks the strict ordering of the deepest retained values on mixed valid tuples. Missing predictions count as failures, and invalid tuples have no last-visible target.

\subsubsection{MD-3K: patterns and material-conditioned geometry}
MD-3K provides 3,161 real point pairs with foreground/background relations and material labels \citep{twodepths2026}. At point $p$ of pair $n$, we map transparency $z_{np}$ to target validity $\mathbf v^*(z)=(1,z,0,0)$ in L1/L3/L5/L7 order. For post-decoder bits $\hat{\mathbf v}_{np}$,
\begin{equation}
 E_n=\bigwedge_{p=1}^{2}[\hat{\mathbf v}_{np}=\mathbf v^*(z_{np})],\qquad
 \operatorname{ACC}_{\rm MD}=N^{-1}\sum_n\indicator[E_n].
 \label{eq:md-exact}
\end{equation}
Pattern ACC checks all eight pair-level bits. Point-level measures distinguish missed backgrounds from extra emissions. With $\hat z_{np}$ denoting decoded L3 validity,
\begin{equation}
 \operatorname{Recall}=\frac{\sum_{n,p}z_{np}\hat z_{np}}{\sum_{n,p}z_{np}},\qquad
 \operatorname{OverPred}=\frac{\sum_{n,p}(1-z_{np})\hat z_{np}}{\sum_{n,p}(1-z_{np})}.
 \label{eq:md-cover-reject}
\end{equation}
The frozen split contains 1,485/1,350/326 TT/TO/OO pairs and 4,320 transparent/2,002 opaque requests. This validity convention does not imply that a ray contains at most two physical surfaces.

Foreground ordering uses L1; background ordering uses L3 at transparent points and L1 at opaque points, without replacing missing transparent-point L3 with L1. Let $O_n$ require both orderings to be correct. Headline multilayer ordinal ACC pools TT and TO, while Figure~\ref{fig:iclr-set-quality} uses all pairs. The all-pair evaluation retains six pairs---three OO/Reverse and three TO/Reverse---whose foreground and background orderings cannot both hold under primary front-to-back decoding. Joint correctness combines presence $E$ and ordering $O$ on a common evaluation population $S$:
\begin{equation}
 \Pr(E\cap O\mid S)=\Pr(E\mid S)\Pr(O\mid E,S).
 \label{eq:md-joint-decomposition-protocol}
\end{equation}
If no pair satisfies $E$, conditional accuracy is undefined and joint frequency is zero. Pattern ACC thus evaluates pair-level presence, Recall/OverPred individual L3 requests, and Joint ACC presence with both orderings. Same/Reverse denotes agreeing/flipped foreground and background orders. Released ordinal scores may use shallower fallback, unlike Joint ACC.

\subsubsection{LD-Syn: cardinality and metric localization}
The 500 held-out images provide four dense metric channels and validity masks \citep{layereddepth2025}. Target and decoded masks are compact prefixes on this split, so four-bit pattern accuracy equals count accuracy. The principal population is occupied rays $\Omega_+=\{x:c_x\ge1\}$, with
\begin{align}
 \operatorname{ACC}_c&=\operatorname{mean}_{x:c_x=c}\indicator[\hat c_x=c],\qquad
 \operatorname{Macro}=\tfrac14\sum_{c=1}^{4}\operatorname{ACC}_c,\label{eq:syn-count-exact}\\
 \operatorname{MAE}&=\operatorname{mean}_{x\in\Omega_+}|\hat c_x-c_x|.\label{eq:syn-cardinality-metrics}
\end{align}
Occupied-micro ACC and under/over-count rates pool $\Omega_+$. Explicitly labeled all-pixel fields also include empty masks. Target counts retain supplied ties without distance-based deduplication. For localization, own-slot error compares each decoded rank with its target channel where both are valid. Let $\mathcal S_3$ be this intersection for L3:
\begin{equation}
 \operatorname{AR}^{\rm own}_3=\frac{100}{|\mathcal S_3|}\sum_{x\in\mathcal S_3}
 \frac{|\hat d_{3x}-d_{3x}|}{d_{3x}}.
 \label{eq:l3-own-absrel}
\end{equation}
Errors and counts are pooled before division. Abstaining on difficult rays changes the evaluated population, so conditional error must be read with support. Section~\ref{sec:primary-geometry} gives all primary rank-wise denominators, support recall, and depth-accurate recall/precision. The complementary injective-set diagnostic uses a distinct cohort and an absolute geometric tolerance.

\subsubsection{Replication and retrospective selection}
Training seed is the replication unit. Absolute results give means and sample SD; paired comparisons report within-seed changes, SD, and directional agreement. Pixel/image variation is not training-seed replication. Four seed pairs cannot attain $p<.125$ in a two-sided exact sign-flip test, so we emphasize effect sizes and metric trade-offs. Primary MD pattern ACC and accompanying Recall/OverPred, LD tuple outcomes, and synthetic count/geometry were organized retrospectively, without preregistration or score-blind selection. Uncertainty summaries therefore include each cohort's selection context.

\subsubsection{Released-system transfer and adapter boundaries}
\label{sec:released-adapters}
Table~\ref{tab:iclr-external} evaluates released models without retraining. SeeGroup \citep{seegroup2026} decodes with a $.01$ validity-score cutoff and $.02$ minimum depth and gap. World Tracing \citep{worldtracing2026} maps its first four intersections to L1/L3/L5/L7 without reordering or merging. LD uses finite positive camera-$Z$ alone, whereas MD requires the mask and finite positive $Z$. Each image uses one 20-step sample with seed $42+i$. Releases r69l/r75b differ in resolution and XYZ normalization.

LaRI \citep{lari2026} keeps native intersection order without sorting or merging. The scene release uses finite positive $Z$ without a mask checkpoint, whereas the separately trained object release applies its ray-stop mask. This compares distinct checkpoints. Scale conventions differ: SeeGroup/LaRI use relative scale, WT r69l returns relative XYZ, and r75b restores metric units with fixed normalization statistics. Ordinal annotations assess ordering, while metric accuracy requires depth ground truth.

All five released LD-Syn evaluations use the same 500 images and 460,008,237 occupied rays as the primary results, retaining supplied ties in the target counts. MD uses the eight-bit events and fixed point denominators defined above. These scores measure transfer to the supplied visible targets. Differences in native masks, ordering, scale, training, and stochastic generation prevent attributing the results solely to architecture or ranking each model on its native task.

Released-model AbsRel uses a different calibration protocol from the primary comparison. For all five models, the evaluator fits affine scale and shift per image and layer on valid GT--prediction pairs with $0<d,\hat d<10^4$, retains native values when fewer than 16 fitting pixels exist, and clips predictions to $[.001,30]$\,m. We average four conditional same-rank AbsRel scores without shallower-layer inheritance. SeeGroup's calibrated pass covers all 500 images and preserves decoded validity masks. These GT-calibrated estimates assess conditional within-layer localization after alignment, leaving native scale and inter-layer geometry to separate evaluation. Support remains important: LaRI objects retains only $3.63\%/1.64\%$ of third-/fourth-rank target support, so its deeper errors describe a small selected population rather than recovery of all annotated visible layers.

\suppressfloats[t]
\subsection{Ablation Factors and Configuration Provenance}
\label{sec:ablation-factors}

The five primary variants in Table~\ref{tab:iclr-primary-ablation} ask how to determine emitted cardinality and assign observed depths to components. All share the recorded data, encoder/recurrent trunk, capacity, optimization schedule, and four-seed terminal-800k protocol, with geometric auxiliary terms disabled.

Count regularization versus fixed cardinality tests a complete global-stopping configuration: categorical predictor, loss, and selector. Count regularization averages depth costs over valid entries; other variants average per-ray sums. For assignment, MAP, ExactMB, and ordered assignment share model and training fields within each seed, differing only in assignment criterion. Cross-group comparisons also change recurrence gating and selection rather than isolating an additional loss.

A separate four-seed study compares PPP, MAP, and ExactMB. Its recorded synthetic-data sampler and scheduler differ from the primary cardinality-interleaved sampler and warmup--hold--cosine schedule; these snapshots are incomplete launch records. Within this cohort, PPP changes the point-process law, while MAP and ExactMB differ in assignment criterion. The PPP control uses $\lambda_j=\operatorname{softplus}(e_j)$, with decoder score $\sigma(e_j)$ equal to its nonzero-count probability before the numerical rate floor. An intensity component may generate multiple returns. ExactMB minus MAP changes MD pattern accuracy by $-1.35\pm6.69$ points in the primary study and $+1.96\pm1.37$ in the mechanism study (paired mean and sample SD). Interpretation is therefore recipe-specific.

Endpoint records are fuller than training histories. Checkpoint and evaluator hashes cover all 20 primary and 12 point-process-study endpoints, but synthetic geometry records lack a historical evaluator-source hash. The 20 primary exports describe resumed states without a source commit; the mechanism-study catalog preserves selected evaluator-resolved fields. These records support the comparisons but do not identify the source version throughout every historical training segment.

\subsection{Annotation Semantics and the Training Interface}
\label{sec:synthetic-interface-audit}
The supplied annotations do not certify the completeness or distinctness assumed by the set likelihood. Compaction retains supplied order and quantized repetitions, and all-false element masks contribute $M=0$ without a missingness flag. ExactMB sums assignments over these entries, whereas ordered assignment fixes their compacted-channel correspondence. Normalization and propriety hold for the continuous simple-set law, whose assumptions need not be satisfied by these annotations.

In the 20-image $432\times768$ audit, 202 of 6,635,520 rays at one image's right boundary have all-false masks of unresolved empty-versus-missing status. Of the remaining rays, 11,604 (.175\%) have exact valid-depth ties and 66,567 (1.003\%) have a pair within 2 cm. The 2 cm prediction filter leaves targets unchanged, so pruning correct predictions of close targets can undercount.

The full-resolution audit tests depth validity (positive and no greater than 80 m) across 460,800,000 pixels in 500 synthetic images. No invalid channel precedes a valid one. Cardinalities one through four occupy 402,453,131 / 29,100,688 / 17,298,739 / 11,155,679 pixels, with 791,763 empty masks. Adjacent valid depths tie in 1,720,016 pixels and descend in 37,774. Thus, compact prefixes equate pattern and count accuracy without certifying sorted, distinct targets or physically empty rays.

\FloatBarrier
\Needspace{10\baselineskip}
\section{Additional Experimental Results}
\suppressfloats[t]
\label{sec:experiments}
\label{sec:evaluation-questions}

Count agreement, depth accuracy, and surface recovery can favor different choices. We expand primary four-seed terminal-800k comparisons with paired effects, rank-wise geometry, and error breakdowns, then examine separate point-process, geometric-loss, and architecture studies. Operating-point analyses and real-world visualizations connect decoding choices to retained support.

\subsection{Primary Ablations and Released-Model Transfer}
\label{sec:exp-behavior}

The primary variants---depth stacking (fixed count), explicit count regularization, MAP assignment, marginalized assignment (ExactMB), and ordered assignment---share a prediction family while varying how many depths are emitted and how predictions are assigned to targets.

\subsubsection{Paired Effects of Explicit Count Regularization}
\label{sec:exp-system}

Table~\ref{tab:iclr-primary-ablation} reports all five variants without geometric auxiliaries at the principal operating point. Fixed-count and count-regularized configurations differ in objective, depth-loss reduction, and selector. The three presence-based variants instead share gated recurrence and differ only in recorded assignment criterion per seed, as specified in Appendix~\ref{sec:ablation-factors}. We first examine the count-regularized contrast.

\begin{table}[!htbp]
\centering
\AppendixTableSetup
\caption{\textbf{The count-regularized configuration reduces false emission at the cost of valid returns.}
Paired count-regularized minus fixed-cardinality differences, reported as mean$\pm$sample SD over seeds $7/42/61/123$, with native selectors and common geometric filters. Rate differences are percentage points. Count MAE is in layers. Every paired seed shares the displayed direction, including losses in LD valid-query accuracy, transparent-point recall, and four-return count accuracy.}
\label{tab:empirical-primary-paired}
\begingroup
\AppendixNumericTableStyle
\begin{tabular*}{\linewidth}{@{\extracolsep{\fill}}llr@{}}
\toprule
\textbf{Benchmark} & \textbf{Outcome} & \shortstack[r]{\textbf{Signed change}\\\textbf{Count reg. $-$ fixed}} \\
\midrule
LD-Real & Ordinal ACC $\uparrow$ & $-8.5_{\mathord{\pm}1.3}$ \\
LD-Real & Tuple-level OverPred $\downarrow$ & $-91.0_{\mathord{\pm}0.8}$ \\
MD-3K & Pattern ACC $\uparrow$ & $+65.8_{\mathord{\pm}1.2}$ \\
MD-3K & Transparent-point Recall $\uparrow$ & $-20.9_{\mathord{\pm}1.0}$ \\
MD-3K & Opaque-point OverPred $\downarrow$ & $-89.9_{\mathord{\pm}0.5}$ \\
LD-Syn & Occupied-micro count ACC $\uparrow$ & $+91.32_{\mathord{\pm}0.18}$ \\
LD-Syn & Cardinality-macro count ACC $\uparrow$ & $+48.07_{\mathord{\pm}0.28}$ \\
LD-Syn & Count MAE $\downarrow$ & $-2.177_{\mathord{\pm}0.032}$ \\
LD-Syn & Four-return count ACC $\uparrow$ & $-21.55_{\mathord{\pm}0.80}$ \\
\bottomrule
\end{tabular*}
\endgroup
\end{table}

Table~\ref{tab:empirical-primary-paired} shows that count regularization reduces false emission and improves aggregate count accuracy, but loses valid returns, especially at full cardinality. These effects compare complete configurations, not an isolated count loss. Converting rejection to overprediction reverses signs without changing SD; four paired seeds give an exact two-sided sign-flip resolution of $.125$.

\paragraph{Released-model transfer.}
\label{sec:exp-external}

Table~\ref{tab:iclr-external} complements controlled ablations with released models, without retraining. Appendix~\ref{sec:released-adapters} specifies native-output adapters and GT-calibrated depth evaluation.

\subsection{Primary Four-Seed Geometric Recovery}
\label{sec:primary-geometry}

Count agreement alone does not establish which surfaces are recovered. Table~\ref{tab:primary-geometry} therefore reports per-rank geometry and support for the same
primary terminal-800k checkpoints and four seeds as the main comparison. Each
checkpoint is evaluated on all 500 \LDSyn validation images with the prescribed
decoder. Ranks 1--4 correspond to benchmark channels L1/L3/L5/L7. We pool pixels
within each seed and report the mean and sample SD across seeds $7/42/61/123$.
The SD therefore measures training-seed variation, excluding image-sampling uncertainty.

To separate localization from retained support, let $G_r$ and $P_r$ contain pixels with valid annotated and decoded depth at rank $r$,
respectively, and let $I_r=G_r\cap P_r$. Conditional AbsRel averages relative absolute error over $I_r$. RMS is the
square root of mean squared error on $I_r$. Conditional $\delta_1$ is the fraction of $I_r$ satisfying
$\max(\hat d_r/d_r,d_r/\hat d_r)<1.25$. Target-support recall is
$|I_r|/|G_r|$. With $C_r$ the subset satisfying this depth criterion,
depth-accurate recall and precision are $|C_r|/|G_r|$ and $|C_r|/|P_r|$.
We recover $|C_r|$ from archived conditional $\delta_1$ and intersection counts,
verifying integrality at every checkpoint and rank. These scores assess depth
and presence per rank. Appendix~\ref{sec:iclr-geometric-learning} tests recovery of all ray returns via absolute-tolerance injective matching on a separate cohort.

\begin{table}[t]
\centering
\AppendixTableSetup
\caption{\textbf{Depth error, support, and depth-and-presence recovery for the primary five systems.}
Mean$\pm$sample SD over four training seeds. All entries except RMS (meters) are
percentages. AbsRel, RMS, and $\delta_1$ are conditional on both same-rank values
being valid. Support and depth-accurate recall use the fixed GT denominators shown
below. Depth-accurate precision uses each system's predicted support. Different
conditional supports preclude interpreting AbsRel alone as complete recovery.}
\label{tab:primary-geometry}
\AppendixNumericTableStyle
\begin{tabular*}{\linewidth}{@{\extracolsep{\fill}}lrrrrrr@{}}
\toprule
\TableHead{Method} & \TableHead{AbsRel $\downarrow$} & \TableHead{RMS (m) $\downarrow$}
& \TableHead{$\delta_1$ $\uparrow$} & \TableHead{Support $\uparrow$}
& \TableHead{Depth rec. $\uparrow$} & \TableHead{Depth prec. $\uparrow$} \\
\midrule
\rowcolor{appendixRankBand}[0pt][0pt]
\multicolumn{7}{@{}l@{}}{\textbf{Rank 1 (L1)}\quad $|G_r|=460\,008\,237$ pixels}\\
Fixed count & $14.02_{\mathord{\pm}0.19}$ & $0.240_{\mathord{\pm}0.001}$ & $85.34_{\mathord{\pm}0.16}$ & $100.00_{\mathord{\pm}0.00}$ & $85.34_{\mathord{\pm}0.16}$ & $85.20_{\mathord{\pm}0.16}$ \\
Count reg. & $14.14_{\mathord{\pm}0.07}$ & $0.235_{\mathord{\pm}0.001}$ & $85.41_{\mathord{\pm}0.12}$ & $99.99_{\mathord{\pm}0.00}$ & $85.40_{\mathord{\pm}0.12}$ & $85.35_{\mathord{\pm}0.12}$ \\
MAP & $14.17_{\mathord{\pm}0.11}$ & $0.268_{\mathord{\pm}0.002}$ & $84.71_{\mathord{\pm}0.10}$ & $100.00_{\mathord{\pm}0.00}$ & $84.71_{\mathord{\pm}0.10}$ & $84.61_{\mathord{\pm}0.10}$ \\
ExactMB & $14.26_{\mathord{\pm}0.10}$ & $0.263_{\mathord{\pm}0.004}$ & $84.72_{\mathord{\pm}0.16}$ & $99.99_{\mathord{\pm}0.01}$ & $84.71_{\mathord{\pm}0.16}$ & $84.62_{\mathord{\pm}0.16}$ \\
Ordered & $14.22_{\mathord{\pm}0.32}$ & $0.237_{\mathord{\pm}0.003}$ & $85.55_{\mathord{\pm}0.14}$ & $100.00_{\mathord{\pm}0.00}$ & $85.55_{\mathord{\pm}0.14}$ & $85.45_{\mathord{\pm}0.14}$ \\
\addlinespace[0.25ex]
\rowcolor{appendixRankBand}[0pt][0pt]
\multicolumn{7}{@{}l@{}}{\textbf{Rank 2 (L3)}\quad $|G_r|=57\,555\,106$ pixels}\\
Fixed count & $17.22_{\mathord{\pm}0.21}$ & $0.550_{\mathord{\pm}0.018}$ & $80.89_{\mathord{\pm}0.57}$ & $99.15_{\mathord{\pm}0.04}$ & $80.20_{\mathord{\pm}0.59}$ & $10.07_{\mathord{\pm}0.08}$ \\
Count reg. & $16.00_{\mathord{\pm}0.11}$ & $0.461_{\mathord{\pm}0.003}$ & $83.84_{\mathord{\pm}0.30}$ & $89.89_{\mathord{\pm}0.28}$ & $75.36_{\mathord{\pm}0.35}$ & $74.96_{\mathord{\pm}0.35}$ \\
MAP & $16.99_{\mathord{\pm}0.15}$ & $0.602_{\mathord{\pm}0.007}$ & $78.77_{\mathord{\pm}0.46}$ & $97.46_{\mathord{\pm}0.06}$ & $76.78_{\mathord{\pm}0.42}$ & $53.95_{\mathord{\pm}0.49}$ \\
ExactMB & $17.86_{\mathord{\pm}0.17}$ & $0.608_{\mathord{\pm}0.019}$ & $78.46_{\mathord{\pm}0.44}$ & $96.59_{\mathord{\pm}0.14}$ & $75.78_{\mathord{\pm}0.49}$ & $54.04_{\mathord{\pm}1.24}$ \\
Ordered & $16.41_{\mathord{\pm}0.25}$ & $0.491_{\mathord{\pm}0.002}$ & $82.69_{\mathord{\pm}0.22}$ & $97.77_{\mathord{\pm}0.18}$ & $80.85_{\mathord{\pm}0.25}$ & $58.33_{\mathord{\pm}0.52}$ \\
\addlinespace[0.25ex]
\rowcolor{appendixRankBand}[0pt][0pt]
\multicolumn{7}{@{}l@{}}{\textbf{Rank 3 (L5)}\quad $|G_r|=28\,454\,418$ pixels}\\
Fixed count & $18.09_{\mathord{\pm}0.16}$ & $0.606_{\mathord{\pm}0.012}$ & $78.75_{\mathord{\pm}0.10}$ & $95.76_{\mathord{\pm}0.09}$ & $75.41_{\mathord{\pm}0.09}$ & $5.21_{\mathord{\pm}0.06}$ \\
Count reg. & $16.82_{\mathord{\pm}0.19}$ & $0.517_{\mathord{\pm}0.008}$ & $81.04_{\mathord{\pm}0.17}$ & $79.08_{\mathord{\pm}0.10}$ & $64.08_{\mathord{\pm}0.07}$ & $66.67_{\mathord{\pm}0.57}$ \\
MAP & $19.30_{\mathord{\pm}0.51}$ & $0.716_{\mathord{\pm}0.009}$ & $74.47_{\mathord{\pm}0.40}$ & $93.34_{\mathord{\pm}0.19}$ & $69.51_{\mathord{\pm}0.49}$ & $35.73_{\mathord{\pm}0.12}$ \\
ExactMB & $21.04_{\mathord{\pm}0.26}$ & $0.735_{\mathord{\pm}0.022}$ & $74.19_{\mathord{\pm}0.09}$ & $89.76_{\mathord{\pm}0.37}$ & $66.60_{\mathord{\pm}0.29}$ & $34.35_{\mathord{\pm}0.64}$ \\
Ordered & $17.36_{\mathord{\pm}0.30}$ & $0.594_{\mathord{\pm}0.009}$ & $79.57_{\mathord{\pm}0.15}$ & $93.11_{\mathord{\pm}0.29}$ & $74.09_{\mathord{\pm}0.32}$ & $42.39_{\mathord{\pm}0.24}$ \\
\addlinespace[0.25ex]
\rowcolor{appendixRankBand}[0pt][0pt]
\multicolumn{7}{@{}l@{}}{\textbf{Rank 4 (L7)}\quad $|G_r|=11\,155\,679$ pixels}\\
Fixed count & $23.80_{\mathord{\pm}0.40}$ & $0.895_{\mathord{\pm}0.013}$ & $73.05_{\mathord{\pm}0.84}$ & $74.05_{\mathord{\pm}0.75}$ & $54.10_{\mathord{\pm}0.46}$ & $2.40_{\mathord{\pm}0.12}$ \\
Count reg. & $21.93_{\mathord{\pm}0.59}$ & $0.621_{\mathord{\pm}0.026}$ & $77.08_{\mathord{\pm}0.42}$ & $52.51_{\mathord{\pm}1.06}$ & $40.47_{\mathord{\pm}0.73}$ & $52.46_{\mathord{\pm}0.62}$ \\
MAP & $24.42_{\mathord{\pm}0.76}$ & $0.930_{\mathord{\pm}0.062}$ & $73.38_{\mathord{\pm}1.16}$ & $64.94_{\mathord{\pm}1.53}$ & $47.65_{\mathord{\pm}0.95}$ & $17.84_{\mathord{\pm}0.24}$ \\
ExactMB & $23.57_{\mathord{\pm}0.55}$ & $0.994_{\mathord{\pm}0.021}$ & $75.22_{\mathord{\pm}0.84}$ & $52.30_{\mathord{\pm}0.95}$ & $39.34_{\mathord{\pm}0.49}$ & $14.51_{\mathord{\pm}0.25}$ \\
Ordered & $22.43_{\mathord{\pm}0.53}$ & $0.843_{\mathord{\pm}0.028}$ & $74.65_{\mathord{\pm}0.89}$ & $71.39_{\mathord{\pm}0.50}$ & $53.29_{\mathord{\pm}0.83}$ & $23.27_{\mathord{\pm}0.35}$ \\
\bottomrule
\end{tabular*}
\end{table}

Ordered assignment improves depth-accurate recall and precision over
marginalization at every rank in all four paired seeds. At rank four, the gains are
$13.95\pm0.58$ and $8.76\pm0.59$ percentage points, respectively. Fixed count retains
high target coverage with low deeper-rank precision. Compared with ordered
assignment, the count-regularized model attains higher precision at ranks 2--4 with lower recall.
Thus, localization, coverage, and precision favor different configurations.

\paragraph{Metric and provenance scope.}
We analyze archived same-rank optional-depth outputs without last-visible fallback,
alignment, or clipping, excluding legacy per-layer metrics that clip to $[0.001,30]$m.
The 20 core metric hashes match the ledger, and all 80 same-rank records match
archived summaries. The current evaluator confirms these definitions, but historical
synthetic records lack a source hash identifying its version. This reanalysis uses
existing aggregate outputs without new inference.

\FloatBarrier

\subsection{Error Structure Across Cardinalities and Relations}
\label{sec:exp-error-structure}
To explain aggregate rankings, we stratify by cardinality and material pattern, then distinguish presence from ordering failures within each experiment population.
\subsubsection{Cardinality and Material-Conditioned Reversals}
\label{sec:exp-cardinality}

Figure~\ref{fig:tpami-regime-reversal} stratifies the five primary variants by true occupied count and MD material pattern. Count regularization leads at $c=1$, $2$, and $3$, with the best cardinality macro ($75.0\%$), occupied-micro exactness ($95.0\%$), and MAE (0.067 layers). At $c=4$, however, its $52.5\%$ exactness trails ordered assignment's $71.4\%$, reversing their lower-cardinality ranking.

\begin{figure}[!t]
\centering
\includegraphics[page=4,width=\textwidth]{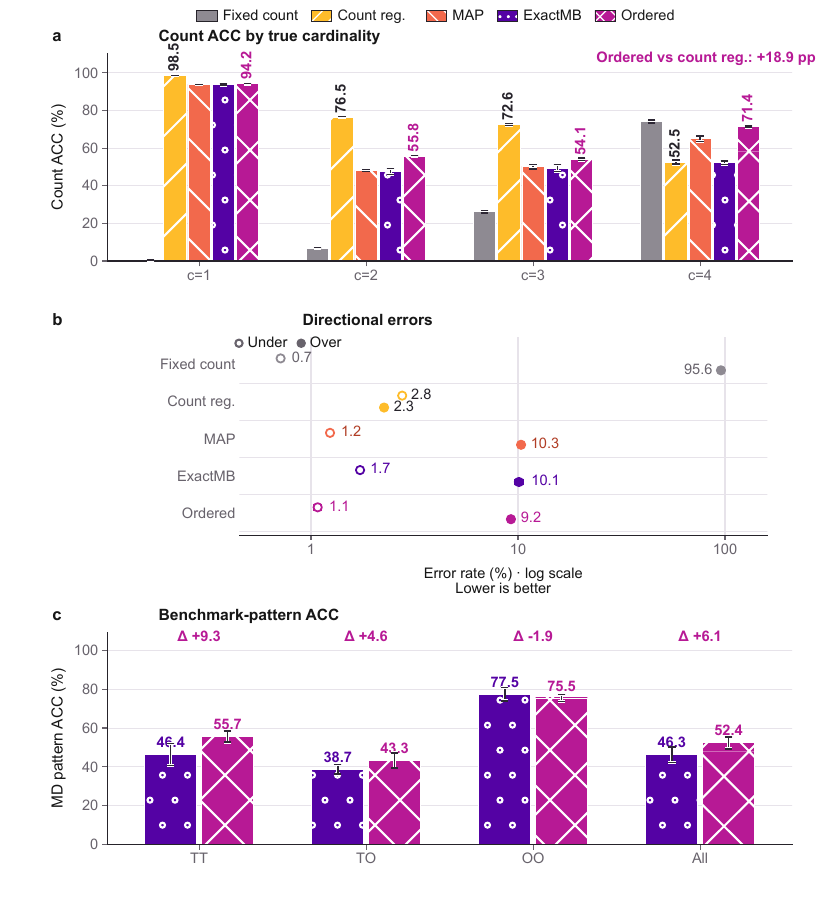}
\caption{\textbf{Conditioning reveals that the aggregate ranking is regime-dependent.}
\textbf{(a)} Grouped bars stratify exact decoded count by true occupied count $c$. Whiskers are sample SD, and direct values identify the count-regularized and ordered-assignment variants, which reverse at $c=4$.
\textbf{(b)} Open and filled points show under- and over-counting, respectively, on one logarithmic axis. Printed values retain native percentages and whiskers show sample SD. All intervals are strictly positive, with no numerical offsets or truncation. The log view resolves small error rates alongside the fixed-cardinality variant's 95.6\% over-counting.
\textbf{(c)} Paired bars compare marginalized (\ExactMB) and ordered assignment on TT, TO, OO, and overall \MDReal benchmark-pattern exactness. Here, $\Delta$ is ordered minus marginalized assignment. Each stratum tests the full decoded validity pattern at both query points, before ordinal fallback.
All displayed variants use the same four primary training seeds as the controlled rows in Table~\ref{tab:iclr-primary-ablation}.
The fixed-cardinality variant's high $c=4$ value reflects compulsory four-slot emission, so it cannot demonstrate successful stopping.}
\label{fig:tpami-regime-reversal}
\end{figure}

\ResultLead{Cardinality and material pattern reverse ablation rankings.}
\looseness=-1 Relative to the count-regularized model, ordered assignment changes exact count accuracy by $-4.30/-20.70/-18.43/+18.88$ points for $c=1/2/3/4$. Its aggregate changes are $-6.14$ points macro, $-5.31$ occupied-micro, and 0.087 more MAE. Rankings also depend on material pattern: relative to \ExactMB, ordered assignment improves MD pattern ACC by $9.29\pm6.95$ points on TT and $4.56\pm4.94$ on TO, but changes OO by $-1.92\pm4.78$ and raises point-level OverPred by $2.42\pm4.46$. Compulsory four-slot emission likewise reaches $74.1\%$ at $c=4$, despite a $3.7\%$ occupied-micro score and 2.244-layer MAE overall.

\subsubsection{Joint Existence and Ordering Failures}
\label{sec:exp-joint-errors}

Correct cardinality does not ensure correct ordering. At 800k and $\tau=.5$, Figure~\ref{fig:tpami-od-error-anatomy} separates existence and ordering failures on valid LD tuples in the separate ViT-L recurrent geometry study described in Appendix~\ref{sec:extended-analyses}. Figure~\ref{fig:iclr-set-quality} applies the same decomposition to all MD-3K pairs across the five primary four-seed ablation configurations at the principal operating point.
\begin{figure}[t]
\centering
\includegraphics[page=3,width=\textwidth]{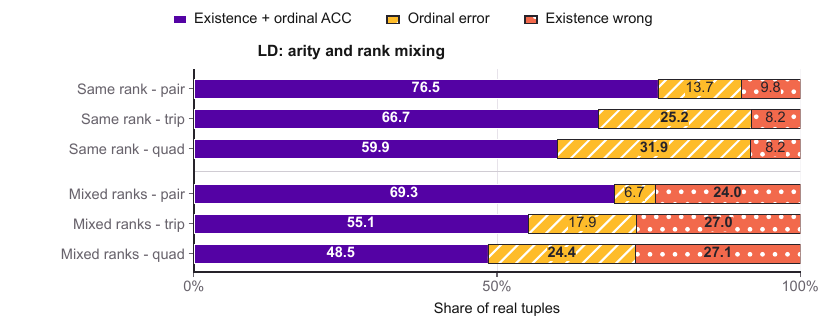}
\caption{\textbf{Ordering accuracy alone can hide incorrect presence patterns.}
In the geometry study at latest 800k and $\TauDiag{=}\TauDiagDefault$, each valid LD tuple has exactly one outcome: existence and order both correct, existence correct but order wrong, or existence wrong. Seeds are averaged within recipe family, then the 16 designed families receive equal weight. Mixed-rank tuples have more existence failures. Quadruplets have a larger ordering-failure share than pairs even when queried ranks exist.}
\label{fig:tpami-od-error-anatomy}
\end{figure}

\ResultLead{The source of failure changes with both rank composition and arity.}
From pairs to quadruplets, same-rank LD existence failures remain near $8$--$10\%$, while correct-existence/wrong-order outcomes rise from $13.72\%$ to $31.87\%$. Mixed-rank existence failures rise from $23.97\%$ to $27.15\%$. Across all recipe families and six thresholds, seed-averaged mixed-rank existence accuracy trails same-rank accuracy at every arity. Quadruplets have larger ordering-failure shares than pairs, consistent with the additional ordering constraints that must hold simultaneously.

In the primary MD study, ordered assignment exceeds \ExactMB by $4.28$ Joint ACC points and $6.11$ pattern ACC points despite lower $\Pr(O=1\mid E=1)$ ($88.9\%$ versus $91.4\%$): better presence-pattern recovery outweighs lower conditional ordering accuracy. These summaries use all 3,161 pairs, whereas headline ordinal ACC is evaluated only on the TT+TO subset.

\FloatBarrier

\subsection{Point-Process and Assignment Controls}
\label{sec:exp-modeling}
\subsubsection{Point-Process Law and Assignment Reduction}
\label{sec:exp-mechanism}

To examine point-process and assignment effects, we compare PPP, MAP, and ExactMB using a shared $K=4$ model and four seeds at 800k updates. This separate study uses a different recorded sampler and schedule. Appendix~\ref{sec:ablation-factors} describes its incomplete historical training records. Figure~\ref{fig:tpami-mechanism-ladder} contrasts PPP's global void event and repeated intensity ownership with MAP's injective Bernoulli singleton/null costs and ExactMB's marginalization of these explanations.

\begin{figure}[!t]
\centering
\includegraphics[page=19,width=\textwidth]{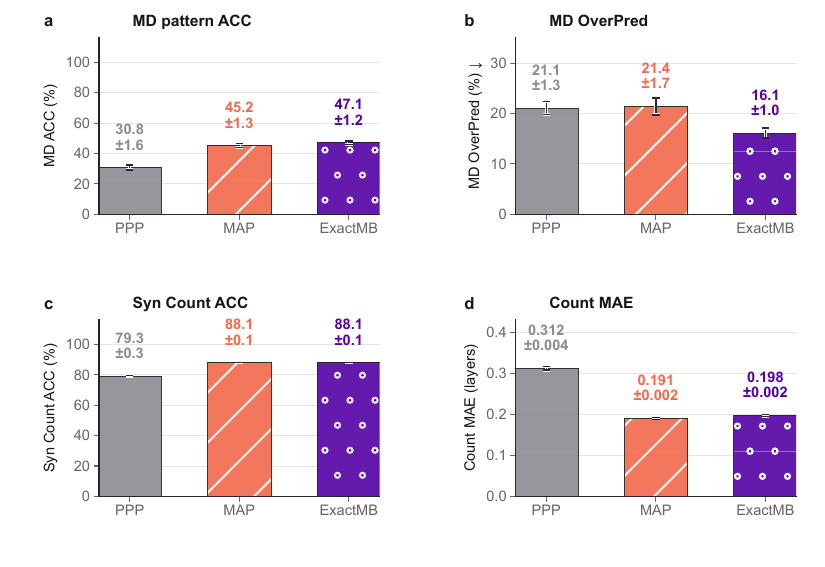}
\caption{\textbf{MAP improves pattern and count accuracy over PPP. Marginalization reduces false emission.}
(a) MD pattern ACC. (b) MD point-level OverPred. (c) LD-Syn occupied-micro count ACC. (d) LD-Syn count MAE.
PPP has a global void event and permits repeated intensity ownership. MAP assignment uses injective Bernoulli--Laplace costs with null factors. \ExactMB sums those explanations. PPP uses numerical rate and log-intensity floors. Normalization describes the underlying law.
PPP$\rightarrow$MAP assignment is a bundled point-process contrast. MAP$\rightarrow$marginalized assignment changes only the criterion class among the available configuration fields.
Bars report terminal/latest 800k means over seeds $7/42/61/123$. Whiskers and $\pm$ readouts give sample SD. Higher exactness and lower OverPred and MAE indicate better performance.}
\label{fig:tpami-mechanism-ladder}
\end{figure}

\ResultLead{Changing the point-process model improves pattern and count recovery.}
PPP$\rightarrow$MAP raises MD pattern ACC by $14.39\pm2.84$ points and synthetic occupied-micro count ACC by $8.79\pm0.15$, and lowers MAE by $0.121\pm0.003$ layers across all four seeds. MD point-level OverPred changes by $+0.32\pm1.59$ points, positive in three seeds. PPP already supplies absence evidence, so this contrast combines changes in point-process law, multiplicity, and assignment.

MAP$\rightarrow$ExactMB changes only the recorded assignment criterion. Posterior summation lowers LD tuple-level OverPred by $1.56\pm0.47$ points and MD point-level OverPred by $5.31\pm1.49$, and raises MD pattern ACC by $1.96\pm1.37$ across all four seeds. Count MAE worsens by $0.0069\pm0.0030$ layers in every seed. Synthetic occupied-micro exactness changes by $-0.014\pm0.214$ points, negative in only one of four seeds. This near-zero mean does not establish equivalence. Marginalization thus reduces false emission at a small count-MAE cost. Appendix~\ref{sec:ablation-factors} reports the opposite pattern-accuracy change and greater seed variation in the primary four-seed comparison.

\FloatBarrier

\subsection{Matched Geometric-Loss Interventions}
\label{sec:exp-exactmb-geometry}
\label{sec:exp-geometry-loss}

With ExactMB fixed, Table~\ref{tab:tpami-exactmb-interventions} tests geometric supervision across seven configurations and 26 terminal checkpoints at 800k (epoch 108), sharing a four-slot recurrent decoder and $\tau=.01$. G0 includes geometric regularization, so contrasts concern schedules and added terms. This shared-seed study is separate from the primary ablations; links to historical training code are incomplete.

\begin{table}[!htbp]
\centering
\AppendixTableSetup
\caption{\textbf{Seven-arm geometric-loss study design.} All arms use ExactMB. G0 is the reference. G1 changes the final weight and ramp of sorted-layer gradient matching (SortedGrad). G5 and G6 share only seeds 42 and 123, so their difference cannot establish a replicated schedule interaction.}
\label{tab:tpami-exactmb-interventions}
\begingroup
\AppendixTextTableStyle
\begin{tabularx}{\linewidth}{@{}l>{\raggedright\arraybackslash}Xll@{}}
\toprule
\textbf{Arm} & \textbf{Intervention} & \textbf{Reference} & \textbf{Seeds} \\
\midrule
G0 & SortedGrad $\lambda=.02$; linear ramp completes at $.15T$ & --- & 7/42/61/123 \\
G1 & SortedGrad $\lambda=.05$; ramp completes at $.05T$ & G0 & 7/42/61/123 \\
G2 & Depth-channel weights $[.5,1,1.5,2]$ replace $[1,1,1,1]$ & G1 & 7/42/61/123 \\
G3 & Depth diversity: margin $.01$, $\lambda:0\to.03$ & G1 & 7/42/61/123 \\
G4 & Intensity supervision: $\gamma=.1$, $\lambda=.2$ & G1 & 7/42/61/123 \\
G5 & AssignGrad: detached assignment-aligned gradient term, $\lambda=.05$ & G0 & 7/42/123 \\
G6 & AssignGrad: detached assignment-aligned gradient term, $\lambda=.05$ & G1 & 42/61/123 \\
\bottomrule
\end{tabularx}
\endgroup
\end{table}

\begin{table}[!htbp]
\centering
\AppendixTableSetup
\caption{\textbf{Paired effects of the geometric-loss interventions.}
Mean$\pm$sample SD over the intersecting training seeds in Table~\ref{tab:tpami-exactmb-interventions}. Positive denotes improvement: accuracy differences, OverPred reduction, or count-MAE reduction. All rate effects are percentage points. The last column is in layers. A dagger marks an effect with the same nonzero direction in every paired seed. LD last-visible ordering pools mixed valid tuples. MD pattern ACC requires the complete post-decoder pair of validity patterns.}
\label{tab:empirical-auxiliary-effects}
\begingroup
\AppendixNumericTableStyle
\newcommand{\EmpiricalEffect}[3]{\ensuremath{#1_{\mathord{\pm}#2}#3}}
\begin{tabular*}{\linewidth}{@{\extracolsep{\fill}}lrrrrrrr@{}}
\toprule
\multicolumn{8}{@{}l@{}}{\textbf{Paired improvement $\uparrow$}}\\
 & \multicolumn{3}{c}{\textbf{LD-Real}} & \multicolumn{2}{c}{\textbf{MD-3K}} & \multicolumn{2}{c}{\textbf{LD-Syn}} \\
\cmidrule(lr){2-4}\cmidrule(lr){5-6}\cmidrule(l){7-8}
\textbf{Contrast} & \shortstack{\textbf{Ordinal}\\\textbf{ACC}} & \shortstack{\textbf{OverPred}\\\textbf{reduction}} & \shortstack{\textbf{Last-visible}\\\textbf{ACC}} & \shortstack{\textbf{Pattern}\\\textbf{ACC}} & \shortstack{\textbf{OverPred}\\\textbf{reduction}} & \shortstack{\textbf{Count}\\\textbf{ACC}} & \shortstack{\textbf{Count MAE}\\\textbf{reduction}} \\
\midrule
G1--G0 & \EmpiricalEffect{+1.18}{0.80}{^{\dagger}} & \EmpiricalEffect{+0.48}{0.36}{^{\dagger}} & \EmpiricalEffect{+4.34}{1.50}{^{\dagger}} & \EmpiricalEffect{+0.69}{2.53}{} & \EmpiricalEffect{-0.17}{0.96}{} & \EmpiricalEffect{-0.11}{0.16}{} & \EmpiricalEffect{-0.0002}{0.0023}{} \\
G2--G1 & \EmpiricalEffect{+0.24}{1.79}{} & \EmpiricalEffect{+0.04}{0.99}{} & \EmpiricalEffect{-0.07}{2.04}{} & \EmpiricalEffect{+0.49}{2.27}{} & \EmpiricalEffect{+0.65}{1.39}{} & \EmpiricalEffect{+0.07}{0.15}{} & \EmpiricalEffect{-0.0004}{0.0021}{} \\
G3--G1 & \EmpiricalEffect{-0.10}{1.63}{} & \EmpiricalEffect{+0.06}{1.52}{} & \EmpiricalEffect{-0.12}{1.87}{} & \EmpiricalEffect{+0.37}{3.62}{} & \EmpiricalEffect{-0.70}{1.20}{} & \EmpiricalEffect{+0.09}{0.16}{} & \EmpiricalEffect{+0.0009}{0.0030}{} \\
G4--G1 & \EmpiricalEffect{-0.02}{1.95}{} & \EmpiricalEffect{-1.84}{1.95}{} & \EmpiricalEffect{-0.49}{2.52}{} & \EmpiricalEffect{-3.01}{4.84}{} & \EmpiricalEffect{+1.81}{3.57}{} & \EmpiricalEffect{-0.11}{0.59}{} & \EmpiricalEffect{-0.0040}{0.0092}{} \\
G5--G0 & \EmpiricalEffect{+0.90}{1.90}{} & \EmpiricalEffect{+0.30}{2.38}{} & \EmpiricalEffect{+2.83}{1.50}{^{\dagger}} & \EmpiricalEffect{-1.28}{1.28}{} & \EmpiricalEffect{+0.15}{0.38}{} & \EmpiricalEffect{-0.12}{0.07}{^{\dagger}} & \EmpiricalEffect{-0.0022}{0.0020}{^{\dagger}} \\
G6--G1 & \EmpiricalEffect{-0.05}{0.74}{} & \EmpiricalEffect{+0.38}{0.89}{} & \EmpiricalEffect{-0.85}{1.26}{} & \EmpiricalEffect{-0.47}{1.77}{} & \EmpiricalEffect{+0.87}{0.88}{} & \EmpiricalEffect{-0.05}{0.08}{} & \EmpiricalEffect{-0.0010}{0.0005}{^{\dagger}} \\
\bottomrule
\end{tabular*}
\endgroup
\end{table}

\textbf{Relational gains and set decisions respond differently.}
Table~\ref{tab:empirical-auxiliary-effects} shows that the stronger, faster-ramping SortedGrad schedule improves LD valid-query accuracy and last-visible ordering while reducing LD OverPred in all four paired seeds. Effects on MD pattern ACC, MD OverPred, and synthetic count accuracy vary across seeds; mean MD OverPred rises and count accuracy falls. Depth weighting and diversity likewise show mixed directions. Intensity supervision reduces mean MD OverPred at the cost of lower MD pattern exactness and higher LD OverPred.

AssignGrad shows a related trade-off: under weak SortedGrad, it improves last-visible ordering while worsening exact count and MAE in all three seeds. Under strong SortedGrad, its mean ordering effect reverses and MAE again worsens in every seed. Matched-seed comparisons could separate schedule effects from the different three-seed populations. These results motivate evaluating geometric supervision through both relational accuracy and full visible-set recovery.

\FloatBarrier

\clearpage
\subsection{Encoder Capacity and Decoder Design}
\label{sec:exp-architecture}

\paragraph{Encoder capacity.}
Holding the recurrent decoder and ExactMB recipe fixed, Figure~\ref{fig:architecture-comparison} examines encoder capacity. Across ViT-S/B/L, LD ACC rises from $53.60\%$ to $57.44\%$ and $63.62\%$, while synthetic count ACC increases from $81.50\%$ to $85.48\%$ and $88.15\%$. MD pattern accuracy improves mainly from Small to Base ($28.52\%$ to $46.69\%$), with Large at $46.24\%$. Encoder scaling therefore improves ordinal and count recovery more consistently than exact presence patterns.

\begin{figure}[!ht]
\centering
\includegraphics[page=5,width=\linewidth]{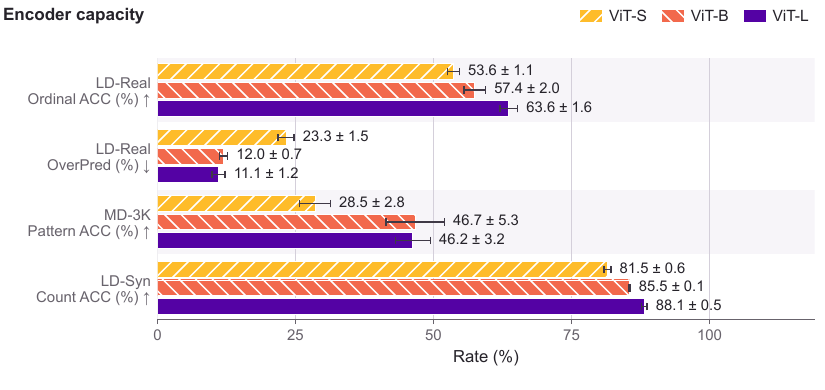}
\caption{\textbf{Encoder scaling improves ordinal and count recovery.}
ViT-S/B/L share a recurrent $K=4$ decoder, ExactMB, and SortedGrad weight $.05$.
Bars and error bars report means and sample SD across four matched seeds ($7/42/61/123$), at terminal-800k checkpoints with $\tau=.01$.}
\label{fig:architecture-comparison}
\end{figure}

\FloatBarrier

\paragraph{Decoder cost and recovery.}
\label{sec:exp-decoder-cost}

\begin{figure}[!ht]
\centering
\includegraphics[page=6,width=\linewidth]{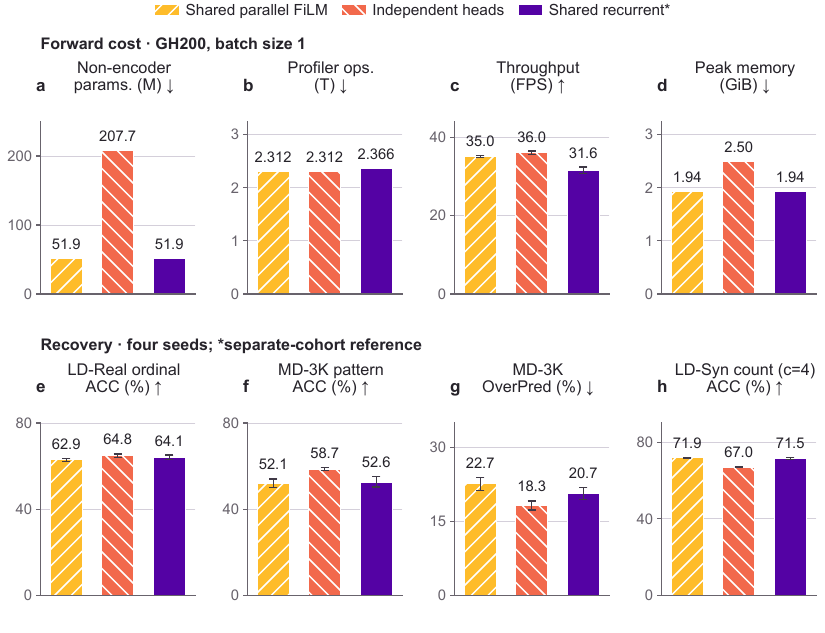}
\caption{\textbf{Decoder design trades storage for selective recovery gains.}
Top: GH200 120\,GB forward profiles ($518\times518$, batch size 1, TF32), with five warmups and 20 timed forwards per round.
Parameters exclude the 304.37M encoder. Operations use the archived MAC convention.
Bottom: ViT-L, $K=4$, ordered assignment, SortedGrad $.02$, terminal-800k checkpoints, and $\tau=.01$.
Parallel heads differ in weight sharing and parameter count. Shared recurrence (*) is a separate-cohort reference.
Error bars show sample SD across three timing rounds for FPS and four seeds for recovery.}
\label{fig:decoder-cost}
\end{figure}

Figure~\ref{fig:decoder-cost} examines decoder sharing through cost and recovery in a separate ordered-assignment study. Independent heads improve ordinal and presence-pattern accuracy over shared parallel FiLM, but selectively. LD-Syn occupied-micro count ACC rises by $.38\pm.13$ points, while four-return count ACC falls by $4.88\pm.39$ points and LD last-visible ACC falls by $3.87\pm1.30$ points (paired mean$\pm$sample SD). All four seeds share these directions. Aggregate accuracy can therefore improve while high-cardinality recovery and last-visible ordering deteriorate.

These selective gains also increase storage cost: independent heads use four times as many non-encoder parameters and raise peak memory from $1.94$ to $2.50$ GiB, despite similar operations and throughput. Shared recurrence retains compact storage but has lower profiled throughput. Profiles use randomly initialized models and exclude training, preprocessing, and postprocessing.

\FloatBarrier

\subsection{Archived Decoder Operating Points}
\label{sec:exp-practical-boundaries}
\label{sec:exp-robustness}

To relate recovery to decoding, Figure~\ref{fig:archived-recovery-operating-points} shows archived operating points from separate inference passes, not a fixed-prediction threshold sweep. It covers 12 primary MAP, ExactMB, and ordered-assignment checkpoints at $\tau\in\{.01,.1,.3,.5,.7,.9\}$. These terminal 800k checkpoints share finite-depth, depth-floor, and last-retained-gap filters. The principal threshold remains the predeclared $.01$, without test-set retuning or assumed shared calibration. Count-regularized and fixed-cardinality selectors ignore $\tau$, so export variation cannot reflect threshold sensitivity.

\begin{figure}[t]
\centering
\includegraphics[page=32,width=\linewidth]{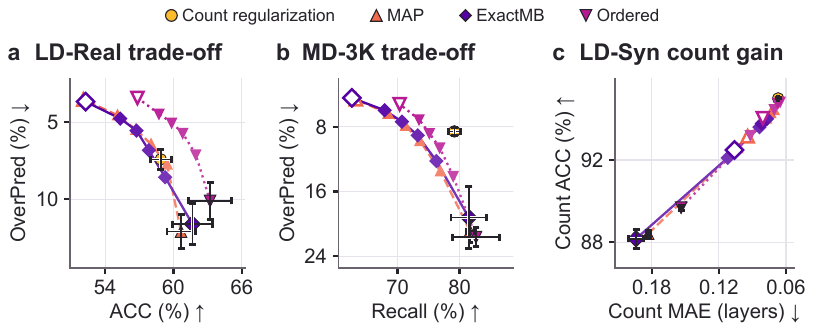}
\caption{\textbf{Archived recovery operating points.} (a) LD-Real ACC versus OverPred. (b) MD-3K recall versus OverPred. (c) LD-Syn occupied-micro count ACC versus MAE. Curves connect separate inference passes at $\tau=.01,.1,.3,.5,.7,.9$. Filled/hollow endpoints mark $.01/.9$. Means and SD are shown at $.01$. Upper right is preferred. Fixed-count output lies outside the displayed ranges. The count-regularized model retains its native count-based selection.}
\label{fig:archived-recovery-operating-points}
\end{figure}

Across same-checkpoint passes, ExactMB's $.5$ cutoff lowers LD ACC/OverPred by $4.90/6.07$ points, raises synthetic occupied-micro count ACC by $5.82$ points, and lowers count MAE by $.117$ layers relative to $.01$. Count agreement and false emission thus improve while ordinal recovery falls. The geometric precision--recall analysis below examines which accurate depths remain. Since cutoffs index separate exports, calibration and geometric dominance require separate tests.

\subsubsection{Geometric Precision--Recall Across Archived Operating Points}
\label{sec:primary-geometry-operating-points}

The rank-wise geometry archive contains 120 synthetic evaluations of all 20 primary terminal checkpoints, indexed by six cutoffs. The 12 Bernoulli-slot checkpoints trace operating-point curves, while the eight fixed-count and count-regularized checkpoints contribute their native exports. Each evaluation covers all 500 LD-Syn images. For each rank, we recover the integer number of depth-accurate predictions from the archived conditional $\delta_1$ score and valid-intersection count. Following Section~\ref{sec:primary-geometry}, dividing by GT-valid and prediction-valid counts gives recall and precision, respectively. At the native cutoff, all 80 rank records agree with the original primary-geometry report.

\begin{figure}[!htbp]
\centering
\includegraphics[page=18,width=\linewidth]{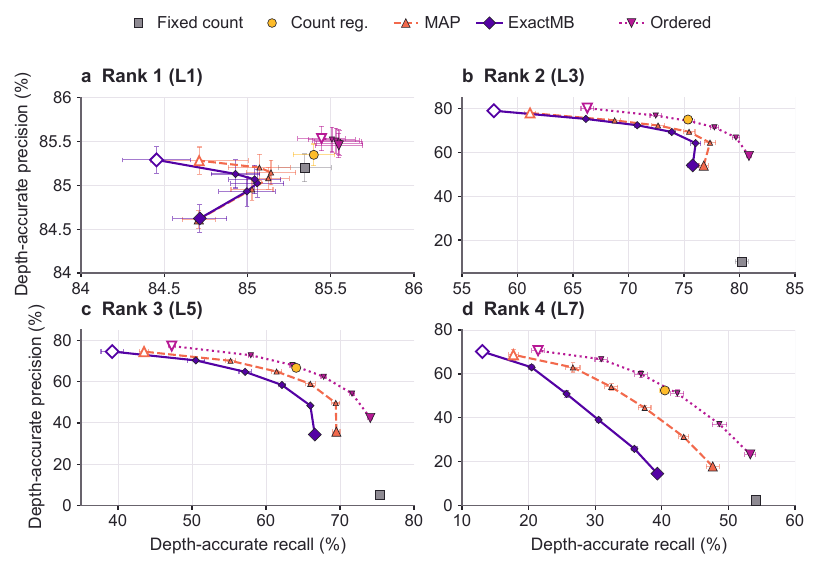}
\caption{\textbf{Rank-wise geometric operating points for the primary models.}
Depth-accurate recall and precision use the strict relative-error criterion $\max(\hat d/d,d/\hat d)<1.25$, without alignment or optional-depth clipping. Curves connect archived passes at $\tau=.01,.1,.3,.5,.7,.9$. Large filled/hollow markers denote $.01/.9$, with small markers at intermediate cutoffs. Fixed-count and count-regularized variants use native $.01$ exports. Points and whiskers give four-seed means and sample SD. Rank-wise scores leave injective whole-set recovery unassessed. Axis ranges vary.}
\label{fig:primary-geometric-operating-points}
\end{figure}

\par\addvspace{\intextsep}
\noindent\begin{minipage}{\linewidth}
\makeatletter\def\@captype{table}\makeatother
\centering
\AppendixTableSetup
\caption{\textbf{Geometry at a shared, existing cutoff.}
MAP, ExactMB, and ordered assignment use $\tau=.5$. The count-regularized model retains its native selector and $.01$ export. Entries are depth-accurate recall/precision (\%, mean$\pm$sample SD). A common cutoff does not ensure equal recall, equal precision, or matched probability calibration across methods.}
\label{tab:primary-geometric-operating-point}
\begingroup
\AppendixNumericTableStyle
\begin{tabular*}{\linewidth}{@{\extracolsep{\fill}}llrrrrrr@{}}
\toprule
 & & \multicolumn{2}{c}{\textbf{Rank 2 (L3)}} & \multicolumn{2}{c}{\textbf{Rank 3 (L5)}} & \multicolumn{2}{c}{\textbf{Rank 4 (L7)}} \\
\cmidrule(lr){3-4}\cmidrule(lr){5-6}\cmidrule(l){7-8}
\textbf{Method} & $\boldsymbol{\tau}$ & \textbf{Recall}$\uparrow$ & \textbf{Precision}$\uparrow$ & \textbf{Recall}$\uparrow$ & \textbf{Precision}$\uparrow$ & \textbf{Recall}$\uparrow$ & \textbf{Precision}$\uparrow$ \\
\midrule
Count reg. & --- & $75.36_{\mathord{\pm}0.35}$ & $74.96_{\mathord{\pm}0.35}$ & $64.08_{\mathord{\pm}0.07}$ & $66.67_{\mathord{\pm}0.57}$ & $40.47_{\mathord{\pm}0.73}$ & $52.46_{\mathord{\pm}0.62}$ \\
MAP & $.5$ & $72.66_{\mathord{\pm}0.59}$ & $72.20_{\mathord{\pm}0.29}$ & $61.43_{\mathord{\pm}0.69}$ & $64.87_{\mathord{\pm}0.57}$ & $32.41_{\mathord{\pm}0.99}$ & $54.06_{\mathord{\pm}1.46}$ \\
ExactMB & $.5$ & $70.78_{\mathord{\pm}0.20}$ & $72.35_{\mathord{\pm}0.59}$ & $57.18_{\mathord{\pm}0.83}$ & $64.73_{\mathord{\pm}0.68}$ & $25.70_{\mathord{\pm}0.20}$ & $50.85_{\mathord{\pm}1.71}$ \\
Ordered & $.5$ & $75.54_{\mathord{\pm}0.42}$ & $74.23_{\mathord{\pm}0.46}$ & $63.58_{\mathord{\pm}0.35}$ & $67.99_{\mathord{\pm}0.59}$ & $36.87_{\mathord{\pm}0.93}$ & $59.88_{\mathord{\pm}1.55}$ \\
\bottomrule
\end{tabular*}
\endgroup
\end{minipage}
\par\addvspace{\intextsep}

\paragraph{Higher precision can come at the cost of recall.}
Across ExactMB's archived $.01$ and $.5$ passes in Figure~\ref{fig:primary-geometric-operating-points}, fourth-rank precision rises from $14.51\%$ to $50.85\%$ while recall falls from $39.34\%$ to $25.70\%$: paired changes are $+36.34\pm1.72$ and $-13.64\pm.66$ points. Table~\ref{tab:primary-geometric-operating-point} shows that ordered assignment at $.5$ approaches the count-regularized model's rank-2/3 operating points. At rank four, it trades lower recall ($36.87\%$ versus $40.47\%$) for higher precision ($59.88\%$ versus $52.46\%$). The changing precision gaps emphasize the role of the chosen operating point in method comparisons.

\paragraph{Operating-point and metric scope.}
These are comparisons between archived prediction passes. Despite its $\tau$-independent selector, the count-regularized model varies across exports by up to $2.34$ recall and $2.21$ precision points. We therefore use its native export while the source of this variation remains unresolved. Curves retain annotation ties and same-rank pairing, with no calibrated or independently validated thresholds. Their aggregate records assess rank-wise recovery but lack the joint per-ray residuals needed for injective matching, whole-set correctness, or geometric-tolerance sweeps. Section~\ref{sec:iclr-geometric-learning} examines these on separate fixed rays with absolute tolerances.

\FloatBarrier

\subsection{Additional Real-World Qualitative Results}
\label{sec:exp-real-qualitative}

\textbf{Qualitative out-of-distribution multilayer depth comparisons.}
Figures~\ref{fig:qual-comparison-mirror}--\ref{fig:qual-captured-cafe} visualize retained depth on fourteen real-world images, comparing the supplementary video's auxiliary-regularized ExactMB checkpoint with released SeeGroup, LaRI, and World Tracing models. World Tracing uses the r69l scene and r75b object releases, matching Table~\ref{tab:iclr-external}. These scenes are out of distribution for our synthetic-only task training. ExactMB's deeper layers show spatially selective support, while several released-model outputs repeat foreground or background structure across ranks.

\textbf{Additional captured scenes.}
The eight additional photographs in Figures~\ref{fig:qual-captured-stemware}--\ref{fig:qual-captured-cafe} span tabletop glassware, decorative enclosures, flower arrangements, and larger interiors. Glassware and vase examples show support differences around transparent objects; bathroom, restaurant, and cafe views extend the comparison to cluttered interiors with localized transparency. Together, they illustrate how depth and retained support vary across ranks and scene content.

\textbf{Display protocol.}
We sort valid decoded depths front to back per pixel. Each method and image uses a logarithmic color scale shared across retained ranks. Colors indicate relative depth, not comparable distances across methods. Gray denotes no retained depth, and dashes mark ranks beyond decoded capacity. ExactMB retains finite depths above $10^{-4}$\,m with $q>.5$, without gap filtering. SeeGroup uses its released evaluation decoder, while LaRI's object model retains native stopping. LaRI and World Tracing object models receive full-frame images without automatic background removal.

\begin{figure}[!p]
\centering
\includegraphics[page=24,width=\linewidth]{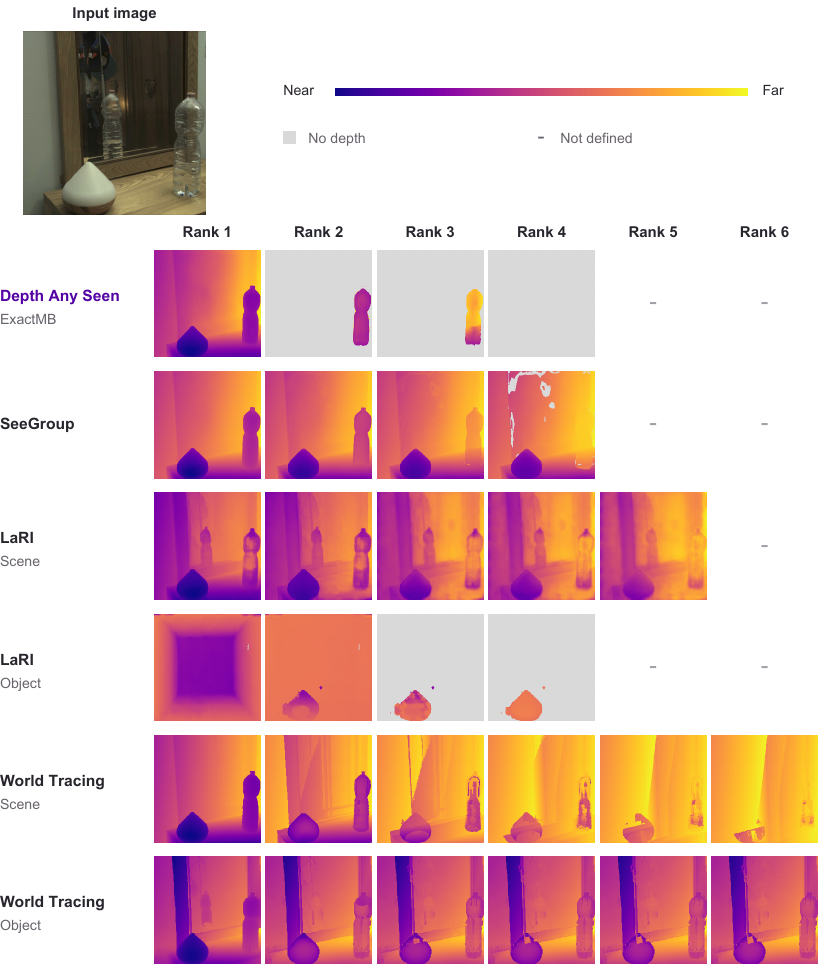}
\caption{\textbf{Qualitative out-of-distribution multilayer depth comparisons.} Mirror and bottle.}
\label{fig:qual-comparison-mirror}
\end{figure}

\begin{figure}[!p]
\centering
\includegraphics[page=29,width=\linewidth]{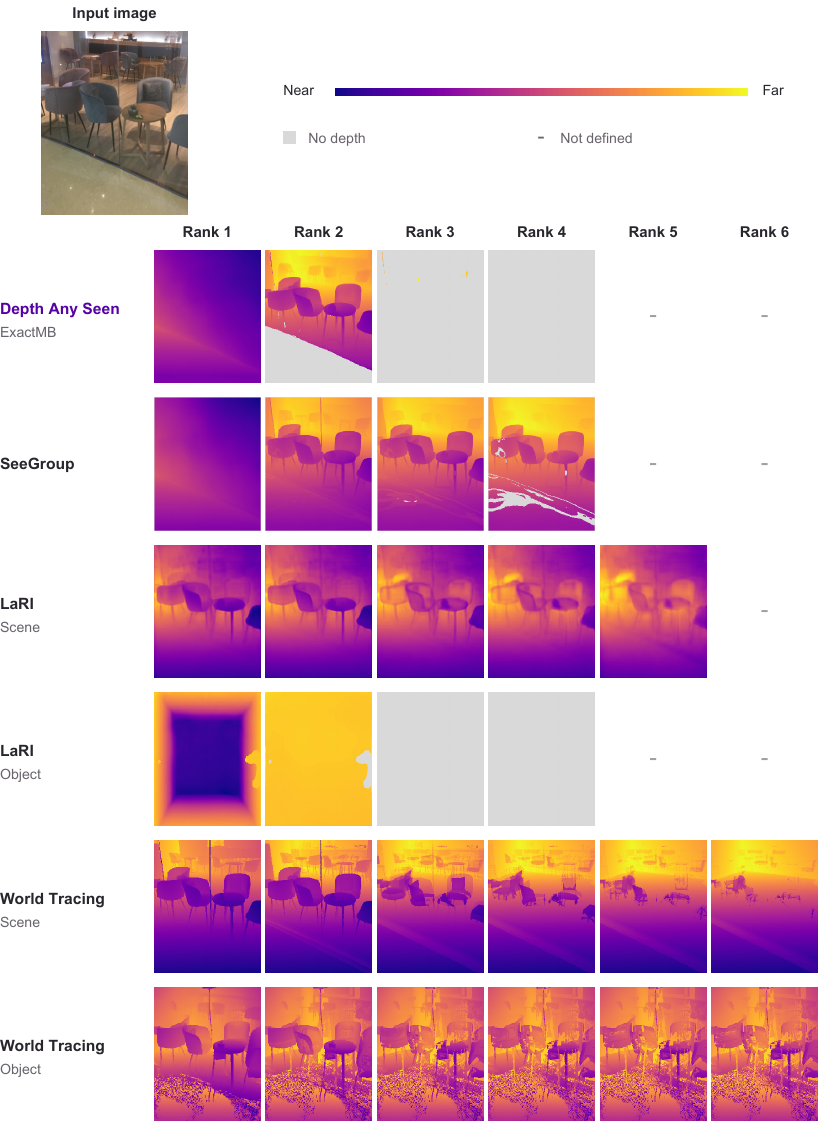}
\caption{\textbf{Qualitative out-of-distribution multilayer depth comparisons.} Chairs and tables viewed through a glass partition, with seating at different depths (MD-3K 743).}
\label{fig:qual-comparison-seating}
\end{figure}

\begin{figure}[!p]
\centering
\includegraphics[page=28,width=\linewidth]{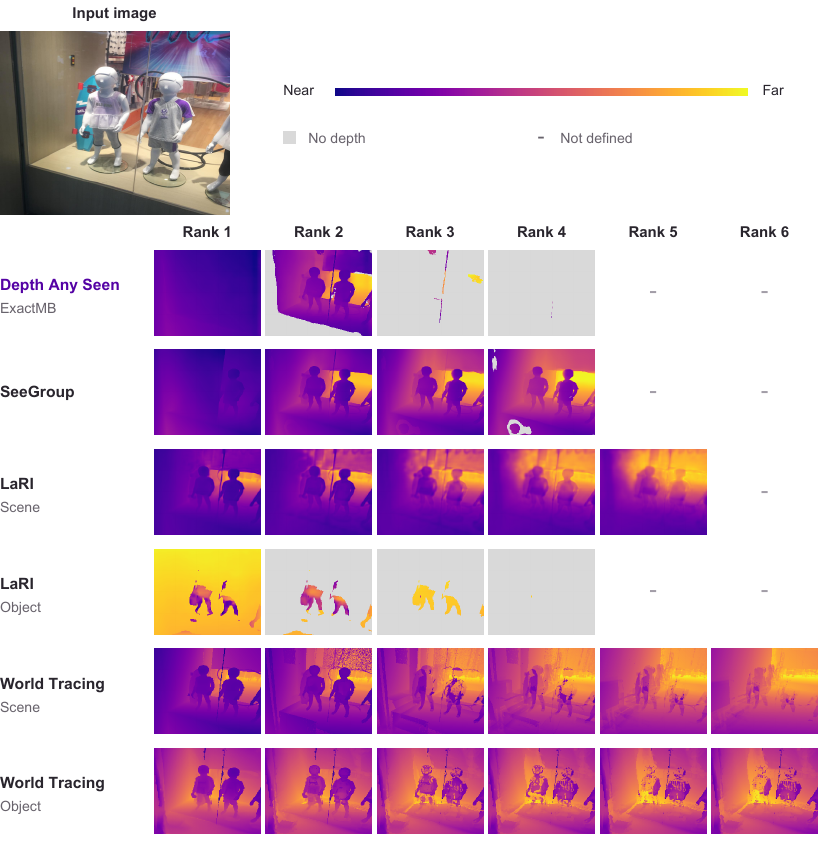}
\caption{\textbf{Qualitative out-of-distribution multilayer depth comparisons.} A window display with two astronaut figures behind glass and a colorful backdrop (MD-3K 540).}
\label{fig:qual-comparison-display}
\end{figure}

\begin{figure}[!p]
\centering
\includegraphics[page=27,width=\linewidth]{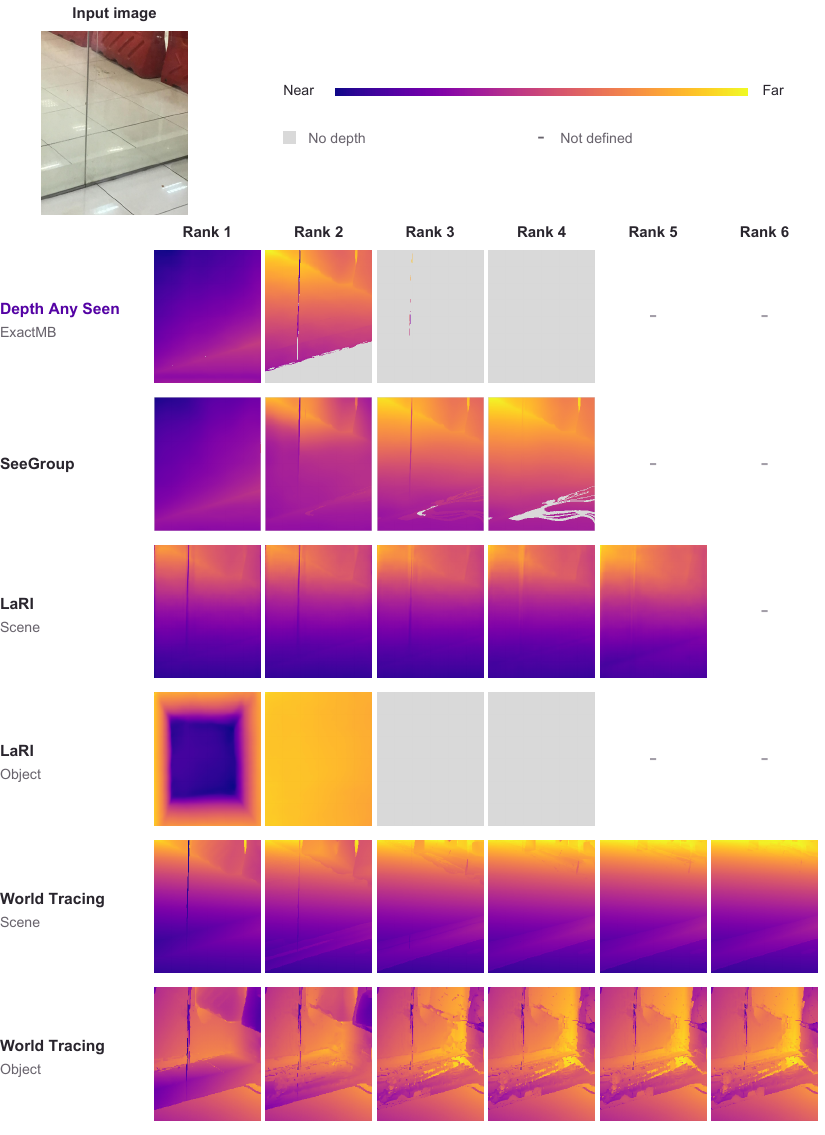}
\caption{\textbf{Qualitative out-of-distribution multilayer depth comparisons.} A glass entrance above a tiled floor, with a door seam and red seating behind the glass (MD-3K 3134).}
\label{fig:qual-comparison-entrance}
\end{figure}

\begin{figure}[!p]
\centering
\includegraphics[page=26,width=\linewidth]{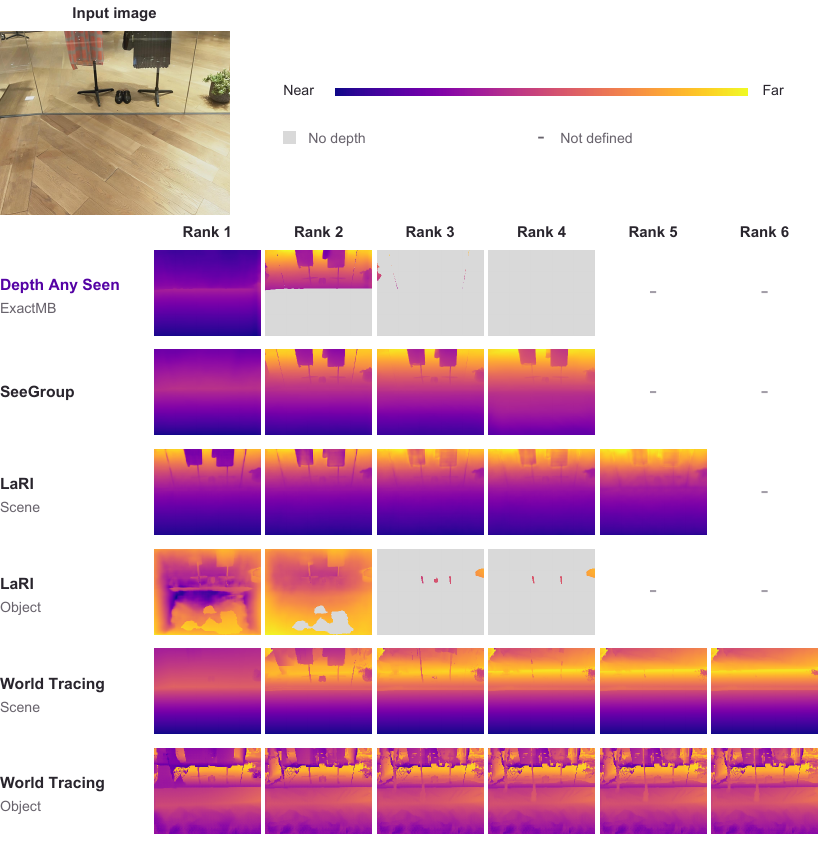}
\caption{\textbf{Qualitative out-of-distribution multilayer depth comparisons.} A clothing display behind glass, with hanging garments above a wooden floor (MD-3K 2814).}
\label{fig:qual-comparison-clothing}
\end{figure}

\begin{figure}[!p]
\centering
\includegraphics[page=25,width=\linewidth]{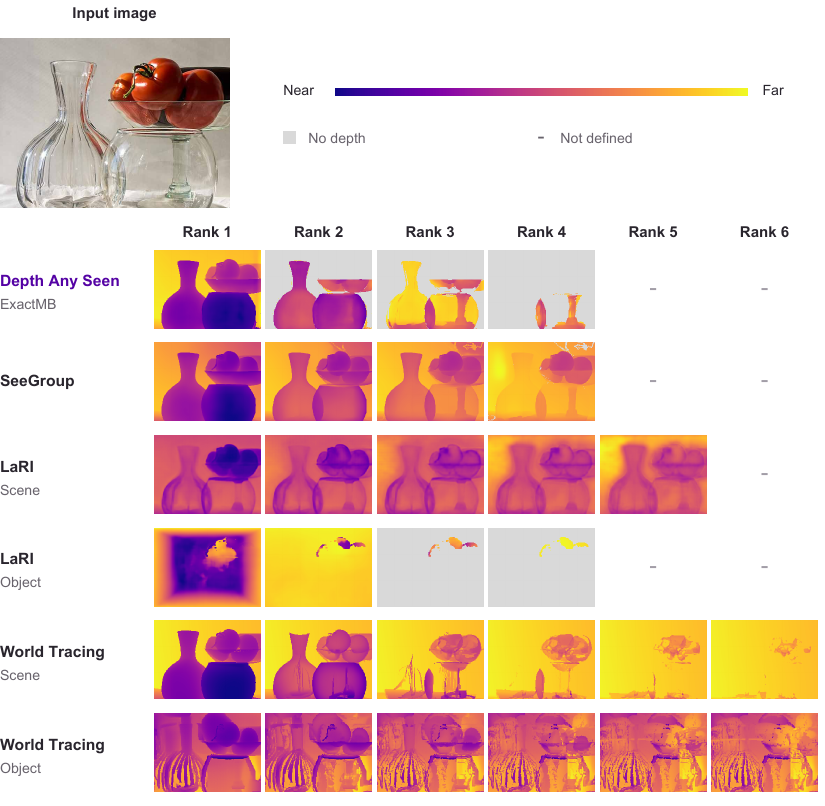}
\caption{\textbf{Qualitative out-of-distribution multilayer depth comparisons.} Clear glass vessels beside a raised fruit bowl, with overlapping transparent surfaces (LD-Real 179).}
\label{fig:qual-comparison-vessels}
\end{figure}

\begin{figure}[!p]
\centering
\includegraphics[page=15,width=\linewidth]{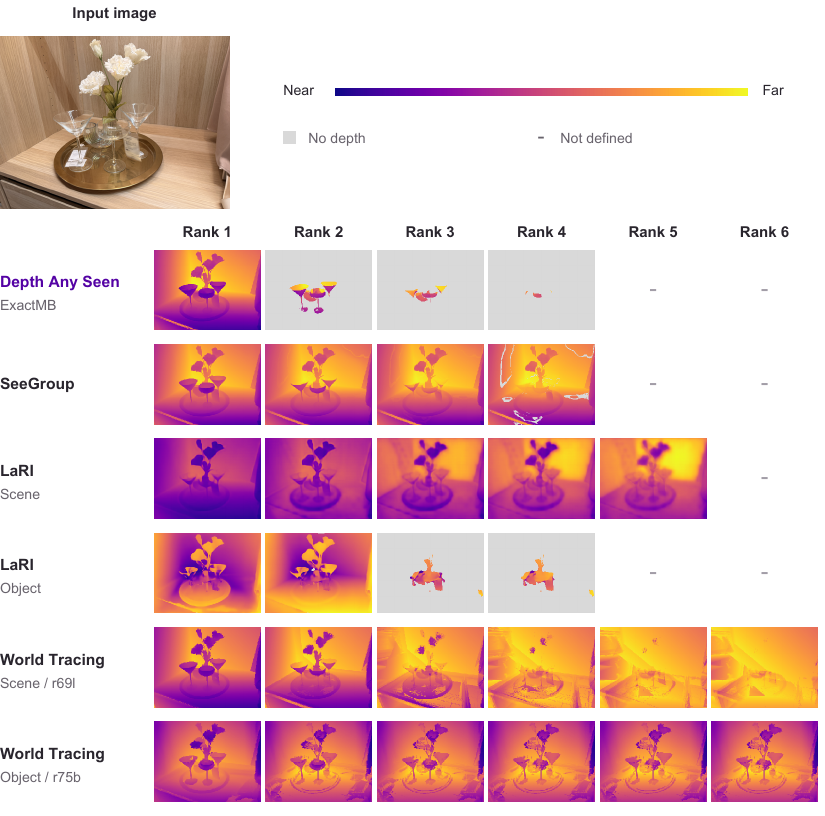}
\caption{\textbf{Qualitative out-of-distribution multilayer depth comparisons.} Captured glassware arrangement with white flowers and stemware on a round tray.}
\label{fig:qual-captured-stemware}
\end{figure}

\begin{figure}[!p]
\centering
\includegraphics[page=12,width=\linewidth]{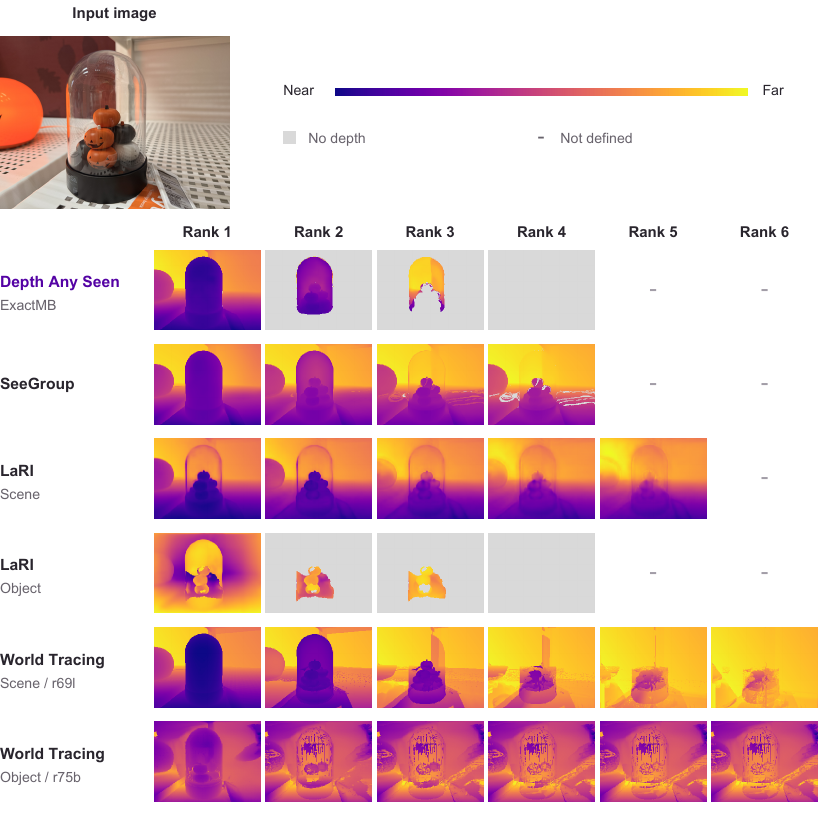}
\caption{\textbf{Qualitative out-of-distribution multilayer depth comparisons.} Captured display of stacked pumpkin decorations enclosed by a transparent glass dome.}
\label{fig:qual-captured-dome}
\end{figure}

\begin{figure}[!p]
\centering
\includegraphics[page=9,width=\linewidth]{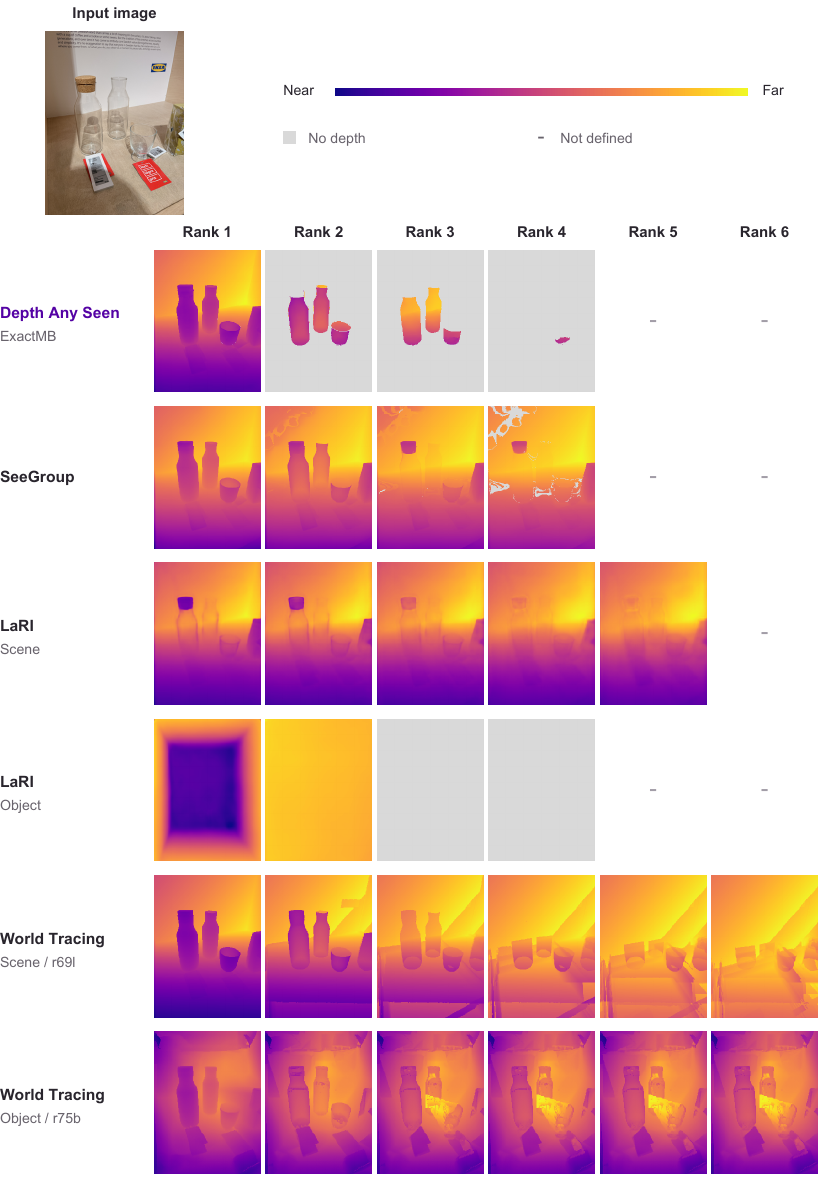}
\caption{\textbf{Qualitative out-of-distribution multilayer depth comparisons.} Captured tabletop display with clear bottles and a tumbler against an opaque background.}
\label{fig:qual-captured-bottles}
\end{figure}

\begin{figure}[!p]
\centering
\includegraphics[page=13,width=\linewidth]{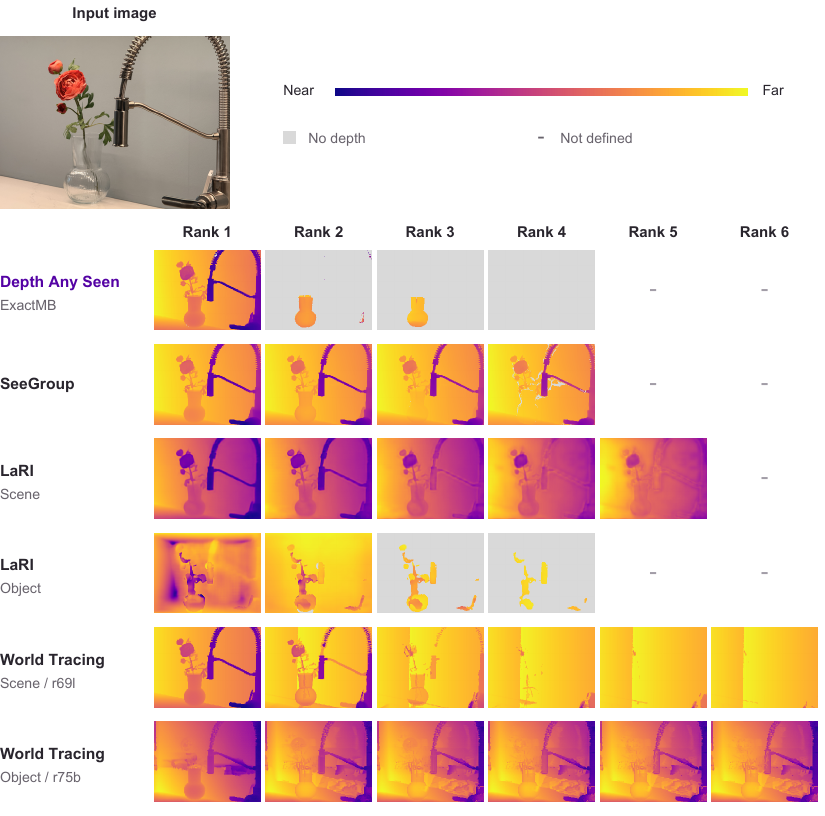}
\caption{\textbf{Qualitative out-of-distribution multilayer depth comparisons.} Captured kitchen scene with red flowers in a ribbed glass vase beside a metal faucet.}
\label{fig:qual-captured-faucet}
\end{figure}

\begin{figure}[!p]
\centering
\includegraphics[page=8,width=\linewidth]{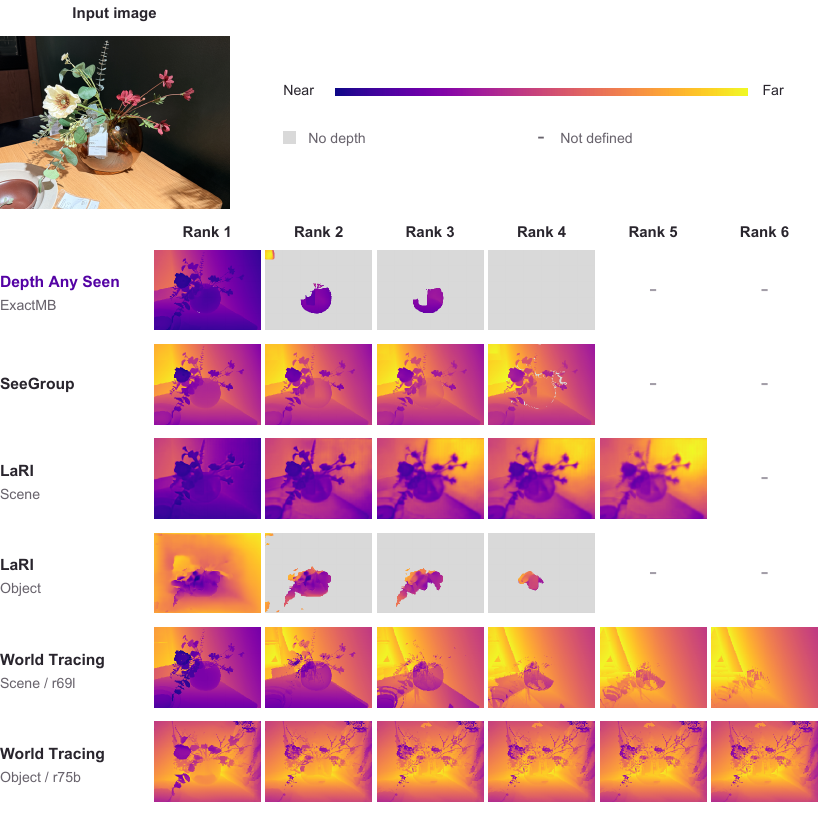}
\caption{\textbf{Qualitative out-of-distribution multilayer depth comparisons.} Captured flower arrangement in a rounded amber glass vase on a wooden tabletop.}
\label{fig:qual-captured-amber}
\end{figure}

\begin{figure}[!p]
\centering
\includegraphics[page=14,width=\linewidth]{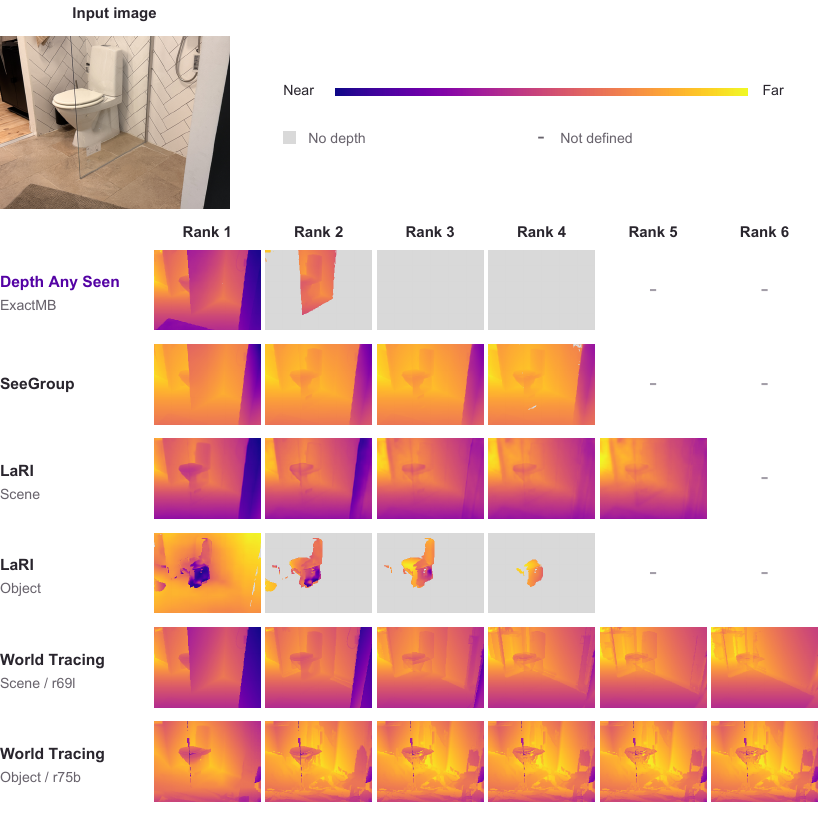}
\caption{\textbf{Qualitative out-of-distribution multilayer depth comparisons.} Captured bathroom interior with a glass shower screen in front of a toilet and tiled wall.}
\label{fig:qual-captured-shower}
\end{figure}

\begin{figure}[!p]
\centering
\includegraphics[page=11,width=\linewidth]{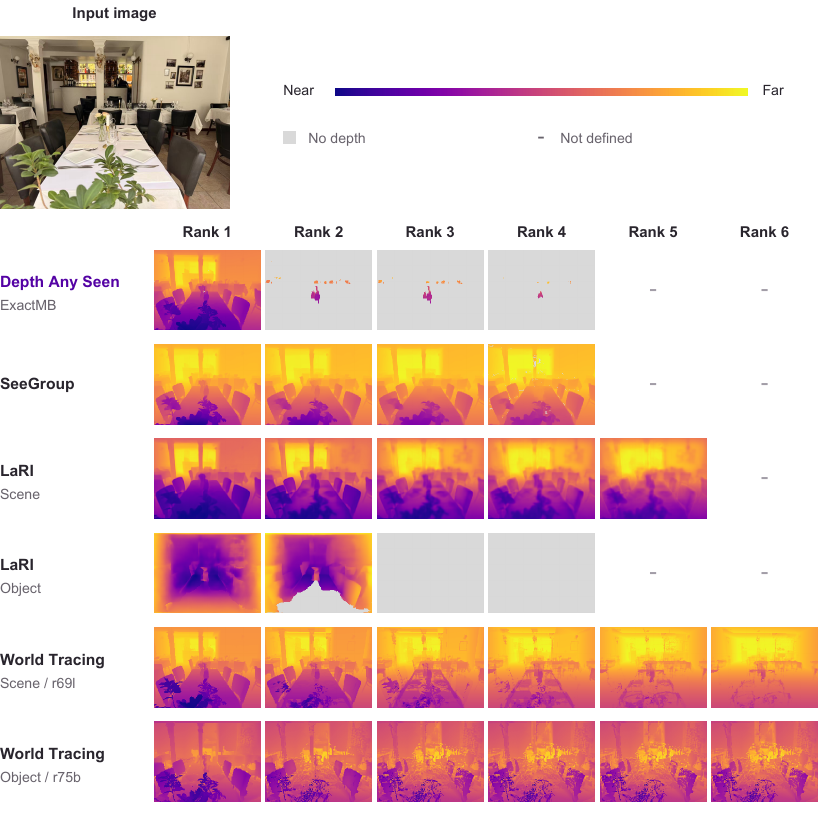}
\caption{\textbf{Qualitative out-of-distribution multilayer depth comparisons.} Captured restaurant interior with glassware along a dining table and furnishings behind it.}
\label{fig:qual-captured-dining}
\end{figure}

\begin{figure}[!p]
\centering
\includegraphics[page=10,width=\linewidth]{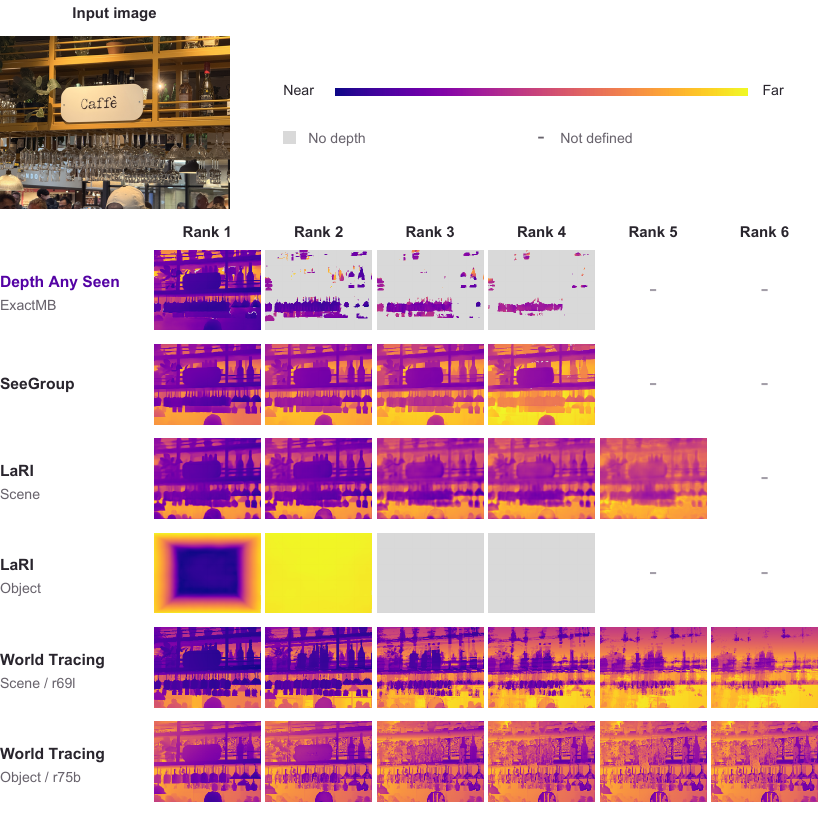}
\caption{\textbf{Qualitative out-of-distribution multilayer depth comparisons.} Captured cafe interior with hanging glassware, bottles, and shelves at different depths.}
\label{fig:qual-captured-cafe}
\end{figure}

\FloatBarrier

\FloatBarrier

\FloatBarrier
\Needspace{10\baselineskip}
\section{Learning Dynamics and Geometric Recovery}
\suppressfloats[t]
\label{sec:iclr-count-geometry-diagnostics}

The aggregate gradient in Equation~\eqref{eq:iclr-mass} links ExactMB to expected
count, but emitted cardinality and geometric recovery also depend on individual
presence scores and candidate depths.
The primary four-seed histories track expected-count fitting
and discrete activation. A separate auxiliary-regularized trajectory follows
uncertainty and geometry on fixed rays, distinguishing missing candidate depths
from losses during decoding. A third study examines cardinality weighting and
retained support. The latter two studies provide descriptive evidence beyond the
primary replications.

\subsection{Count, Activation, and Held-Out Learning}
\label{sec:exp-optimization}
\label{sec:iclr-replicated-learning}

We first compare the primary MAP, ExactMB, and ordered-assignment histories, using
seeds $7/42/61/123$ through update 800,000. Training diagnostics average
minibatches, while synthetic validation uses 64 fixed images and weights contributing
images equally within each target cardinality. We measure raw activation at
$q>.5$, compared with $.01$ in the
principal decoder. Numerical summaries use unsmoothed records without selecting
maxima, with sample SD across training seeds.

\begin{figure}[!htbp]
 \centering
 \includegraphics[page=2,width=.99\textwidth]{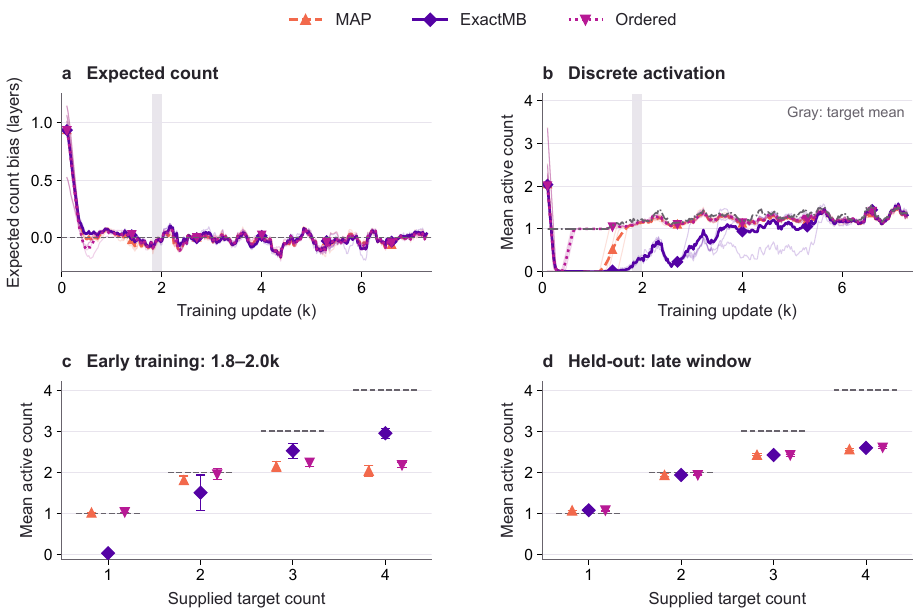}
 \caption{\textbf{Expected count and discrete activation follow different trajectories.}
 Learning histories for primary MAP, ExactMB, and ordered assignment under the
 protocol in Section~\ref{sec:exp-optimization}.
 (a) Minibatch expected-count bias. (b) Minibatch thresholded count at $q>.5$.
 Thin curves show individual seeds and bold curves their means, smoothed over a
 centered 200-update window for display. Shading marks the raw 1.8--2.0k summary window.
 (c) Conditional counts in that window, weighted by cardinality-specific pixel
 exposure. (d) Held-out conditional counts, averaging seven checkpoints from
 762.2k to 800k on 64 fixed images. Symbols and whiskers show seed mean
 $\pm$ sample SD. Dashed marks denote target counts.
 Diagnostic and principal gates are $.5$ and $.01$.}
 \label{fig:primary-learning-dynamics}
\end{figure}

\paragraph{Expected count and discrete activation.}
The expected count is $\bar c=\sum_jq_j$, whereas the thresholded count is
$\hat c_{.5}=\sum_j\mathbf1[q_j>.5]$. The summed gradient
in~\eqref{eq:iclr-mass} supplies an expected-count signal without determining
the occupied subset. For example, four probabilities of $.25$ have expected count
one but activate none at $.5$. Even expected-count agreement in aggregate can hide
canceling ray-wise errors.
Figure~\ref{fig:primary-learning-dynamics} compares expected and thresholded counts,
both overall and by target cardinality.

This distinction is visible early in training. Over raw updates 1.8--2.0k, the target mean is $1.200\pm.027$ layers.
MAP, ExactMB, and ordered assignment predict expected means
$1.167\pm.012$, $1.175\pm.020$, and $1.174\pm.021$, but thresholded means
$1.126\pm.012$, $.274\pm.047$, and $1.145\pm.022$.
ExactMB's activation delay is cardinality-selective: one-layer pixels supply $88.7\%$ of early
exposure, and ExactMB activates $.037\pm.010$ components at $c=1$, versus
MAP's $1.023\pm.007$ and ordered assignment's $1.029\pm.005$.
At $c=4$, ExactMB instead activates $2.954\pm.119$, versus
$2.041\pm.136$ and $2.172\pm.060$, and is higher in every seed.
These activation differences occur despite similar expected counts. Early same-ray
diagnostics and neural subset/correspondence controls would help test the mechanism
behind this delay and any relation to the symmetric saddle.

\paragraph{Held-out high-cardinality deficit.}
By termination, the largest activation deficit occurs at high cardinality. Over the seven-checkpoint
terminal window, held-out counts at $q>.5$ average
$1.07$--$1.08$, $1.93$--$1.94$, $2.42$--$2.43$, and $2.56$--$2.60$ layers
for $c=1/2/3/4$. The $c=4$ stratum occurs in 53 of 64 images but only
$1.42\%$ of pixels, so all-pixel averages dilute a roughly $1.4$-layer
conditional deficit. To examine these learning signals, we compare
costs near 125k and termination:
training uses complete 7,400-update cycles and excludes the single-layer-dominated
final partial cycle, while validation averages three baseline checkpoints and seven
checkpoints in the terminal window.

We probe all three variants with a temperature-one
injection posterior and occupied-existence cost $C_{\rm occ}=-\sum_j\rho_j\log q_j$.
Only ExactMB uses this posterior in training; MAP and ordered assignment use their
own assignment rules. In all four seeds, each variant reduces this diagnostic cost
on training batches but increases it on held-out data.
For ExactMB, the paired changes are $-.0365\pm.0037$ and
$+.0681\pm.0352$ nats/pixel. Interpretation must account for differences in cropping, model mode,
composition, averaging, and posterior weights. Training is pixel-exposure-weighted,
while validation weights contributing images equally within each cardinality.
Matched training--validation analysis would help relate these contextual
differences to the held-out high-cardinality deficit.

\subsection{Fixed-Ray Uncertainty and Geometric Recovery}
\label{sec:saved-ray-dynamics}

Count histories alone do not show whether the proposed depths explain the target
surfaces. We therefore track assignment uncertainty and geometric recovery in an
ExactMB run (seed 6) with an assignment-aligned gradient auxiliary of weight $.05$.
Native $432\times768$ predictions cover 20 evenly spaced validation
images at nine epochs $22/40/61/82/100/121/142/160/181$
(162.8k--1,339.4k updates). This separate auxiliary-regularized trajectory complements
the primary four-seed study with a view of late-stage refinement.
A fixed random seed samples up to 64 rays per cardinality and image
from first-stage target masks. Sampled coordinates and supplied
target depths remain fixed.

Excluding 64 rays with no valid targets and 185 occupied rays with exact ties leaves
4,513 positive, finite, distinct-target rays. The $M=1/2/3/4$ supports are
$1{,}280/1{,}216/1{,}152/865$ rays from $20/19/19/18$ images. Tie exclusion removes
184 of 1,049 sampled four-target rays ($17.54\%$). One three-target image retains
only one ray. We first average within each image/cardinality group, then equally across
contributing images, with variability reported across images. Without a complete-ray
flag, recovery is evaluated against supplied targets of unknown completeness.

\paragraph{Assignment uncertainty.}
For each supplied distinct target $\mathcal T$ of size $M$, we evaluate the injection posterior $\omega_\phi$ from Equation~\eqref{eq:iclr-posterior}. Let $\Phi$ denote the random assignment and $S=\im(\Phi)$ its occupied subset. Using~\eqref{eq:count-conditional-subset-terms}, the subset posterior is
\begin{equation}
 \pi_S=\sum_{\phi:\im(\phi)=S}\omega_\phi
 =\frac{w_q(S)G_S(\mathcal T)}{\sum_{S'}w_q(S')G_{S'}(\mathcal T)}.
 \label{eq:subset-posterior}
\end{equation}
Geometry reweights each subset's prior. Within a fixed subset, presence factors cancel, leaving correspondence dependent on localization densities. Shannon entropy with natural logarithms separates uncertainty about which components are occupied from uncertainty about their target correspondence:
\begin{equation}
 H(\Phi\mid\mathcal T)=H(S\mid\mathcal T)
 +\E_{S\mid\mathcal T}H(\Phi\mid S,\mathcal T).
 \label{eq:selection-entropy-chain}
\end{equation}
The subset and correspondence terms have maxima $\log\binom KM$ and $\log M!$, respectively. We normalize each term per ray by its own maximum, assigning zero to a singleton event space, then average equally across contributing images within each cardinality. Missing strata are excluded. The separately normalized terms need not sum to the normalized total assignment entropy.

\begin{figure}[!t]
 \centering
 \includegraphics[page=23,width=.96\textwidth]{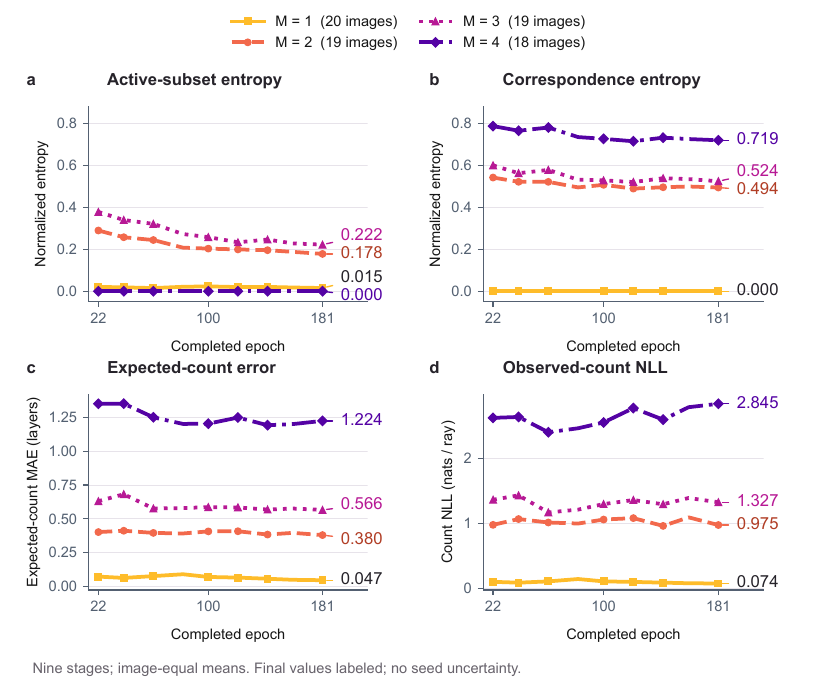}
 \caption{\textbf{Lower assignment uncertainty need not improve the count log score.}
 An auxiliary-regularized ExactMB trajectory (seed 6, weight $.05$), separate from
 the primary study, evaluated on the fixed rays with distinct supplied targets in
 Section~\ref{sec:saved-ray-dynamics}.
 (a) Active-subset entropy. (b) Conditional-correspondence entropy. Both are
 normalized per ray by their respective combinatorial maxima. Subset entropy at $M=4$ and
 correspondence entropy at $M=1$ are zero by construction.
 (c) Absolute expected-count error. (d) Pre-decoder observed-count NLL.
 Lines connect nine measured stages without smoothing. Means weight images
 equally within cardinality. All stages follow the same sampled rays and supplied
 targets within this single training run.}
 \label{fig:archived-uncertainty}
\end{figure}

\paragraph{Assignment uncertainty and count scores.}
Figure~\ref{fig:archived-uncertainty} tracks these normalized entropies alongside expected-count error and count NLL. From first to last stage, subset entropy falls
$.289\to.178$ at $M=2$ and $.378\to.222$ at $M=3$, while correspondence entropy
changes $.542\to.494$ and $.600\to.524$. At $M=4$, subset entropy is identically
zero, but correspondence remains $.787\to.719$. Lower entropy describes more
concentrated assignments, but count likelihood need not improve in parallel.
For $M=4$, expected-count MAE falls $1.350\to1.224$ layers while count
NLL rises $2.627\to2.845$ nats/ray. Paired changes are
$-.126\pm.415$ layers and $+.217\pm1.575$ nats. At~$M=K$, these scores are
$\sum_j(1-q_j)$ and $-\sum_j\log q_j$: their different penalties permit opposite
trends. The image-level dispersion describes variation within this single training
run. Geometric evaluation is also needed because correspondence depends on the
learned localization scales.

These entropy trends are insensitive to the numerical treatment of saturated
probabilities. Without saved logits, sigmoid endpoints are moved to the nearest
interior float32 value. Clipping instead to $[10^{-6},1-10^{-6}]$ changes
image-equal mean absolute subset and correspondence entropies by at most
$7.52\times10^{-7}$ and $1.42\times10^{-9}$ nats, respectively. No sampled
observed-count probability is zero.

\paragraph{Aggregate summaries obscure multilayer uncertainty.}
To assess the effect of population averaging, we extend the analysis beyond the sampled distinct-target rays above to
all $6{,}635{,}520$ pixels in the same 20 archived rasters. Only $8.67\%$ have at least
two supplied targets. From epochs 22 to 181, the fraction of pixel--slot
probabilities with $q_j\leq.05$ or $q_j\geq.95$ rises $89.90\to93.30\%$ overall,
versus $39.41\to60.41\%$ on multilayer rays. Pooling obscures uncertainty
on multilayer rays.

Localization scales show a related aggregation effect. For posterior-co-occupied
slot pairs, let the multiplicative scale gap be
$G_\beta=\exp(\mathbb E[|\log\beta_j-\log\beta_k|])$, with expectations weighted
by joint posterior occupancy across pixels and pairs. This gap narrows
$1.513\to1.311$, while the fraction of pair weight with both scales at most
$.101$\,m rises $22.34\to53.79\%$. The model's lower bound is $.1$\,m.
Conditioning instead on at least one scale above $.101$\,m gives gaps
$1.703\to1.795$, without comparable narrowing. Pair statistics draw on 19
images with multilayer targets. Thus, aggregate narrowing accompanies concentration
near the scale floor, rather than comparable narrowing among pairs with larger
scales. Empirical calibration and convergence call for separate diagnostics.

\subsubsection{Candidate Geometry Versus Selection and Filtering}
\label{sec:proposal-emission-dynamics}
\label{sec:iclr-geometric-learning}

We next connect these uncertainty summaries to geometric recovery by separating the depths the network proposes
from those the decoder retains. Let $D_{\rm all}$ contain all finite native centers and $D_{\rm emit}$ the
thresholded, depth-filtered, last-retained-gap output, preserving native component
indices. At tolerance $\epsilon$, $m_\epsilon(\mathcal T,D)$ is the maximum number
of one-to-one matches with residual at most $\epsilon$. One candidate cannot
recover multiple targets. On occupied rays,
\begin{equation}
 \begin{aligned}
 R_{\rm proposal}&=m_\epsilon(\mathcal T,D_{\rm all})/M,\\
 R_{\rm emit}&=m_\epsilon(\mathcal T,D_{\rm emit})/M,\\
 N_{\rm extra}&=|D_{\rm emit}|-m_\epsilon(\mathcal T,D_{\rm emit}).
 \end{aligned}
 \label{eq:proposal-emission-matching}
\end{equation}
Complete geometric recovery requires both the correct output count and a
distinct match for every target:
\begin{equation}
 \operatorname{SetOK}_\epsilon
 =\mathbf1[|D_{\rm emit}|=M\ \land\ m_\epsilon(\mathcal T,D_{\rm emit})=M].
 \label{eq:geometric-set-exactness}
\end{equation}
Proposal recall measures how many targets the native-center pool can explain before
selection, using ground-truth one-to-one matching. Unmatched emissions include both
redundancy and mislocalization, so their interpretation differs from annotation-defined
OverPred. These same-ray diagnostics use absolute tolerance, whereas the primary
rank-wise scores use relative error on a different cohort.

We fix $\epsilon=.05$\,m before inspection and test $.02/.10$\,m sensitivity.
Decoding uses strict $q>.01$, depth $>.02$\,m, and a strict $.02$\,m last-retained
gap. Each decoding stage deletes candidates without moving centers. Differences in
recall between successive stages measure losses at the existence gate, depth floor,
and gap filter. Together with emitted recall and the fraction of targets missed by
all candidates, these losses sum to one. Figure~\ref{fig:archived-geometry} accounts
for this budget on the same fixed rays.

\begin{figure}[!t]
 \centering
 \includegraphics[page=7,width=.99\textwidth]{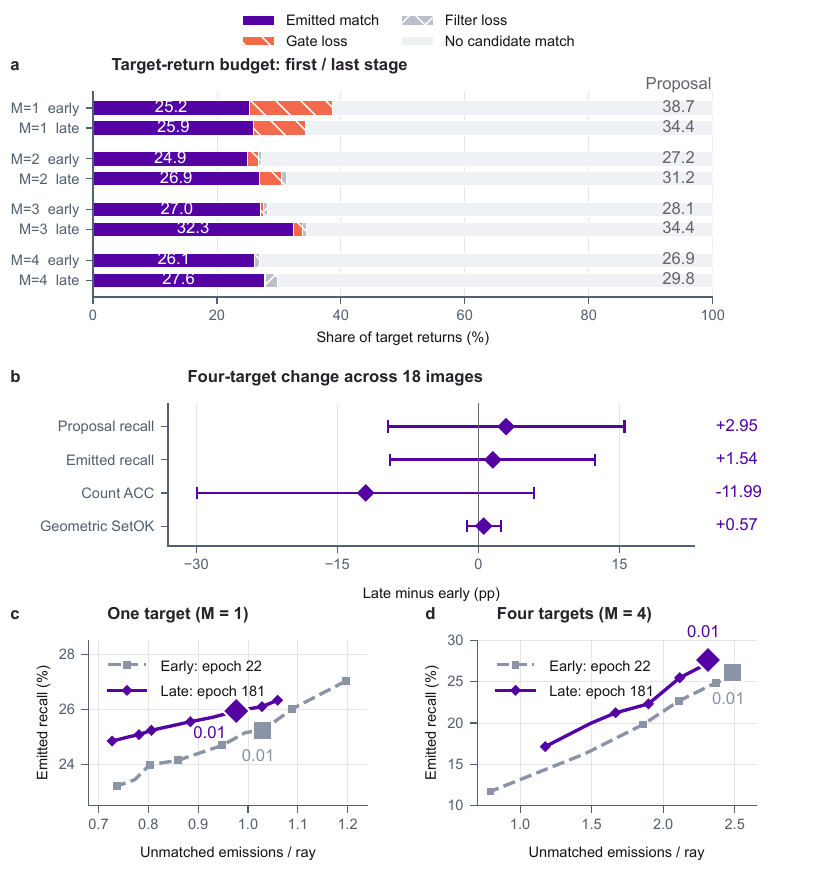}
 \caption{\textbf{Candidate geometry, emitted support, and count accuracy evolve differently on the same rays.}
 ExactMB (seed 6) with auxiliary regularization, matching tolerance $\epsilon=.05$\,m,
 and the fixed rays with distinct supplied targets defined in Section~\ref{sec:saved-ray-dynamics}.
 (a) First/last archived stages (epochs 22/181). Stacked bars account for the
 full target-return budget: emitted matches, gate losses, depth-floor/gap losses,
 and the deficit of all-candidate one-to-one matching. Numbers inside solid
 emitted-match segments give emitted recall. The right column gives proposal
 recall. All quantities are averaged equally across images using identical rays.
 (b) Within-image last-minus-first changes for four-target rays. Symbols and
 whiskers show mean $\pm$ sample SD across 18 images rather than training seeds.
 (c) First/last threshold curves for $M=1$. (d) First/last threshold curves for $M=4$.
 Only the existence threshold varies across nine fixed values from $.001$ to $.9$.
 Large symbols mark $\tau=.01$. Curves connect recall and unmatched-emission
 measurements across thresholds.}
 \label{fig:archived-geometry}
\end{figure}

\paragraph{Missing candidate geometry dominates the late four-target deficit.}
At the last stage of this auxiliary-regularized run, four-target proposal recall
at $.05$\,m is $29.83\%$ and emitted recall $27.63\%$.
All-candidate matching misses $70.17\%$ of target returns, versus $.224$ percentage points
lost at the gate, $1.984$ at the gap rule, and none at the depth floor. The dominant bottleneck is missing accurate candidates, rather than loss during $.01$ gating or subsequent gap filtering.

Proposal recall improves by $2.95\pm12.59$ points and emitted recall by
$1.54\pm10.90$ points across the 18 paired four-target images. Proposal/emit
recall increase in $13/12$ images, decline in two images each, and tie in the
others. Meanwhile gap-stage loss rises $.80\to1.98$ points and unmatched
emissions fall $2.487\to2.313$ per ray. Better candidate recovery and additional
pruning therefore coexist, even though the number of emitted depths decreases.
At $M=1$, proposal recall instead falls $38.67\to34.38\%$ while emitted recall
changes $25.23\to25.94\%$: the aggregate gate-loss budget shrinks even though
all-candidate matching recovers a smaller share of targets.

\paragraph{Count agreement and geometric recovery can move in opposite directions.}
On four-target rays, exact output count falls $64.37\to52.38\%$, but
$\operatorname{SetOK}_{.05}$ rises $.121\to.686\%$. Although geometric recovery
remains rare, its paired mean change is $+.565\pm1.824$ points, with three positive
images, one negative, and 14 ties. These opposing trends motivate evaluating count
agreement together with geometry, and highlight reliable four-surface recovery as
an open challenge. At $.02$\,m, emitted recall changes
$13.44\to13.88\%$ and complete geometric recovery remains zero. At $.10$\,m
they change $41.14\to42.57\%$ and $4.42\to5.35\%$. Absolute recovery levels
depend on geometric tolerance, but exact count and simultaneous geometric
recovery remain different evaluation events.

For targets separated by more than $\max(2\epsilon,.02\,\mathrm m)$, the
$.05$\,m four-target sensitivity retains only 137 rays in ten images. The
$.10$\,m sensitivity has just ten rays in three images. These small strata motivate
broader sampling to compare closely spaced and well-separated target surfaces.

To distinguish geometric learning from score changes, note that with fixed centers,
a common strictly increasing score transformation preserves
all subsets obtainable by thresholding. Changes in
proposal recall in Figure~\ref{fig:archived-geometry} therefore reflect evolving
candidate geometry beyond score remapping. These curves characterize score and
geometry changes, motivating targeted causal and calibration analyses. The
principal threshold remains fixed throughout the analysis.

\FloatBarrier
\subsection{Cardinality Weighting and Retained Support}
\label{sec:extended-analyses}
\label{sec:extended-gate-support}

\raggedbottom

We finally examine population weighting and fallback in aggregate evaluation.
This separate study compares 16 configuration families: 14 with seeds 7/61
and two with seed 61 only. Latest-800k summaries average seeds within families,
then weight families equally. Thresholds $.01/.1/.3/.5/.7/.9$ use separate inference
passes at the same checkpoints, not re-thresholding one prediction
tensor or independent replications. Selected operating points
characterize this split without held-out calibration.

First, cardinality weighting changes the apparent benefit of an operating point.
One-layer rays contribute $87.488\%$ of occupied pixels. Moving from the $.5$ to the $.3$ evaluation pass,
micro exact-count accuracy decreases by $.343$ points, while cardinality-macro accuracy
increases by $1.730$ points and accuracy at $c=4$ increases by $7.602$ points.
All 16 families share these directions. The preferred global operating point therefore depends on the cardinality weighting.

Second, shallower-depth fallback changes which predictions contribute to error.
From $\tau=.01$ to $.9$, L7 own-slot support falls $69.31\to27.29\%$, while
last-visible support remains $100.00\to99.52\%$ because an absent deep slot
inherits the last surviving shallower prediction. Fallback AbsRel consequently
appears to improve $17.15\to14.98\%$, whereas own-slot AbsRel ends slightly
worse ($21.86\to22.05\%$). Fallback therefore entangles depth quality with selective
omission, making the retained-support denominator essential to interpretation.
These family-level trends describe the separate geometry study without changing the primary decoder.

\FloatBarrier
\Needspace{10\baselineskip}
\Needspace{38\baselineskip}
\section{Output and Membership Conventions}
\label{sec:representation-details}

Multilayer methods differ in their predicted geometry and valid-output selection. Table~\ref{tab:layered-representation-taxonomy} distinguishes variable-count decoding from explicit component-wise presence and absence modeling.

\par\addvspace{\intextsep}\noindent
\begin{minipage}{\textwidth}
\makeatletter\def\@captype{table}\makeatother
\centering
\AppendixTableSetup
\caption{\textbf{Visible layers, depth hypotheses, and amodal intersections.} The cited variants differ in their native outputs, supervised geometry, and treatment of unused capacity.}
\label{tab:layered-representation-taxonomy}
\label{tab:iclr-paradigms}
\AppendixTextTableStyle
\begin{tabularx}{\textwidth}{@{}>{\raggedright\arraybackslash}p{.16\textwidth}>{\raggedright\arraybackslash}p{.24\textwidth}>{\raggedright\arraybackslash}p{.23\textwidth}>{\raggedright\arraybackslash}X@{}}
\toprule
\textbf{Method} & \textbf{Native output} & \textbf{Target geometry} & \textbf{Count and membership} \\
\midrule
MDA~\citep{bian2026ambiguity} & $K$ weighted depth components & Boundary hypotheses; transparent-layer extension & Categorical alternatives; sigmoid weights permit coexisting transparent depths \\
\addlinespace[.5ex]
DepthFocus~\citep{depthfocus2026} & One stereo depth map per scalar query & Focus-selected surface & One layer per query; no simultaneous ray-set null/count law \\
\addlinespace[.5ex]
LayeredDepth baselines~\citep{layereddepth2025} & $K$ pixel-aligned scalar depth maps & Coexisting visible layers & Fixed rank/query identity; no learned per-map null/count law \\
\addlinespace[.5ex]
SeeGroup~\citep{seegroup2026} & Four recurrent Laplace components & Coexisting visible layers & Variable count after validity and gap filtering; no component-wise Bernoulli empty event \\
\addlinespace[.5ex]
World Tracing~\citep{worldtracing2026} & Six front-to-back camera-space XYZ maps & Visible and generated occluded intersections & Fixed stack; final real point is forward-filled \\
\addlinespace[.5ex]
LaRI~\citep{lari2026} & Ordered XYZ maps and stopping index & Visible, unseen, and back-facing intersections & Learned stopping gives a $0,\ldots,L$ valid prefix \\
\midrule
\textbf{Depth Any Seen: ExactMB} & Metric depth centers, localization scales, and presence & Annotated coexisting visible returns & $0,\ldots,K$ optional components; Bernoulli absence and an injective likelihood \\
\bottomrule
\end{tabularx}
\AppendixTableNote{``Per ray'' refers to one input pixel before multiview fusion. Fixed-capacity outputs can yield variable counts through validity masks, stopping, or repeated-point removal. ExactMB's normalized law uses untruncated densities before positive-depth filtering and deterministic decoding; ordered assignment additionally fixes correspondence. Across these comparisons, annotations define visibility, and recovery from a single image may be ambiguous.}
\end{minipage}
\par\addvspace{\intextsep}

\FloatBarrier

\FloatBarrier
\endgroup

\FloatBarrier
\end{document}